\documentclass{amsart}

\usepackage{color}
\usepackage{hyperref}
\usepackage[pdftex]{graphicx}
\usepackage{amsmath,amsthm,amscd}
\usepackage{amssymb,amsfonts,mathrsfs,amsxtra,dsfont}
\usepackage{enumitem}
\usepackage{algorithm}
\usepackage{algpseudocode}

\usepackage[capitalise,nameinlink]{cleveref}

\newcommand{\mbf}[1]{\mathbf{#1}}
\newcommand{\fat}[1]{\mathds{#1}}

\newcommand{\NN}{\fat{N}}
\newcommand{\NNH}{\widehat{\fat{N}}}
\newcommand{\ZZ}{\fat{Z}}
\newcommand{\RR}{\fat{R}}
\renewcommand{\SS}{\fat{S}}

\newcommand{\RRplus}{\RR_{_{\geq 0}}}
\newcommand{\dd}{\mathrm{d}}
\newcommand{\at}[1]{\big|_{\scriptscriptstyle{#1}}}
\newcommand{\indicator}[1]{\fat{1}_{#1}}
\newcommand{\wild}{{\scriptscriptstyle\bullet}}

\newcommand{\rv}[1]{\mathrm{#1}}
\newcommand{\prob}[1]{\mathbf{Pr}\!\left(#1\right)}
\newcommand{\ee}{\mathtt{e}}

\newcommand{\zero}{0}

\newcommand{\set}[2]{\left\{#1\left|#2\right.\right\}}

\newcommand{\THEN}{\Rightarrow}
\newcommand{\IFF}{\Leftrightarrow}
\newcommand{\DEFF}{\;\stackrel{\begin{array}{c}
	\scriptscriptstyle{def.}\\
\end{array}}{\Longleftrightarrow}}		
\newcommand{\minus}{\smallsetminus}		
\newcommand{\symplus}{\vartriangle}		
\newcommand{\into}{\hookrightarrow}		
\newcommand{\id}[1]{\mathtt{id}_{#1}}	
\newcommand{\power}[1]{\mathbf{2}^{#1}}	
\newcommand{\card}[1]{\left|#1\right|}	
\newcommand{\com}{{^{\scriptscriptstyle\ast}}}
\newcommand{\inv}{{^{\scriptscriptstyle -1}}}	
\newcommand{\expect}[1]{\mathbb{E}\left[#1\right]}
\newcommand{\diam}[1]{\mathtt{diam}\!\left(#1\right)}
\newcommand{\dist}[2]{\mathtt{dist}\!\left(#1,#2\right)}	
\newcommand{\hull}[2]{\mathtt{hull}\!\left(#1;#2\right)}		

\newcommand{\alphabet}{\mathbb{A}} 

\newcommand{\orth}[1]{\mathbb{O}(#1)}
\newcommand{\sel}{\mbf{S}} 
\newcommand{\val}[1]{\varphi_{#1}}

\newcommand{\north}{\mathtt{n}}
\newcommand{\south}{\mathtt{s}}
\newcommand{\west}{\mathtt{w}}
\newcommand{\east}{\mathtt{e}}

\newcommand{\cubings}{\mathscr{C}}

\newcommand{\spc}{\mathbf{X}}
\newcommand{\sens}{\mbf{\Sigma}}

\newcommand{\env}{\mbf{E}}					
\newcommand{\model}{\mathbf{M}}				

\newcommand{\ellone}[2]{\mbf{\Delta}\!\left(#1,#2\right)} 
\newcommand{\minP}{\mathbf{0}}		
\newcommand{\maxP}{\mathbf{0^\ast}}	
\newcommand{\up}[1]{#1\!\uparrow}		
\newcommand{\down}[1]{#1\!\downarrow}	
\newcommand{\coh}[2]{\mathtt{coh}_{#1}(#2)}			
\newcommand{\med}[3]{\mathrm{med}\!\left(#1,#2,#3\right)} 	
\newcommand{\morph}[3]{\mathrm{Hom}_{\scriptscriptstyle{#1}}\!\left(#2,\,#3\right)}	
\newcommand{\dual}[1]{\mathtt{Dual}\!\left(#1\right)}		
\newcommand{\cube}[1]{\mathtt{Cube}\!\,(#1)}				
\newcommand{\flip}[2]{\left[#1\right]_{_{#2}}}		
\newcommand{\half}[1]{\mathfrak{h}(#1)}		
\newcommand{\sepr}[1]{\mathtt{sep}\!\left(#1\right)} 

\newcommand{\pos}{\mathtt{pos}}			

\newcommand{\ppoc}{\mathbf{P}}		
\newcommand{\pog}{\mbf{\Gamma}}

\newcommand{\ori}[2]{\omega(#1#2)} 		
\newcommand{\tri}[2]{\partial(#1#2)} 			
\newcommand{\proj}[1]{\mathtt{proj}_{#1}}

\newcommand{\delay}{\sharp}

\newcommand{\propagate}[2]{\mathtt{PROP}(#1;#2)}

\newcommand{\witness}[2]{\mathtt{w}^{#1}_{#2}} 		
\newcommand{\observe}[1]{\mathtt{Obs}^{#1}}			

\newcommand{\prediction}[1]{\mathtt{Pred}^{#1}}

\newcommand{\current}[1]{\mathtt{Curr}^{#1}}		
\newcommand{\divergence}[1]{\mathtt{Div}(#1)}	

\newcommand{\lt}{\mathtt{lt}}
\newcommand{\rt}{\mathtt{rt}}

\newcommand{\ham}{\mathbb{H}}
\newcommand{\pointmass}[2]{\delta_{#1,#2}}
\newcommand{\minset}[1]{\mathtt{M}(#1)}
\newcommand{\residual}[1]{\mathbf{Res}\!\left(#1\right)}
\newcommand{\derived}[1]{\mathbf{Der}\!\left(#1\right)}
\newcommand{\completion}[1]{\widehat{#1}}
\newcommand{\twores}[1]{#1^{\scriptscriptstyle{(2)}}}
\newcommand{\twoclose}[1]{\widehat{#1}}
\newcommand{\supp}{^{\scriptscriptstyle{\sharp}}}

\newcommand{\kldiv}[2]{\mathbf{D}_{_{KL}}\!\left(#1\big{\Vert}#2\right)}

\theoremstyle{plain}
\newtheorem{theorem}{Theorem}[section]
\newtheorem{proposition}[theorem]{Proposition}
\newtheorem{lemma}[theorem]{Lemma}
\newtheorem{corollary}[theorem]{Corollary}

\theoremstyle{definition}
\newtheorem{definition}[theorem]{Definition}

\newtheorem{example}[theorem]{Example}
\newtheorem{remark}[theorem]{Remark}

\begin{document}

\title{Iterated Belief Revision Under Resource Constraints: Logic as Geometry}

\author{Dan P.~Guralnik}
\address{Electrical \& Systems Engineering, School of Engineering \& Applied Sciences, University of Pennsylvania, Penn Engineering Research \& Collaboration Hub (PERCH), 3401 Grays Ferry Ave., Pennovation Center, Building 6176, 3rd Floor, Philadelphia, PA 19146}
\email{guraldan@seas.upenn.edu}

\author{Daniel E.~Koditschek}
\address{Electrical \& Systems Engineering, School of Engineering \& Applied Sciences, University of Pennsylvania, Penn Engineering Research \& Collaboration Hub (PERCH), 3401 Grays Ferry Ave., Pennovation Center, Building 6176, 3rd Floor, Philadelphia, PA 19146}
\email{kod@seas.upenn.edu}

\begin{abstract}
We propose a variant of iterated belief revision designed for settings with limited computational resources, such as mobile autonomous robots.
The proposed memory architecture---called the {\em universal memory architecture} (UMA)---maintains an epistemic state in the form of a system of default rules similar to those studied by Pearl and by Goldszmidt and Pearl (systems $Z$ and $Z^+$).

A duality between the category of UMA representations and the category of the corresponding model spaces, extending the Sageev-Roller duality between discrete poc sets and discrete median algebras provides a two-way dictionary from inference to geometry, leading to immense savings in computation, at a cost in the quality of representation that can be quantified in terms of topological invariants.
Moreover, the same framework naturally enables comparisons between different model spaces, making it possible to analyze the deficiencies of one model space in comparison to others.

This paper develops the formalism underlying UMA, analyzes the complexity of maintenance and inference operations in UMA, and presents some learning guarantees for different UMA-based learners.
Finally, we present simulation results to illustrate the viability of the approach, and close with a discussion of the strengths, weaknesses, and potential development of UMA-based learners.
\end{abstract}

\maketitle

\section{Introduction}\label{introduction}
\subsection{Motivation.}\label{intro:motivation}
Iterated belief revision (BR) deals with the problem of maintaining syntactic propositional knowledge representations that are sufficiently flexible to accommodate reasoning about a stream of incoming observations in the form of propositional formulae (over a finite alphabet of atomic propositions), while taking into account the possibility of any such observation being inconsistent with the current state of the knowledge representation.
It is not unreasonable then to argue that BR operators should be used for maintaining well-reasoned internal representations for autonomous learning agents (see, e.g.~\cite{To_Son_Pontelli-planning_incomplete_information}).
However, one needs merely to observe the high computational costs associated with revision operators~\cite{Liberatore_Schaerf-BR_complexity_model_checking,Liberatore-complexity_iterated_BR} to conclude that such representations are too expensive to implement them in a mobile autonomous agent.
Attempts at making the representations more palatable using prime forms~\cite{Boutilier-unified_model_qualitative_belief_change,Marchi_Bittencourt_Perrussel-prime_forms_and_min_change_in_belief} have been made, but the fundamental complexity barriers remain~\cite{Kean_Tsiknis-computing_prime_implicants}.

We introduce a computationally cheap form of iterated propositional belief revision---the universal memory architecture (UMA)---which harnesses the geometry of model spaces in place of the model-theoretic techniques characteristic of this field.
The computational advantages come at the price of modifying the notion of an observation and restricting the syntactic form of the epistemic state maintained by the agent (understood in the broad sense of Darwiche and Pearl~\cite{Darwiche_Pearl-iterated_belief_revision}) to a special type of default system in the sense of~\cite{Pearl-system_Z}.
Most notably, observations are no longer allowed to take the form of arbitrary propositional formulae; rather, we restrict them to conjunctive monomials in the underlying propositional variables.
Equivalently, an observation is a partial truth-value assignment to the agent's inputs.
In addition, each observation is accompanied by a {\em value signal}---a quantity indicating a notion of the value of the experience to the agent at that time.
\footnote{The value signal should not be confused with the notion of reward, as used in Reinforcement Learning. One of our learning schemes (see Section~\ref{snapshot:real-valued}) leads to a (partial) syntactic representation of the distribution from which observations are being drawn, and does not encode any preference of one state over another.}

These alterations to the classical setting of iterated BR are motivated by the prospect of implementing iterated BR on mobile robotic platforms in real time.
While the Boolean component of the observation corresponds to the robot's raw sensory inputs, the value signal may correspond to an encoding of a task, or to feedback from a teacher.
The limited form of the epistemic state maintained by an UMA instance reduces the space and time complexity costs of maintenance (applying the revision operator) and exploitation (e.g. inference) down to an absolute minimum, as we review next.

\subsection{Contributions: Introduction and Analysis of UMAs.}\label{intro:contrib}
Motivated by the problem of realizing iterated belief revision and update in a bounded resources setting, we seek a class of lightweight general-purpose representations.
From a learning perspective, ours is a problem of {\em learning from positive examples}: an observer of an unknown, unmodeled system $\mathscr{S}$ experiences some process---a sequence of transitions---in that system through an array of Boolean sensors, and is required to reason about regularities in the observed sequence of experiences, constructing a formal theory of what is possible for that system.

We assume that observations occur in discrete time steps.
An {\em observation} at time $t$ will consist of (1) a complete truth-value assignment\footnote{Henceforth, the symbol $(\at{t})$ appended to anything else should be read as ``at time $t$''.} $\observe{}\at{t}$---the {\em observation at time $t$}---to a fixed set $\sens$---the {\em sensorium}---of Boolean queries of the agent's interactions with its environment; and of (2) a sample $\val{}\at{t}$ of a fixed {\em value signal}, $\val{}$.

Little needs to be assumed about the sensorium:
for the purpose of this paper, we allow any query expressible as a Boolean function of the state history (finite or infinite) of the system $\mathscr{S}$ (an appropriate formalism is developed in \Cref{models:first notions});
it is also assumed that truth-value assignments $\observe{}\at{t}$ are {\em consistent} in the sense that each agrees with the values of the available queries on the history that manifested at the corresponding time; 
finally, observations are assumed to be {\em time shift-invariant} in the sense that observing the same histories at different times must yield the same Boolean observation vector.
The value signal, for now, is assumed to be {\em static}, in the sense that it factors through a function of the observation (more detail in \Cref{high-level:observation model}).
The architecture itself does not rely on any of these assumptions, but the learning guarantees we provide in this paper do.

An UMA representation integrates its accumulated experiences by repeatedly revising two structural components, based on the incoming observations:
(a) a relation $G\at{t}$, called a {\em pointed complemented relation} (PCR), representing a system of implications, or {\em defaults}, which the agent believes to hold true among the queries in $\sens$;
and (b) a set $\current{}\at{t}\subset\sens$, representing the agent's belief regarding the current state of the system.
The machinery for maintaining these data structures will be referred to as a {\em snapshot}.
Briefly, our results about UMA representations are as follows.

\paragraph{Universality of Representation.}\label{contrib:universality}
In our intended setting, the learner's sensors realize the formal sensorium $\sens$ as a family of subsets of the space of histories, closed under complementation.
The possible worlds actually witnessed by points of this space correspond to the learner's perceptual equivalence classes (in the sense of, e.g.~\cite{Donald_Jennings-planning_from_sens_equiv_classes,Rivest_Schapire-diversity}).
Intuitively, an element $(a,b)$ of the PCR $G\at{t}$ should be seen as correct if no history falsifies the formula $a\rightarrow b$, and, more generally, if histories falsifying $a\rightarrow b$ are improbable, or insignificant according to the user's formal model of these notions.

It turns out that a PCR $G$ supports a natural {\em dual space}, a set $\model=G^\circ$ of possible worlds canonically associated with the PCR.
Recall that a {\em possible world} over $\sens$ is a complete truth value assignment $\sens\to\{\bot,\top\}$.
We prove that, given a PCR $G$ over a set of literals $\sens$, its dual space $\model$ has the following universality property (\Cref{prop:universality}):
$\model$ is the smallest set of possible worlds over $\sens$ which, for {\em any} realization $\rho$ of $\sens$ as a set of Boolean queries over a space $\spc$ not falsifying a relation listed in $G$, contains every model for $\rho$.

Returning to UMA learners, this means that the model space $\model\at{t}$ encoded by the PCR $G\at{t}$ is a minimal envelope for the true space of possible worlds, provided {\em just the information} that all the relations recorded in $G\at{t}$ are correct.

\paragraph{Computational Complexity.}\label{contrib:complexity}
From a computational perspective, the maintenance costs of an UMA representation are roughly the same as those of maintaining a neural representation (=the cost of maintaining and using a matrix of weights), but with the added benefit of affording a formal understanding of the model space, its geometry, and its deficiencies.
Here are some results, all of which are corollaries of the {\em geometric} properties of the class of model spaces defined by PCRs.
Let $N$ denote the cardinality of the sensorium $\sens$.
Then:
\begin{itemize}
    \item Maintaining an UMA snapshot structure requires $O(N^2)$ space;
    \item Update operations for learning the PCR structure require $O(N^2)$ time;
    \item Inference requires $O(N^2)$ time, reducible to $O(N)$ on fully parallel hardware.
    \footnote{We will remark that our current implementation is, in fact, an $O(N^3)$ implementation utilizing matrix multiplication on a GPU.
    This kind of implementation makes it possible to multiply fairly big matrices very quickly, improving on the performance of the na\"ive quadratic algorithm we provide later in this paper.}
\end{itemize}
%

\paragraph{Multiple Learning Paradigms.}\label{contrib:learning}
The mathematical foundations for UMA provide sufficient flexibility to admit a variety of learning mechanisms and settings, spanning the range from probabilistic filtering, as proposed in~\cite{GK-Allerton_2012}, to a variation on [iterated] revision and update introduced in~\cite{Darwiche_Pearl-iterated_belief_revision}, while keeping maintenance costs down to the bare minimum (see preceding paragraph).
Depending on the {\em snapshot type}, different learning scenarios and guarantees may be provided, while maintaining a uniform revision and update scheme at the symbolic level.

\paragraph{Flexibility of Representation.}\label{contrib:flexibility} 
A central feature of the UMA architecture is that the duality theory of PCRs allows one to interpret maps between PCRs as maps between the associated model spaces and vice versa.
This makes it possible to formally introduce---as well as operate with---notions of approximate equivalence, of redundancy and negligibility of queries.
This also enables the study of the impact on model space geometry of operations augmenting a sensorium with new queries (see, for example, \Cref{poc:ex:moving bead on an interval}) or removing existing ones.
In particular, this opens a way to formal (and, possibly, automated) cost/benfit analysis of such extension and pruning operations---a topic of ongoing research at the moment, which we will touch upon briefly in our final discussion of the results presented in this paper.

%
%
%

\subsection{Related Work.}\label{related}

Given the focus of this work on the representation of knowledge using defaults, we believe it is most tightly related to work in the field of propositional iterated belief revision.
Early work in BR resulted in wide acceptance of the AGM framework~\cite{Alchourron_Makinson-belief_set_revision_0,Alchourron_Makinson-belief_set_revision_1,AGM-belief_set_revision} for maintaining a {\em belief set}---a deductively closed set of formulae representing the state of the observed system.
Convenient, intuitive axioms for belief revision in the propositional setting, the KM axioms, were developed by Katsuno and Mendelzon in~\cite{Katsuno_Mendelzon-prop_knowledge_base_revision}.

Pointing out some inadequacies of the KM axioms in the context of repeated application of revisions, Darwiche and Pearl (DP) argue in their seminal paper~\cite{Darwiche_Pearl-iterated_belief_revision} that, to achieve the overarching goal of iterated revision, one must maintain a set of conditional statements---an {\em epistemic state}---which, upon revision by an incoming observation, {\em always produces a belief set accommodating that observation} (axiom $\mbf{R\!\ast\! 1}$ of the DP system of axioms for iterated revision).
Building on Spohn's framework of ordinal conditional functions~\cite{Spohn-OCF_epistemic_state} and its implications for ranked default systems~\cite{Pearl-system_Z,Goldszmidt_Pearl-system_Z+} and revision of the associated belief sets~\cite{Goldszmidt_Pearl-qualitative_probas_for_belief_revision}, they propose to view ranking functions as epistemic states (interchangeable with the associated system of ranked defaults), as they construct appropriate revision operators. 
Consequent work by many authors~\cite{Jin_Thielscher-iterated_BR_revised,Delgrande_Jin-parallel_BR,KernIsberner_Huvermann-multiple_IBR_wo_independence,KernIsberner_Brewka-strong_syntax_splitting_for_IBR,MKL-BR_for_general_epistemic_states,Konieczny_Perez-more_dynamics_in_revision}---much of it very new---considers different weaknesses and benefits of the DP axioms, relating to the effect of the order in which observations are made and the manner of mutual dependence they present, and resulting in a variety of iterated revision methods, as well as in some proposals to apply belief revision methods to the control of general agents~\cite{To_Son_Pontelli-planning_incomplete_information} based on varying computational approaches to belief revision operators (e.g.~\cite{Boutilier-unified_model_qualitative_belief_change,Marchi_Bittencourt_Perrussel-prime_forms_and_min_change_in_belief} on the use of prime forms for this purpose).

Clearly, the problems tackled by this field generalize the representation problem we posed at the beginning of \Cref{intro:contrib}, but one needs merely to observe the high computational costs associated with revision operators~\cite{Liberatore_Schaerf-BR_complexity_model_checking,Liberatore-complexity_iterated_BR} (or with computing normal forms and prime forms~\cite{Kean_Tsiknis-computing_prime_implicants}) to reach the conclusion that the existing computational approaches cannot be considered viable candidates for a solution of the representation problem in any setting where computational resources are limited.

Aiming to reduce the computational burden on the learner, we shift attention from precise syntactic computation with arbitrary propositional formulae to imposing {\em radical} simplifying assumptions on the allowed model spaces. The postulated mode of interaction between the agent and its environment---specifically the fact that the agent is constrained to processing sequences of samples from the space $\model$ of realizable models (rather than arbitrary propositional formulae)---suggests constructing successive {\em upper approximations} $\model\at{t}\supset\model$ of $\model$, belonging to a restricted class $\cubings$ which satisfying the following intuitive properties:
\begin{enumerate}
	\item Syntactic characterization of an element in $\cubings$ is computationally inexpensive;
    \item Each approximation is, in some sense, optimal/minimal among members of $\cubings$, given its predecessor and the last observation;
    \item Reasoning (e.g., forming a belief set) over a member of $\cubings$ is cheap.
\end{enumerate}
We present results on what is, in essence, the simplest possible class $\cubings$ of model spaces satisfying these three requirements: the class of finite median algebras.
This class of spaces is well studied, in several different guises, and in very disparate fields. These include: event structures in parallel computation~\cite{Pratt-modeling_concurrency_with_geometry}; median graphs in metric graph theory~\cite{Chepoi-median}; simply connected non-positively curved cubical complexes in formalizations of reconfiguration in robotic systems~\cite{Ghrist_Peterson-reconfiguration}; and the spectacular recent achievements in the topology of 3-dimensional manifolds by Agol~\cite{Agol-virtual_Haken} are much due to the notion of a cubulated group from Geometric Group Theory~\cite{Wise-riches_to_raags}.

\subsection{Structure of this Paper.}\label{intro:structure of this paper}
In \Cref{models}, we extend Sageev-Roller duality\footnote{See~\cite{Roller-duality} for a detailed development of that theory; chapters 6-7 of~\cite{Wise-riches_to_raags} for a brief intuitive review; and here, \Cref{poc} for background material and examples developed specifically to support this paper.}, to obtain all finite median algebras as duals (model spaces) of PCRs, viewed as systems of defaults.
Further, we explain how to reason over model spaces in this class by leveraging their geometry to avoid satisfiability checks, or any kind of explicit search in model space, for that matter.
We then explain in \Cref{high-level} how, using UMA snapshot structures to perform a variant of iterated revision, where the model-theoretic outlook on the problem is replaced by its geometric counterpart arising by Sageev-Roller duality.
We discuss the necessity of relaxing the DP axiom $\mbf{R\!\ast\! 1}$, and show there is a natural operator for computing a belief set, the {\em coherent projection}.

\Cref{snapshots} presents two different classes of snapshot structures---mechanisms for learning PCR representations---one motivated by Goldszmidt and Pearl's interpretation of default reasoning as qualitative probabilistic reasoning~\cite{Goldszmidt_Pearl-qualitative_probas_for_belief_revision}, and the other based on statistical integration of the observed value signal.
Finally, \Cref{simulations} presents two kinds of simulation studies:
\begin{enumerate}
    \item First, in a range of settings with {\em a-priori} known (or readily computable) implications in the sensorium, we consider the deviation of the learned PCR from the ground truth as a function of the number of samples.
    This is done for both snapshot types, and under different exploration paradigms: sampling and diffusion.
    \item Next, we consider settings closer to the heart of a roboticist.
    We implement agents with a reactive control paradigm based entirely on their internal UMA representations and conduct comparative simulation studies of their performance given different domains for exploration, and snapshot types.
\end{enumerate}
We close with a discussion of our results and of avenues for additional research in \Cref{discussion}.

\section{Model Spaces for Systems of Approximate Implications.}\label{models}
In this section we construct a representation for finite median algebras (see above) that is sufficiently flexible to be maintained dynamically, and we explain how to reason over these representations.
We review and apply existing results about the geometry of model spaces of this class of representations, leading to complexity bounds on maintenance and exploitation.

\Cref{models:PCRs} formally introduces the basic formal notions required for discussing our representations.
\Cref{models:duals} constructs the model spaces as dual spaces of {\em pointed complemented relations} (PCRs) and discusses their universal properties.
\Cref{models:PCRs to poc sets} relates PCRs and their duals (the associated model spaces) to the earlier duality theory of poc sets that motivated our approach, showing that PCR duals are, in fact, poc set duals.
\Cref{models:convexity} reviews known results about the geometry and topology of poc set duals.
Finally, in \Cref{models:propagation} we discuss the connection between the geometry of PCR duals and algorithms enabling reasoning over PCRs.

\subsection{Pointed Complemented Relations (PCR).}\label[subsection]{models:PCRs} The nature of our application requires a generalization of the formal theory we are about to use, the Sageev-Roller duality theory of poc sets~\cite{Roller-duality}, prompting some changes in the language. We start with:
\begin{definition}[pointed complemented set, PCS]\label[definition]{defn:PCS} A {\em pointed complemented set} is a set $\sens$ endowed with a self-map $a\mapsto a\com$ satisfying
$a\com\com=a$ and $a\com\neq a$ for all $a\in\sens$, and containing a distinguished element, denoted $\mbf{0}$. 
The element $\mbf{0}\com$ will be denoted $\mbf{1}$.
Whenever possible and safe, we will abuse notation and use the symbols $\mbf{0},\mbf{1},\ast$ in different PCSs.
For any $S\subset\sens$ we will denote by $S\com$ the set of all $x\com$, $x\in S$.\qed
\end{definition}
\begin{definition}[PCS morphism]\label{defn:PCS morphism}
By a PCS morphism we mean a function $f:\sens_1\to\sens_2$ between PCSs satisfying $f(\mbf{0})=\mbf{0}$ and $f(a\com)=f(a)\com$ for all $a\in\sens_1$.
The set of all PCS morphisms from $\sens_1$ to $\sens_2$ will be denoted by $\morph{PCS}{\sens_1}{\sens_2}$.\qed
\end{definition}
\begin{example}[set families, power sets]\label[example]{ex:power set as PCS} Any collection $\mathscr U\subseteq\power{X}$ of subsets of a fixed non-empty set $\spc$ satisfying (1) $\varnothing\in\mathscr U$, and (2) $A\in\mathscr U\THEN \spc\minus A\in\mathscr U$.
Then $\mathscr U$ is a PCS with respect to the choices $\mbf{0}:=\varnothing$ and $A\com:=\spc\minus A$.

The power set of a singleton is, up to isomorphism, the smallest PCS, which we denote by $\power{}$, and identify with the set $\{\bot,\top\}$.
Also, the power set $\power{X}$ will be routinely identified with the set of all functions $X\to\{\bot,\top\}$.
\end{example}
\begin{example}[PCS over an alphabet]\label[example]{ex:PCS over an alphabet} Suppose $\alphabet$ is a finite collection of symbols, and think of them as atoms of the propositional calculus over $\alphabet$.
The extended collection of literals over $\alphabet$,
\begin{equation}\label{eqn:PCS from alphabet}
	\sens(\alphabet):=\{\bot,\top\}\cup\bigcup_{a\in\alphabet}\{a,\neg a\}
\end{equation}
may be thought of as a PCS when one declares $\mbf{0}:=\bot$, $\bot\com:=\top$, $\top\com:=\bot$ and $a\com:=\neg a$, $(\neg a)\com:=a$ for all $a\in\alphabet$. Hereafter, $\top$ and $\bot$ stand for the truth values \texttt{True} and \texttt{False}, respectively.
\end{example}

The reason for considering PCSs is that $\ast$-selections ``live on them'':
\begin{definition}[$\ast$-selection, the Hamming cube]\label[definition]{defn:ast-selection} Let $\sens$ be a PCS. By a {\em $\ast$-selection} on $\sens$ we mean a subset $S\subset\sens$ such that $S\cap S\com=\varnothing$. In addition, a $\ast$-selection $S$ on $\sens$ is {\em complete}, if $S\cup S\com=\sens$. The set of all $\ast$-selections $S\subset\sens$ with $\mbf{1}\in S$ will be denoted by $\sel(\sens)$, and referred to as the [combinatorial] {\em Hamming cube on $\sens$}. Its set of vertices, the complete $\ast$-selections in $\sel(\sens)$, will be denoted by $\ham(\sens)$.
\qed\end{definition}
We now consider these notions in the context of our intended application.

\subsubsection{Binary Sensing, Possible Worlds and Perceptual Classes.}\label{models:first notions}
Suppose $\alpha$ is an observer of some system $\mathscr{S}$ as it undergoes the transitions along a state trajectory $(p_t)_{t=-\infty}^\infty$, and suppose $\alphabet$ is a finite set of unique labels for the Boolean queries available to $\alpha$---this observer's {\em sensorium}.
We assume observations of $\mathscr{S}$ by $\alpha$ begin at $t=0$.
It will not matter for our discussion whether the trajectory of $\mathscr{S}$ in any particular instance does indeed extend indefinitely into the past or future: if needed, one may set the value of $p_t$ to be eventually constant (in either direction).

By a {\em history} of $\mathscr{S}$ we mean a sequence of the form $\mbf{x}:=(x_s)_{s=-\infty}^0$, where $x_s$ is a state of $\mathscr{S}$ for all $s$, and $x_0$ represents the {\em current state} of the history $\mbf{x}$; $x_{-1}$ represents the {\em preceding state}, and so on.
Given a trajectory $(p_t)$ of $\mathscr{S}$ observed by $\alpha$, at each time $t\geq 0$, the history that {\em manifests at time $t$} is given by $x_s:=p_{t+s}$.

Henceforth, we let $\spc$ denote the space of histories possible for the system $\mathscr{S}$ given the initial history manifested at time $t=0$ (as is the case in all physical systems, $\mathscr{S}$ may have its own dynamics, disqualifying some histories from manifesting at any time $t>0$, or making such events highly improbable).
To say that $\alpha$'s queries/sensors are time-shift invariant is to say that each query is represented by a fixed Boolean function of the manifested history.
In other words, the sensorium is {\em defined} by a PCS morphism $\rho:\sens\to\power{\spc}$, $\sens:=\sens(\alphabet)$, with a sensor $s\in\sens(\alphabet)$ reporting $\top$ on history $x\in\spc$ if and only if $x\in\rho(s)$.

The mapping $\rho$ induces a partition on $X$---its partition into {\em perceptual classes}---as follows.
Construct a map $\rho\com\colon\spc\to\ham(\sens)$ by setting $s\in\rho\com(x)$ if and only if $x\in\rho(s)$; each point is mapped to the set of queries (including complements) which evaluate to $\top$ on that point.
Two points $x,y\in\spc$ are {\em sensory-equivalent} if $\rho\com(x)=\rho\com(y)$. The image $\model(\rho):=\mathrm{Im}(\rho\com)$ are the {\em possible perceptual states} of $\alpha$ in the system $\mathscr{S}$, given $\rho$ and the system's initial history.
We will also refer to a world/$\ast$-selection $u\in\ham(\sens)$ as {\em consistent}, if, and only if $u\in\model(\rho)$, or, in other words, if and only if $u$ is witnessed (through $\rho$) by a point of $\spc$.

\subsubsection{Concept Presentation of Perceptual States.}\label{models:concept presentation}
Digging deeper into the formalism presented just now, observe that $\ast$-selections $S\subset\sens(\alphabet)$ are in one-to-one correspondence with {\em vectors}, as defined in concept learning~\cite{Valiant-PAC_paper}. Recall that a {\em vector} is an assignment $v:\alphabet\to\{\mbf{0},\mbf{1},\wild\}$ of values standing for $\bot$, $\top$, and ``undetermined'', respectively, to the alphabet $\alphabet$. A vector is {\em total} if it has no $(\wild)$ values. The map $v\mapsto \sigma_v:=v\inv(\mbf{1})\cup (v\inv(\mbf{0}))\com$ is then a correspondence between vectors over $\alphabet$ and $\ast$-selections on the PCS $\sens(\alphabet)$, mapping the set of total vectors onto the set of complete $\ast$-selections.
In more geometric terms, a complete $\ast$-selection---which corresponds to a complete conjunctive monomial (aka {\em complete term}) over $\alphabet$---defines a vertex of the cube $[0,1]^\alphabet$, while a $\ast$-selection $S$ with $|S|=|\alphabet|-d$ corresponds to a $d$-dimensional face.
We will refer to $[0,1]^\alphabet$ as the {\em Hamming cube}.
The advantage of PCS terminology here is that $\ast$-selections on $\sens$ enumerate the faces of the Hamming cube without us having to pick an origin for the cube.

Pushing the geometric viewpoint a bit further, we consider the notion of {\em concepts}.
In ~\cite{Valiant-PAC_paper}, Valiant defines concepts as mappings $F$ of the space of vectors to $\{\bot,\top\}$, satisfying the requirement that $F(v)=1$ on a vector $v$ if and only if $F(w)=1$ for all total vectors $w$ which agree with $v$ on those $a\in\alphabet$ where $v(a)\neq\wild$.
In other words, concepts correspond to collections $K$ of {\em faces} of the Hamming cube, possibly of varying dimensions, satisfying the condition that a face $F$ belongs to $K$ if and only if every vertex of $F$ lay in $K$.
Such $K$ are {\em precisely} the sub-complexes of the Hamming cube obtainable from it by vertex deletions.\footnote{Similarly to case of graphs, the operation of deleting a vertex from a cubical complex requires the removal of all the adjoining faces.}

Now we return to the observer $\alpha$ and the system $\mathscr{S}$ whose evolution it observes through the queries realized by $\rho\colon\sens(\alphabet)\to\power{\spc}$, as discussed in the preceding section.
Thinking of the space of perceptual classes $\model(\rho)$ as a concept gives rise to a cubical sub-complex, say $\cube{\rho}$, of the Hamming cube, whose faces correspond to those $\ast$-selections on the PCS $\sens=\sens(\alphabet)$ that are witnessed (via $\rho$) by a point in $\spc$.
Thus, precise reasoning and planning over $\model(\rho)$ depends on one's ability to efficiently capture/encode: (1) the notion of consistency produced by the map $\rho$; (2) the topological properties (e.g. connectivity, contractibility) of $\cube{\rho}$; and (3) the geometric properties (e.g. shortest paths, curvature, isoperimetric inequalities) of $\cube{\rho}$.
The class of approximating model spaces we propose to use as proxies for $\model(\rho)$ is a result of weakening this notion of consistency to the extreme, all the way to the notion of {\em coherence} discussed in the next section.

\subsubsection{PCRs, Implications and Coherence.}
\begin{definition}[pointed complemented relation, PCR]\label[definition]{defn:PCR} Let $\sens$ be a PCS.
By a {\em pointed complemented relation} over $\sens$ we mean a set $G\subseteq\sens\times\sens$ satisfying\footnote{To avoid a proliferation of parentheses, we write $ab$ to denote the pair $(a,b)\in\sens\times\sens$.} $\mbf{0}a\in G$ and $ab\in G\IFF b\com a\com\in G$ for all $a,b\in\sens$.\qed
\end{definition}
In the context of the representation problem, one should think of a PCR $G$ over $\sens$ as a record of Boolean implications believed to be valid over $\sens=\sens(\alphabet)$, conditioned on the particular space of histories being observed. In this respect, a PCR is a restricted form of the notion of a {\em system of defaults}, as discussed, e.g. in~\cite{Goldszmidt_Pearl-qualitative_probas_for_belief_revision}.
Some of these implications are specified directly ($ab\in G$ to be read as ``it is believed that $b$ follows from $a$''), while others are derived as their consequences, by transitive closure. Hence the following language:
\begin{definition}\label[definition]{PCR notation}
Given a PCR $G$ over a PCS $\sens$, for any $a,b\in\sens$, $S\subseteq\sens$, one defines the following:
\begin{itemize}
	\item Write $a\leq_G b$ if $ab$ lies in the reflexive and transitive closure of $G$;
    \item The {\em $G$-equivalence class} of $a\in\sens$, denoted $[a]_G$, is the equivalence class of $a$ under the relation $a\sim b\IFF a\leq_G b\wedge b\leq_G a$ on $\sens$;
	\item The {\em forward (backward) closure}, $\up{S}$ (resp. $\down{S}$), of $S$ with respect to $G$ is the set of all $b\in\sens$ for which $a\leq_G b$ (respectively $b\leq_G a$) holds for some $a\in S$;
	\item Note that $S\subset\up{S}$. One says that $S$ is {\em forward-closed} if $\up{S}=S$;
	\item Finally, we observe that $\up{S\com}=\down{S}\com$ for all $S\subseteq\sens$.
\end{itemize}
We will often drop the subscripts $G$ when no ambiguity can arise.\qed
\end{definition}
\begin{definition}[PCR morphism] Let $G_1,G_2$ be PCRs over $\sens_1,\sens_2$, respectively. A morphism of PCRS from $G_1$ to $G_2$ is a PCS morphism $f:\sens_1\to\sens_2$, additionally satisfying $f(a)\leq f(b)$ in $G_2$ whenever $ab\in G_1$. 
The set of all morphisms from $G_1$ to $G_2$ will be denoted by $\morph{PCR}{G_1}{G_2}$.\qed
\end{definition}
The primary example of a PCR for this work derives from the view of a power set as a PCS (\Cref{ex:power set as PCS}): 
\begin{example}[Set Families as PCRs]\label[example]{ex:power set as PCR} Let $\spc\neq\varnothing$ be a set. Then any collection $\mathscr{U}$ of subsets of $\spc$ that is closed under complementation and satisfies $\varnothing\in\mathscr{U}$ gives rise to the PCR of all pairs $(A,B)$ with $A\subseteq B$, $\mbf{0}=\varnothing$ and $A\com=\spc\minus A$. In what follows, $\power{X}$ will always be regarded as a PCR in this way, for any $\spc\neq\varnothing$.\qed
\end{example}
Another `canonical' example of a PCR to keep in mind is:
\begin{example}[Less classical PCRs]\label[example]{ex:fuzzy PCR} Let $\spc\neq\varnothing$ be any set. Then $[0,1]^\spc$ may be endowed with the structure of a PCR by setting $\mbf{0}(\mbf{x}):=0$, $\mbf{x}\in\spc$, and, for any $\psi,\varphi\in[0,1]^\spc$, setting $\psi^\ast(\mbf{x}):=1-\psi(\mbf{x})$, $\mbf{x}\in\spc$ and $(\varphi,\psi)\in G$ if and only if $\varphi(t)\leq\psi(\mbf{x})$, $\mbf{x}\in\spc$.\qed
\end{example}
Our notion of model for a PCR rests on the following weak form of consistency:
\begin{definition} Let $G$ be a PCR over $\sens$. A subset $S\subseteq\sens$ is said to be {\em $G$-coherent}, if no pair $a,b\in S$ satisfies $a\leq b\com$.\qed
\end{definition}
Note that a $G$-coherent set is always a $\ast$-selection on $\sens$. Furthermore:
\begin{eqnarray}
	S\text{ is coherent }
		\IFF
        S\cap\down{S\com}=\varnothing
		\IFF
        \up{S}\cap S\com=\varnothing
		\IFF
        \up{S}\cap\down{S\com}=\varnothing\,,
\end{eqnarray}
so coherence is preserved by forward closure. Coherent, forward-closed sets may be thought of as the natural counterparts of the notion of a belief state in this setting. We now turn to studying the appropriate notion of model.

\subsection{Model Spaces as Dual Spaces}\label[subsection]{models:duals}
\begin{definition}[duals]\label[definition]{defn:dual} Let $G$ be a PCR over $\sens$. The set $G^\circ$ of maximal $G$-coherent subsets of $\sens$ is the {\em dual of $G$}. The set of all forward-closed $G$-coherent subsets will be denoted $\mbf{C}(G)$.\qed
\end{definition}
A standard application of Zorn's lemma shows that any $G$-coherent subset of $\sens$ is contained in an element of $G^\circ$. 
Note also that $G^\circ\subseteq\mbf{C}(G)$.
\begin{example}[the orthogonal PCR and the Hamming cube] The simplest example of a dual space is one where the PCR in question is as small as possible. Let $\sens$ be a PCS. The smallest PCR over $\sens$ contains only pairs of the forms $\mbf{0}a$ and $a\mbf{1}$. We will denote this PCR by $\orth{\sens}$ and refer to it as the {\em orthogonal PCR} over $\sens$. It is clear that $\orth{\sens}^\circ=\ham(\sens)$, the ``Hamming cube'' from \Cref{defn:ast-selection}.\qed
\end{example}

\begin{example}[`bad' queries] The definitions given above do not preclude one from considering, for example, the PCR $G_1=\{\mbf{0}\mbf{1},\mbf{1}\mbf{0}\}$. It is easy to see that $G_1^\circ=\{\varnothing\}$. At the same time, the smaller $G_2=\{\mbf{0}\mbf{1}\}$ has $G_2^\circ=\{\mbf{1}\}$. 
More generally, for any $a\in\sens$, having $a\leq a\com$ precludes $a$ from belonging in any $G$-coherent set. In particular, if both $a\leq a\com$ and $a\com\leq a$ hold, then no $G$-coherent set is a complete selection on $\sens$.\qed
\end{example}
Following the last example, two definitions are in order: 
\begin{definition}\label[definition]{defn:trivial PCR} The {\em trivial} PCR, henceforth also denoted by $\power{}$, is the PCR over $\sens=\{\mbf{0},\mbf{1}\}$ containing only $\mbf{0}\mbf{1}$.\qed
\end{definition}
\begin{definition}[negligible query, degenerate graph]\label[definition]{defn:negligible_degenerate} Let $G$ be a PCR over $\sens$. An element $a\in\sens$ is {\em $G$-negligible}, if $a\leq a\com$. Denote the set of negligible elements by $N(G)$. We say that $G$ is {\em degenerate} if $\sens$ contains a negligible element whose complement is also negligible. Note that $\down{N(G)}=N(G)$.\qed
\end{definition}
\begin{proposition}\label[proposition]{prop:when maximal coherent is complete} For a PCR $G$ over $\sens$, the following are equivalent:
\begin{enumerate}
	\item $G$ is non-degenerate;
	\item Every element of $G^\circ$ is a complete selection on $\sens$;
	\item Some element of $G^\circ$ is a complete selection on $\sens$.
\end{enumerate}
\end{proposition}
\begin{proof}
See \Cref{proof:when maximal coherent is complete}.\qed
\end{proof}
The impact of this result on our representation problem is twofold. First, it provides a clear and easily verifiable criterion for when the dual space of a PCR consists (only!) of possible worlds. Second, it introduces a new and consistent notion of a query of low import, not involving arbitrary choices such as thresholding.
\begin{proposition}\label[proposition]{prop:function form of dual} Let $G$ be a non-degenerate PCR over the PCS $\sens$.
Then the mapping $\chi\colon\morph{PCR}{G}{\power{}}\to G^\circ$ defined by $\chi(f)=f\inv(\mbf{1})$ is a bijection.
\end{proposition}
\begin{proof}
See \Cref{proof:function form of dual}.\qed
\end{proof}
\begin{remark} Note that the mapping $\chi$ is independent of the choice of $G$.
\end{remark}
The last proposition explains the sense in which $G^\circ$ may be thought of as a {\em dual space} of $G$. As with other instances of duality, this is useful because it enables dual mappings: 
\begin{definition}\label[definition]{defn:dual map} Let $f:G_1\to G_2$ be a PCR morphism. The {\em dual mapping} $f^\circ:G_2^\circ\to G_1^\circ$ is defined by $f^\circ(S)=f\inv(S)$. Alternatively, upon applying the identification in \Cref{prop:function form of dual}, for any $\varphi\in\morph{PCR}{G_2}{\power{}}$, one has $f^\circ(\varphi)=\varphi\circ f$ to obtain an element of $\morph{PCR}{G_1}{\power{}}$.
\qed\end{definition}
We remark that, since morphisms are composable (meaning that the composition $(f\circ g)(a):=f(g(a))$ of two morphisms is a morphism as well), so are their dual mappings, producing the identity $(f\circ g)^\circ=g^\circ\circ f^\circ$.

\begin{example}\label[example]{ex:embedding of duals} Let $G$ be a non-degenerate PCR over a PCS $\sens$. Then it is clear that the identity mapping $\iota:\orth{\sens}\to G$~ ---~ that is: $\iota(a)=a$ for all $a\in\sens$~ ---~ is a morphism of PCRs. The dual mapping $\iota\circ:G^\circ\to\ham(\sens)$ is then, clearly, an injection. This reflects the intuitive notion that the dual of any (non-degenerate) PCR may be ``excavated'' out of a standard Hamming cube by going over all $G$-incoherent pairs, one by one, and successively deleting any vertices of $\ham(\sens)$ which contain the given pair.
\end{example}
We further specialize the example to our representation problem, considering the effect of fixing a PCR structure on a given PCS:
\begin{proposition}[Universality of Representation]\label[proposition]{prop:universality} Let $G$ be a non-degenerate PCR over $\sens$. Then, for any non-empty set $\spc$ and every PCS morphism $\rho\colon\sens\to\power{\spc}$, the set $\model(\rho)$ of all complete $\ast$-selections witnessed (via $\rho$) by a point in $\spc$ (in the sense of \Cref{models:first notions}) is contained in $G^\circ$ whenever $\rho$ is a PCR morphism. Moreover, $G^\circ$ is the smallest subset of $\ham(\sens)$ having this property.
\end{proposition}
\begin{proof}
See \Cref{proof:universality}.\qed
\end{proof}
Thus, the dual $G^\circ$ of a non-degenerate $G$ serves as a minimal model of the state space of the system $(agent+environment)$, and remains valid under {\em any} change to this system for as long as $\rho$ remains order-preserving. This is a form of robustness of the representation to changes in the coupling between the agent's sensory equipment and the environment: changes leaving the implication record invariant provide no reason for the agent to alter its reasoning.

\subsection{Reducing PCR Representations.}\label[subsection]{models:PCRs to poc sets}
The universality of PCR duals motivates a deeper study of their properties, seeking a better understanding of the degree of redundancy in the description of $G^\circ$ by a PCR $G$.  
This is not a mere technical issue:
while non-degeneracy guarantees the adequacy of our notion of an associated ``possible world'', it is not obvious that it also provides for sufficient control over the quality of inference.
The intended application---inferring {\em approximate} implications from partial observations---is well known to be problematic in the absence of simplifying assumptions (e.g. the ubiquitous restriction to directed acyclic graphs in the context of Bayesian networks).
It is therefore crucial to clarify the precise formal sense in which a PCR may be viewed as encoding a ``record of implications'', which is the purpose of this section.
A crucial notion in any such discussion is that of what it means for a query, as well as for the difference of two queries, to be negligible, because negligible but non-zero differences tend to accumulate in the transitive closure into material ones.

Looking more closely at the setting of the last proposition, notice that, for a fixed $\rho$, the assumption that $\rho$ is a morphism translates into the following. The property $a\leq b\THEN\rho(a)\subseteq\rho(b)$ for all $a,b\in\sens$ implies $\rho(a)=\varnothing$ for any $a\in N(G)$ (because $\varnothing$ is the only negligible element of $\power{\spc}$); furthermore, $\rho(a)=\rho(b)$ must hold whenever $a$ and $b$ are $G$-equivalent (recall \Cref{PCR notation}). These identifications lead us to recall Roller's definition of a {\em poc set} from~\cite{Roller-duality}:
\begin{definition}[poc set]\label[definition]{defn:poc set} A poc set is a tuple $\ppoc=(\sens,\leq,\mbf{0},\ast)$ where $(\sens,\leq)$ is a partially ordered set with a minimum element $\mbf{0}$, endowed with an order-reversing involution\footnote{That is, $a\com\com=a$ and $a\leq b\THEN b\com\leq a\com$ for all $a,b\in\sens$.} $a\mapsto a\com$ satisfying $\mbf{0}\com\neq\mbf{0}$ and $a\leq a\com\THEN a=\mbf{0}$
for all $a\in\sens$.\qed
\end{definition}
In other words, a poc set is a transitive and anti-symmetric PCR over $\sens$ whose only negligible element is $\mbf{0}$.

\begin{proposition}[canonical quotient]\label[proposition]{prop:canonical quotient} For any non-degenerate PCR $G$ there exists a surjective PCR morphism $\pi_G:G\to\widehat{G}$ of $G$ onto a poc set $\widehat{G}$ such that any PCR morphism $f:G\to\ppoc$ gives rise to one and only one PCR morphism $\widehat{f}:\widehat{G}\to\ppoc$ satisfying $f=\widehat{f}\circ\pi$.
\end{proposition}
\begin{proof} We defer the proof to \Cref{proof:canonical quotient}, but define the canonical quotient mapping here. We set:
\begin{equation}\label{eq:equivalence classes in G}
	\pi_G(a):=\left\{\begin{array}{ll}
		[a]_G		&\text{if }a\notin N(G)\cup N(G)\com\\
		N(G)	&\text{if }a\in N(G)\\
		N(G)\com&\text{if }a\in N(G)\com		
	\end{array}\right.
\end{equation}
and let $\widehat{\sens}:=\set{\pi_G(a)}{a\in\sens}$, and setting $\pi_G(a)\leq\pi_G(b)$ to hold in $\widehat{G}$ if and only if $ab\in G$. It remains to verify that (1) $\widehat{\sens}$ is a well-defined PCS; (2) $\widehat{G}$ is a well-defined poc set structure over $\widehat{\sens}$; and (3) the assertions of the proposition hold.\qed
\end{proof}
One should view this result as stating the precise conditions necessary for presenting a poc set in terms of a set of generators and a set of relations. However, the emphasis on what happens to morphisms leads to powerful realizations about dual spaces:
\begin{corollary}[all duals are poc set duals]\label[corollary]{cor:all duals are poc set duals} If $G$ is a non-degenerate PCR then $\pi_G^\circ:\widehat{G}^\circ\to G^\circ$ is a bijection.
\end{corollary}
\begin{proof}
See \Cref{proof:all duals are poc set duals}.\qed
\end{proof}
\begin{corollary}[naturality of canonical quotients]\label[corollary]{cor:naturality of canonical quotients} Let $G,H$ be non-degenerate PCRs. Then, for every morphism $f:G\to H$ there exists one and only one morphism $\hat{f}:\widehat{G}\to\widehat{H}$ satisfying $\pi_H\circ f=\widehat{f}\circ\pi_G$.
\end{corollary}
\begin{proof}
See \Cref{proof:naturality of quotients}.\qed
\end{proof}
A particular consequence of the last corollary is that one also has $\pi_H^\circ\circ\widehat{f}^{\,\circ}=f^\circ\circ\pi_G^\circ$. This means the dual maps of $f$ and $\hat{f}$ coincide up to the identifications between the pre- and post-projection duals. Thus, any results about poc set duals apply to duals of PCRs. In the next two sections we review these results, and then harness them in our construction of the {\em universal memory architecture} (UMA).

\subsection{Convexity theory of PCR duals.}\label[subsection]{models:convexity}
To discuss the geometry of PCR duals, we need to endow PCRs with more structure.
\textsc{From this point on, all PCRS we consider will be finite, with the sole possible exception of power sets.}

\begin{definition}[Hamming metric] Let $G$ be a PCR over $\sens$. The {\em Hamming metric} on $G^\circ$ is defined by $\ellone{u}{w}=\left|\pi_G(u)\minus\pi_G(w)\right|$, where $\pi_G:G\to\widehat{G}$ is the canonical quotient map.
We define $\dual{G}$ to be the simple\footnote{That is: loopless, unoriented, with no multiple edges.} graph with vertex set $G^\circ$, and edges of the form $\{u,w\}$ for all $u,w\in G^\circ$ with $\ellone{u}{w}=1$.\qed
\end{definition}
In the case when $G$ is already a poc set, two vertices $u,w\in\ppoc^\circ$ form an edge if and only if $u\minus w$ is a singleton, that is: the perceptual classes represented by $u$ and $w$ differ by the truth value of a single query.
The common edge they span in the Hamming cube $\sel(\sens)$ corresponds to the $\ast$-selection $u\cap w$ in the concept presentation.
In the general case ($G$ not necessarily a poc set), since both $u$ and $w$ are coherent, each is the union of $N(G)\com$ with a number of $G$-equivalence classes $[a]$, $a\notin N(G)$ (recall \Cref{PCR notation} and \Cref{prop:canonical quotient}).
Thus $u$ and $w$ span an edge in $G^\circ$ if and only if $u\minus v=[a]$ for some $a\notin N(G)\cup N(G)\com$.
Intuitively, we think of the different $b\in[a]$ as counting for a single Boolean query.

We briefly recall the graph-theoretic notion of convexity:
\begin{definition}[convexity in graphs]\label[definition]{defn:convexity in graphs} Let $\Gamma=(V,E)$ be a graph and let $u,v\in V$.
The {\em hop distance} $d_\Gamma(u,v)$ is defined to be the minimum length of an edge-path in $\Gamma$ joining $u$ with $v$.
The {\em interval} $I(u,v)$ is defined to be the set of all vertices $w\in V$ satisfying the equality $d_\Gamma(u,v)=d_\Gamma(u,w)+d_\Gamma(w,v)$.
A set $C\subseteq V$ is said to be {\em convex} in $\Gamma$, if $I(u,v)\subseteq C$ holds for all $u,v\in C$.
A set $H\subseteq V$ is a {\em half-space of $\Gamma$}, if both $H$ and $V\minus H$ are convex sets in $G$.
Finally, we denote by $\mathcal{H}(\Gamma)$ the poc set whose elements are the half-spaces of $\Gamma$ (note that $\varnothing$ is a half-space of $\Gamma$), ordered by inclusion, and with $H\com:=V\minus H$.\qed
\end{definition}
We refer the reader to \cite{Nica-spaces_with_walls}, section 4, for the (very elegant and much more general) proofs of the following two lemmas (stated there for poc sets, but valid for finite non-degenerate PCRs as well, due to \Cref{prop:canonical quotient} and its two corollaries):
\begin{lemma}\label[lemma]{lemma:ellone metric is hop metric} Let $G$ be a finite non-degenerate PCR. Then the hop metric on $\dual{G}$ coincides with the metric $\mbf{\Delta}$.\qed
\end{lemma}
\begin{lemma}\label[lemma]{lemma:ellone halfspaces} Let $G$ be a finite non-degenerate PCR. Then the half-spaces of $\dual{G}$ are precisely the subsets of $G^\circ$ of the form\footnote{Note that $\half{a^\ast;G}=G^\circ\minus\half{a}$ for all $a\in G$, by \Cref{prop:when maximal coherent is complete}.} \begin{equation}
	\half{a;G}:=\set{u\in G^\circ}{a\in u}\,,\quad a\in G\,.
\end{equation}
In particular, subsets of $G^\circ$ of the form
\begin{equation}
	\half{S;G}:=\set{u\in G^\circ}{S\subseteq u}=\bigcap_{a\in S}\half{a;G}\,,\quad S\subset\sens
\end{equation}
are convex in $\dual{G}$, for any $S\subseteq G$.\qed
\end{lemma}
\begin{definition}\label[definition]{defn:halfspaces abused notation} To simplify notation, we will abuse it in the following ways:
\begin{itemize}
	\item Writing $\half{a}$, $\half{S}$ without specifying $G$ will henceforth refer to the subsets of $\ham(\sens)$, those are $\half{a;\orth{\sens}}$ and $\half{S;\orth{\sens}}$, respectively.
    \item When $S$ is explicitly known, $S=\{a_1,\ldots,a_k\}$, we will write $\half{a_1\cdots a_k;G}$ instead of $\half{S;G}$ when convenient.
\end{itemize}
\end{definition}
As a side note, observe that $\half{S;G}=\half{S}\cap G^\circ$, where $\half{S}$ coincides with the vertex set of a face of the hamming cube $\ham(\sens)$.
In particular, presenting any subset of $G^\circ$ as a concept is equivalent to decomposing it as a union of convex subsets of $G^\circ$.

\paragraph{Median Graphs.} The two preceding lemmas are results of $\dual{G}$ being a {\it median graph}~\cite{Chepoi-median,median1}:
\begin{definition}\label[definition]{defn:median graph} A connected simple graph $\Gamma=(V,E)$ is said to be a \emph{median graph}, if the set $I(u,v)\cap I(v,w)\cap I(u,w)$ contains exactly one vertex for each $u,v,w\in V$. This vertex is the \emph{median} of the triple $(u,v,w)$ and denoted by $\med{u}{v}{w}$ -- see \Cref{fig:median}.
For median graphs $\Gamma_i=(V_i,E_i)$, $i=1,2$, a {\em median morphism} of $\Gamma_1$ to $\Gamma_2$ is a map $f\colon V_1\to V_2$ which preserves medians: $f(\med{u}{v}{w})=\med{f(u)}{f(v)}{f(w)}$.
\qed
\end{definition}
Median graphs are a special subfamily of {\it median algebras}, \cite{median4,median5,median3,median2}. Some modern generalizations and applications may be found in \cite{CDH-median_spaces_measured_walls}.
\begin{figure}[t]
	\begin{center}
		\includegraphics[width=.3\columnwidth]{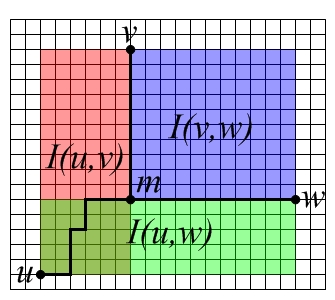}
	\end{center}
	\caption{
	Computing a median in a rectangle $\Gamma$ cut out of the integer grid (all vertices of the form $m\times n$, $m,n\in\ZZ$, with edges joining a vertex $m\times n$ to the vertices $(m\pm 1)\times n$ and $m\times(n\pm 1)$).
	\label{fig:median}}
\end{figure}
A central result in Sageev-Roller duality, specialized here to the finite case, and reformulated for non-degenerate PCRs is:
\begin{theorem}\label{thm:duality and median formula} The dual $\dual{G}$ of a finite non-degenerate PCR $G$ is a finite median graph, with the median calculated according to the formula:
\begin{equation}
	\med{u}{v}{w}=(u\cap v)\cup(u\cap w)\cup(v\cap w)\,,\quad u,v,w\in G^\circ\,,
\end{equation}
and with intervals in $\dual{G}$ calculated according to the formula:
\begin{equation}
	I(u,v)=\set{w\in G^\circ}{u\cap v\subseteq w}=\set{w\in G^\circ}{w\subseteq u\cup v}\,.
\end{equation}
Conversely, if $\Gamma$ is a finite median graph then $\Gamma$ is naturally isomorphic to $\dual{\mathcal{H}(\Gamma)}$ by sending every vertex $v$ to the $\ast$-selection of all half-spaces of $\Gamma$ which contain $v$.\qed
\end{theorem}
This result is the consequence of a very strong convexity theory:
\begin{theorem}[Properties of median graphs, \cite{Roller-duality}, section 2]\label{thm:convexity theory of median graphs}
Let $\Gamma=(V,E)$ be a finite median graph. Then:
\begin{enumerate}
	\item Any family of pairwise intersecting convex sets has a common vertex;
	\item Every convex set is an intersection of halfspaces;
	\item For any convex subset $K\subset V$, the subgraph of $\Gamma$ induced by $K$ is a median graph;
	\item For any convex $K\subset V$ and any $v\in V\minus K$ there is a unique vertex $\proj{K}(v)\in K$ at minimum hop distance from $v$;
	\item For any convex $K\subset V$, the nearest point projection $\proj{K}(\wild)$ is a median preserving, distance non-increasing retraction of $\Gamma$ onto its subgraph induced by $K$.
\end{enumerate}
Property (1) is often referred to as the {\em Helly property}.\qed
\end{theorem}
The Helly property is, perhaps, the most notable of the results stated above.
In our setting of PCR duals, it may be interpreted as guaranteeing the satisfiability of any family of conjunctive monomials over $\alphabet$ in which every pair is separately satisfiable.

\paragraph{Convex hulls.}
Given the central role of half-spaces in the convexity theory of median graphs, a notion of the set of half-spaces dual to a given set of vertices is useful:
\begin{definition} For $K\subset\ham$, its dual set of halfspaces, $K\supp$, is defined to be the set of all $a\in \sens$ with $K\subseteq\half{a}$.\qed
\end{definition}
An immediate corollary of \Cref{thm:convexity theory of median graphs}(2) is:
\begin{corollary}\label[corollary]{cor:convex hull formula} Suppose $G$ is a non-degenerate PCR, and $K\subseteq G^\circ$. Then $K\supp$ is $G$-coherent and forward-closed, and the convex hull $\hull{K}{G}$ of $K$ in $\dual{G}$ coincides with $\half{K\supp;G}$.\qed
\end{corollary}
Thus, every convex subset of $G^\circ$ may be written as $\half{S;G}$ for some $S\in\mbf{C}(G)$. This representation is unique, by last assertion of the following lemma:
\begin{lemma}\label[lemma]{lemma:halfspace properties} Let $G$ be a non-degenerate PCR over $\sens$.
Then, for all $S,S_1,S_2\subset\sens$:
\begin{enumerate}
	\item $\half{S;G}\neq\varnothing$ if and only if $S$ is coherent;
	\item For all $S_1,S_2\subseteq\sens$ one has $\half{S_1\cup S_2;G}=\half{S_1;G}\cap\half{S_2;G}$;
	\item If $S_2\subseteq S_1$ then $\half{S_1;G}\subseteq\half{S_2;G}$;
	\item If $S$ is coherent then $\half{S;G}=\half{\up{S};G}$;
    \item If $S\in\mbf{C}(G)$, then $\half{S;G}=\half{\min(S);G}$
	\item For all $S_1,S_2\in\mbf{C}(G)$ one has $\half{S_1;G}\subseteq\half{S_2;G}\IFF S_2\subseteq S_1$.
\end{enumerate}
\end{lemma}
\begin{proof}
See \Cref{proof:halfspace properties}.\qed
\end{proof}
Another important result helps bound the distance from the points of one convex set to another:
\begin{lemma}\label[lemma]{lemma:divergence} Let $S,T\in\mbf{C}(G)$ for a poc set $G$ over $\sens$. Then $\ellone{u}{\half{T}}\leq\card{T\minus S}$ for all $u\in\half{S}$.
\end{lemma}
\begin{proof}
See \Cref{proof:divergence}.\qed
\end{proof}
This motivates the following definition for the general case:
\begin{definition}\label[definition]{defn:divergence} Let $G$ be a non-degenerate PCR over $\sens$ and let $S,T\in\mbf{C}(G)$. The {\em divergence} of $S$ from $T$ is defined to be $\divergence{S;T}:=\card{T\minus S}$.\qed
\end{definition}
Note how $\divergence{S;T}$ seems independent of $G$; it is not, however, since it is only applied to upwards-closed coherent sets $S,T$.
We will use this notion of divergence in \Cref{sim:BUAs}, to drive the decision-making mechanism of the {\em binary UMA agents} briefly introduced there.

More details about the convexity theory of a median graph will be discussed in the appendices, as we go about proving our algorithmic results.

\subsection{Propagation: A Computational Workhorse. }\label[subsection]{models:propagation}
We are now ready to present another central result of this paper: a low-complexity method for computing nearest point projections in $\dual{G}$, which we call {\em propagation}.
This method obviates the need for maintaining an explicit representation of each vertex of $\dual{G}$ in memory, reducing space requirements for this architecture from $O(2^{\card{\sens}})$ in the worst case to $O(\card{\sens}^2)$.
The time complexity is, at worst, $O(\card{\sens}^2)$, coming down to sub-linear on a fully parallel architecture, as will become evident below.

\begin{definition}[coherent projection]\label[definition]{defn:coherent projection} Let $G$ be a PCR over a finite PCS $\sens$. For any $T\subseteq\sens$, the set $\coh{G}{T}:=\up{T}\minus\down{T\com}=\up{T}\minus(\up{T})\com$ is said to be the {\em $G$-coherent projection} of $T$.\qed
\end{definition}
Coherent projection itself plays an important role in obtaining an observer's belief state from its epistemic state (the learned PCR structure) and the latest observation (see~\Cref{high-level:belief state}).

The promised formula for computing projections works as follows.
\begin{proposition}\label[proposition]{prop:projection formula} Let $G$ be a PCR over a finite PCS $\sens$.  Let $S,T\subset\sens$ and suppose $S$ is $G$-coherent. Let $L=\half{S;G}$ and $K=\half{\coh{G}{T};G}$. Then:
\begin{equation}
	\proj{K}(L)=\up{(S\cup T)}\minus(\up{T})\com=(\up{S}\minus\up{T}\com)\cup\coh{G}{T}\,,
\end{equation}
where $\proj{K}(\wild)$ is the nearest-point projection to $K$ in $\dual{G}$ defined in \Cref{thm:convexity theory of median graphs}.
\end{proposition}
\begin{proof}
See \Cref{proof:projection formula}.\qed
\end{proof}
This description of nearest point projection is easy to visualize as being computed by an algorithm propagating excitation among nodes of a directed graph:
\begin{definition}\label[definition]{defn:loaded PCR} Let $G$ be a PCR over a finite PCS $\sens$.
Let $S\subset\sens$. Denote by $[G,S]$ the graph with vertex set $\sens$, edge set $G$ and with Boolean weights $\lambda(a)=\indicator{S}(a)$, $a\in\sens$ attached to its vertices.
We refer to it as {\em $G$ being loaded with $S$}.\qed
\end{definition}
\begin{definition}\label[definition]{defn:propagation algorithm} A \emph{propagation algorithm over $G$} is any algorithm which, for any {\em $G$-coherent} load $S\subset\sens$ and \emph{any} $T\subseteq\sens$ accepts $[G,S]$ and $T$ as input and produces as its output the loaded graph $[G,\propagate{T,S}{G}]$, where
\begin{equation}
	\propagate{T,S}{G}:=\up{(S\cup T)}\minus\down{T\com}\,.
\end{equation}
Note that coherent closure is obtainable via $\coh{G}{T}=\propagate{T,\varnothing}{G}$.\qed
\end{definition}
Envisioning $G$ as describing a graph of `cells' labeled by $\sens$ and `synapses' labeled by pairs $ab\in G$, the loaded graph $[G,S]$ represents a state of the network indicating that the cells of $S$ are in an excited state.
A propagation algorithm should be seen as exciting, additionally, the cells of $T$ and spreading this excitation along the directed connections while inhibiting $a\com$ for each cell $a\in\sens$ encountered along the way.
Realized on a modern day computer, this may be achieved in quadratic time in $\card{\sens}$.
For example, propagation could be implemented using a variant of depth-first search (DFS) on $\pog\at{t}$, while maintaining an expanding record of vertices visited~\cite{Cormen_et_al-intro_to_algorithms}---see \Cref{alg:propagation_by_DFS}.
On a fully parallel machine allowing the `cells' to compute their own excitation, the time complexity is clearly of the order of the longest directed vertex path in the network, which is sub-linear in $\card{\sens}$.

\begin{algorithm}[t]
\caption{Propagating a signal $T$ over $[G,S]$ using depth-first search.}
\label{alg:propagation_by_DFS}
\begin{algorithmic}
\Function{propagate}{$G,S,T$}
	\State $\mathtt{visited}\gets\varnothing$
	\State $T\gets$\Call{closure}{$G,T$}
    \State $S\gets$\Call{closure}{$G,S$}
	\State \Return $(S\cup T)\minus T^\ast$
\EndFunction

\Function{closure}{$G,T$}\Comment{Forward closure of $T$ in $G$}
	\ForAll{$a\in T$}
		\State\Call{explore}{$G,a$}
	\EndFor
	\State \Return $\mathtt{visited}$
\EndFunction

\Procedure{explore}{$G,v$}\Comment{Recursive step}
	\State $\mathtt{visited}\gets\mathtt{visited}\cup\{v\}$
	\ForAll{$w\in$\Call{children}{$G,v$}$\minus\mathtt{visited}$}
		\State\Call{explore}{$G,w$}
	\EndFor
\EndProcedure

\Function{children}{$G,v$}\Comment{Children of $v$ in $G$}
	\State \Return $\set{w\in\sens}{vw\in G}$
\EndFunction 
\end{algorithmic} 
\end{algorithm}

We now turn to a high-level description of the UMA architecture and its use of the results of this section.

\section{Universal Memory Architecture (UMA): a High-Level View.}\label{high-level}
In this section we provide a high-level description of the basic UMA functionalities: PCR update/revision and maintaining a belief state.

\subsection{Observation Model.}\label{high-level:observation model}
Recall from \Cref{models:first notions} that an observer is given a set $\alphabet$ of initial Boolean queries over the space of histories $\spc$ of the observed system.
The system of queries and their complements is modeled as a PCS morphism $\rho:\sens(\alphabet)\to\power{\spc}$, which is unknown to the observer.
The observer is presented with a sequence of observations $\observe{}\at{t}\in\ham(\sens(\alphabet))$, and values $\val{}\at{t}\in\RRplus$, $t\geq 0$, one per update cycle.
One must distinguish between two settings:
\begin{description}
	\item[\bf Static signal.] The value signal $\val{}\at{t}$ only depends on the raw observation $\observe{}\at{t}$;
    \item[\bf Dynamic signal.] The value signal may produce $\val{}\at{t}\neq\val{}\at{s}$ while $\observe{}\at{t}=\observe{}\at{s}$. 
\end{description}
While ultimately interested in covering the dynamic setting, we will only deal with the static setting in this paper.
However, the setting being {\em static} by no means implies it is {\em unchanging}.
We will see in \Cref{simulations} that instances of the static setting may, nevertheless, have rich and interesting dynamics.
This will happen, in part, as a result of introducing {\em delayed queries}.
By these we mean the following:
if $\delay:\spc\to\spc$ denotes the operation of truncating the last state from a given history, then, for any conjunction $a$ of already available queries it is possible to introduce a new query of the form\footnote{Here and on we abuse notation, applying the symbol $\delay$ to denote both a delayed query and the history truncation/shift operator. Which is which is clear from the context.} $\delay a$, where $\delay a$ reports its value according to the rule $x\in\rho(\delay a)\IFF\delay x\in\rho(a)$, $x\in\spc$.
Of course, implementing this operation requires that the UMA architecture retain the latest raw observation, but this seems like a small price to pay for increasing the range of application of the static setting.

The basic task of an UMA is to evolve a sequence $G\at{t}$, $t\geq 0$ of non-degenerate PCRs over $\sens=\sens(\alphabet)$ while aiming for the PCRs $G\at{t}$ to eventually satisfy the following:
\begin{itemize}
	\item{\bf `Completeness': } $\rho:G\at{t}\to\power{\spc}$ is a PCR morphism, ensuring that every perceptual class is represented;
    \item{\bf `Precision': } $\model\at{t}:=\dual{G\at{t}}$ is as close as possible to the true model space $\model:=\rho\com(\spc)$.
\end{itemize}
These requirements should not be taken literally, however. For example, it stands to reason that in some contexts the observer could afford to misclassify a few perceptual classes of low import.
We will see how---at least under some of the learning schemes we propose---these vague requirements become possible to state precisely in terms of PAC learning.

\subsection{Maintaining a PCR presentation: Snapshot Structures.}\label{high-level:snapshots}
A rather restrictive notion of a {\em snapshot structure}---a method for learning a poc set structure from positive observations---was introduced by the authors in~\cite{GK-Allerton_2012}.
Here we merely review the main ideas to provide intuition, while deferring the formal constructions to \Cref{snapshots}.

\paragraph{Snapshot weights.} Motivated loosely by Hebbian ideas about learning~\cite{hebb2002organization}, we consider maintaining an evolving symmetric system of weights $\witness{}{\wild}\at{t}=(\witness{}{ab}\at{t})_{a,b\in\sens\at{t}}$, with $\witness{}{ab}\at{t}$ quantifying in some prescribed way a notion of {\em cumulative degree of relevance} of the event $\half{ab}$ to the observer, at time $t$.

In addition, rules to maintain $\witness{}{\wild}\at{t}$ as time progresses, must be provided.
First, a completion rule, to insert missing values into $\witness{}{\wild}$ when it undergoes an extension.
Second, an update rule, computing $\witness{}{\wild}\at{t+1}$ from $\witness{}{\wild}\at{t}$ and the incoming observation.

It is important for both rules to be as simple---and as local---as possible, so as not to sacrifice tractability.
In our constructions, we constrain the update laws to ones where $\witness{}{ab}\at{t+1}$ depends only on $\witness{}{ab}\at{t}$, the value signal $\val{}\at{t+1}$, the truth value of the bit $\observe{}\at{t+1}\in\half{ab}$ and possible global parameters (e.g. the system clock $t$).

\paragraph{PCRs from snapshot weights.} Inspired by the rough mechanism proposed in~\cite{GK-Allerton_2012}, we seek weight systems $\witness{}{\wild}$ for which the loosely specified rule---
\begin{equation}
    ab\in G
    \;\Longleftrightarrow\;
    \witness{}{ab\com}
    \text{ is negligible in comparison with }
    \witness{}{ab},
    \witness{}{a\com b\com}\,,
\end{equation}
ranging over all $a,b\in\sens$ with $\{a,a\com\}\neq\{b,b\com\}$ is guaranteed to define a {\em non-degenerate} PCR over $\sens$. 
The motivation for the rule is, of course, the fact that $\rho(a)\cap\rho(b\com)=\varnothing$ is equivalent to $\rho(a)\subseteq\rho(b)$, where $\rho$ is the PCS morphism defining the semantics of the queries in $\sens$.

Finally, note how the properties of a PCR are guaranteed (to the extent that the rule is well-defined, of course), and non-degeneracy is the only remaining question. Of course, the precise notion of `negligible' defined for the purpose of comparing weights is crucial, and is expected to greatly affect the quality and limitations of the emerging representations.

\subsection{Maintaining a Belief State.}\label{high-level:belief state}
Since, for each time $t$, we only get to observe states from $\model(\rho)$, we are facing the problem of having to learn negative statements---that is, the list of $G\at{t}$-incoherent pairs---from the stream of positive examples $(\observe{}\at{t},\val{}\at{t})$.
From what we have observed so far we must reason about what it is we might never encounter. 
Seeing that the implication record $G\at{t}$ is inherently uncertain, providing no guarantee at any time that the completeness requirement from \Cref{high-level:observation model} will be met, it is quite possible for the observation $\observe{}\at{t+1}$ to land outside the model space $\model{}\at{t+1}=\dual{G\at{t+1}}$ despite its prior role in forming this model space, during the snapshot update.
In fact, its value may be too low to trigger a revision of $G\at{t}$ into a $G\at{t+1}$ for which $\observe{}\at{t+1}$ becomes coherent.

Contrary to the approach adopted by modern iterated revision schemes based on Darwiche and Pearl's~\cite{Darwiche_Pearl-iterated_belief_revision}, we {\em do not insist} on a revision forcing $\observe{}\at{t+1}$ into $\model\at{t+1}$.
Instead, we apply $G=G\at{t+1}$ to the raw observation with aim to relax it, replacing it with a $G$-coherent and forward-closed set:
\begin{equation}\label{eqn:current state update}
	\current{}\at{t+1}:=\coh{G}{\observe{}\at{t+1}}\subseteq\observe{}\at{t+1}\,,
\end{equation}
in the role of the current state of record, or the belief state.
This way, UMA naturally resolves possible contradictions at the price of introducing ambiguity into its record of the current state:
instead of marking a single vertex of $\model\at{t+1}$ as the current state, any vertex of the convex set $\half{\current{t+1}}$ may turn out to be the correct current state from the observer's point of view.

The choice of the coherent projection for the purpose of forming the belief state is motivated by its geometric and categorical properties.
In our class of model spaces it is a {\em canonical} method of producing coherent sets, as witnessed by the following two results:
\begin{proposition}[Coherent Approximation]\label[proposition]{prop:coherent approximation} Let $G$ be a PCR over $\sens$. Then, for any $A\in \ham(\sens)$, if $B\in G^\circ$ realizes the Hamming distance 
\begin{equation}
	\ellone{A}{G^\circ}:=\min_{u\in G^\circ}\ellone{A}{u}
\end{equation}
---that is, if $\ellone{A}{B}=\ellone{A}{G^\circ}$---then we must have $B\in\half{\coh{G}{A};G}$.
\end{proposition}
\begin{proof}
See \Cref{proof:coherent approximation}.\qed
\end{proof}
Thus, the operation $\coh{\wild}$ yields the ``best approximation'' of $\observe{}\at{t+1}$ by a convex subset of $\model\at{t+1}$, echoing the principle of minimal change as seen through Dalal's way~\cite{Dalal-theory_of_KB_revision} of quantifying the distance between theories. Moreover:
\begin{proposition}[Coherent Projection]\label[proposition]{prop:coherent projection} Let $G$ be a PCR over $\sens$. Then the following hold for all $A\subseteq\sens$:
\begin{itemize}
	\item{\rm(a) } $\coh{G}{A}$ is coherent and $\up{\coh{G}{A}}=\coh{G}{A}$;
	\item{\rm(b) } $\coh{G}{\coh{G}{A}}=\coh{G}{A}$;
	\item{\rm(c) } $A\subseteq\coh{G}{A}$ whenever $A$ is $G$-coherent;
	\item{\rm(d) } $\coh{G}{A}=A$ if and only if $A$ is $G$-coherent and $\up{A}=A$.
\end{itemize}
In other words, as a self-map of $\power{\sens}$, the operator $A\mapsto\coh{G}{A}$ is an idempotent whose image coincides with $\mbf{C}(G)$.
\end{proposition}
\begin{proof}
See \Cref{proof:coherent projection}.\qed
\end{proof}
Note how properties (a) and (c) turn $\coh{G}{\wild}$ into a closure operator on the subspace of $G$-coherent sets with respect to inference (implication). At the same time, (b) and (d) characterize the set $\mbf{C}(G)$ of all terms that are closed under inference.

Overall, \Cref{eqn:current state update} provides an intriguingly natural way of maintaining an internal model and belief state with a built-in degree of resilience to observations that fail to make immediate sense to the agent given its epistemic state.
Finally, the complexity of this computation is the complexity of propagation over $G\at{t+1}$, by \Cref{prop:projection formula} and the discussion following \Cref{defn:propagation algorithm}.

\section{Learning Algorithms for UMAs: Snapshot Structures.}\label{snapshots}

\subsection{Qualitative Snapshot Structures.}\label{snapshots:qualitative}
The goal of this section is to construct a snapshot structure suitable for a scenario in which the learner's value signal is a ranking function in the sense of Pearl~\cite{Pearl-system_Z,Goldszmidt_Pearl-qualitative_probas_for_belief_revision} (which is a special form of Spohn's OCFs~\cite{Spohn-OCF_epistemic_state}), its values providing a qualitative notion of the degree of irrelevance of the current experience.
Thus, an observation with $\val{}\at{t}=0$ is considered desirable, while $\val{}\at{t}=1,2,\ldots$ renders an observation increasingly more irrelevant.

\subsubsection{Rankings and 2-rankings}\label{snapshot:qualitative:rankings}
Throughout this section we let $\sens$ be a PCS and let $\ham$ denote the Hamming cube $O(\sens)^\circ$. Also, let $\NNH:=\{0,1,2,\ldots,\infty\}$.
We use the slight variation of the notion of a ranking from~\cite{Pearl-system_Z}, which was introduced in~\cite{Darwiche_Pearl-iterated_belief_revision}:
\begin{definition}\label[definition]{defn:ranking} A {\em ranking on $\ham$} is a function $\kappa:\power{\ham}\to\NNH$, satisfying:
\begin{itemize}
	\item $\kappa(F)=\min_{\sigma\in F}\kappa(\sigma)$ for all $F\subset\ham$;
	\item $\kappa(\sigma)<\infty$ for some $\sigma\in\ham$;
	\item $\kappa(\varnothing)=\infty$.
\end{itemize}
Hereafter, we shall abuse notation, writing $\kappa(\sigma)$ to mean $\kappa(\{\sigma\})$ whenever $\sigma\in\ham$. Note that the minimum value of a ranking $\kappa$ is $\kappa(\ham)$.
\qed\end{definition}
\begin{remark} Note that, since $\sens$ is assumed to be finite, the first requirement may be replaced with the requirement that $\kappa(F_1\cup F_2)=\min\{\kappa(F_1),\kappa(F_2)\}$ for all $F_1,F_2\subset\ham$.
\end{remark}

The simplest examples of rankings seem to be:
\begin{example}[point-mass ranking]\label[example]{ex:point-mass ranking} Let $u\in\ham$ and $r\in\NNH$, $r\neq\infty$. Then the following function $\pointmass{u}{r}$ is a ranking:
\begin{equation}\label{eqn:point-mass}
	\pointmass{u}{r}(F):=\left\{\begin{array}{cl}
    	r &\text{if }u\in F\\
        \infty &u\notin F\,.
    \end{array}\right.
\end{equation}
\end{example}
\begin{example}[pointwise minimum]\label[example]{ex:minimum} If $\kappa_1,\kappa_2$ are rankings on $\NNH$, then the function $\kappa(F):=\min(\kappa_1(F),\kappa_2(F))$ is also a ranking.\qed
\end{example}
Recall now the sets $\half{S}$ from \Cref{lemma:ellone halfspaces}. They will help us study the interaction between rankings and concepts:
\begin{definition}\label[definition]{defn:concept representation} The {\em concept representation} of a ranking $\kappa$, is the function $\witness{\kappa}{S}:=\kappa\left(\half{S}\right)$, where $S$ ranges over subsets of $\sens$. To simplify notation, we will often write $\witness{\kappa}{a_1a_2\cdots a_k}:=\witness{\kappa}{S}$ whenever $S=\{a_1,\ldots,a_k\}$ is explicitly provided.
\qed\end{definition}
\begin{remark} Note that $\witness{\kappa}{S}=\infty$ if $S$ is not a $\ast$-selection. Also, $\witness{\kappa}{\varnothing}=\kappa(\ham)$, the minimum value of $\kappa$. 
\end{remark}
\begin{lemma}[triangle inequality]\label[lemma]{lemma:triangle inequality} For any ranking $\kappa$ on $\ham$, the following holds $\witness{\kappa}{ac\com}\geq\min\left\{\witness{\kappa}{ab\com},\witness{\kappa}{bc\com}\right\}$ for all $a,b,c\in\sens$.
\end{lemma}
\begin{proof}
See \Cref{proof:triangle inequality}.\qed
\end{proof}
We are interested in studying the interactions between rankings on $\ham$ and non-degenerate poc-graph structures on $\sens$. A weakened notion of ranking is required for this purpose. 
\begin{definition}\label[definition]{defn:2-ranking} A {\em 2-ranking on $\sens$} is a symmetric matrix $\witness{}{\wild}=(\witness{}{ab})_{a,b\in\sens}$ with entries in $\NNH$, satisfying the following for all $a,b,c\in\sens$:
\begin{enumerate}
	\item $\witness{}{\mbf{0}a}=\witness{}{aa\com}=\infty$;
    \item $\min(\witness{}{aa},\witness{}{a\com a\com})=\min(\witness{}{bb},\witness{}{b\com b\com})<\infty$;
    \item $\witness{}{aa}=\min\left(\witness{}{ab},\witness{}{ab\com}\right)$;
    \item $\witness{}{ac\com}\geq\min\left(\witness{}{ab\com},\witness{}{bc\com}\right)$.
\end{enumerate}
We will say that a ranking $\kappa$ {\em agrees with $\witness{}{\wild}$}, if $\witness{\kappa}{ab}=\witness{}{ab}$ for all $a,b\in\sens$. Also, we will abbreviate as follows: $\witness{}{a}:=\witness{}{aa}$ and $\witness{}{\varnothing}:=\min(\witness{}{a},\witness{}{a\com})$, for all $a\in\sens$.
Finally, note how $\witness{}{a}=\min\{\witness{}{a0\com},\witness{}{a0}\}$ must hold, too, for all $a\in\sens$, by virtue of requirements 1. and 3.\qed
\end{definition}
Of course, the idea is to have a 2-ranking play the role of a snapshot weight, from which one needs to derive a non-degenerate PCR.
In our learning setting, the best one could do is to derive from the samples of the value signal $\val{}$ the 2-ranking $(\witness{\val{}}{ab})_{a,b\in\sens}$.
The main question is, then, how much of the original $\val{}$ could be recovered from this information.
The following family of PCRs helps answer this question:
\begin{proposition}\label[proposition]{prop:residual PCR} Suppose $\witness{}{\wild}$ is a 2-ranking, and let $\witness{}{\varnothing}\leq\delta\in\NNH$. Consider the PCRs on $\sens$ defined by:
\begin{equation}\label{implication from ranking finite}
	ab\in\residual{\witness{}{\wild};\delta}\DEFF
	\{a,a\com\}\cap\{b,b\com\}=\varnothing\,
	\text{ and }\,
	\witness{}{ab\com}>\delta
\end{equation}
for $\delta<\infty$, and by:
\begin{equation}\label{implication from ranking infinite}
	ab\in\residual{\witness{}{\wild};\infty}\DEFF\
	\{a,a\com\}\cap\{b,b\com\}=\varnothing\,,
	\text{ and }\,
	\witness{}{ab\com}=\infty\,.
\end{equation}
Then $\residual{\witness{}{\wild};\delta}$ is a non-degenerate PCR for all $\delta\in\NNH$.
\end{proposition}
\begin{proof}
See \Cref{proof:residual PCR}.\qed
\end{proof}
A surprising consequence of the non-degeneracy of these PCRs is the following corollary, leading to the conclusion that every 2-ranking has a ranking that agrees with it:
\begin{corollary}\label[corollary]{cor:point mass clamp} Let $\witness{}{\wild}$ be a 2-ranking on $\sens$, and let $a,b\in\sens$. Set $r:=\witness{}{ab}$. Then there exists a vertex $u\in\ham$ such that the point mass ranking $\nu=\pointmass{u}{r}$ satisfies $\witness{\nu}{pq}\geq\witness{}{pq}$ for all $p,q\in\sens$.
\end{corollary}
\begin{proof}
See \Cref{proof:point mass clamp}.\qed
\end{proof}
\begin{proposition}\label[proposition]{prop:completions exist} Let $\witness{}{\wild}$ be a symmetric $\NNH$-valued matrix. Then $\witness{}{\wild}$ is a 2-ranking if and only if there exists a ranking with which it agrees. Moreover, if $\witness{}{\wild}$ is a 2-ranking, then there exists one and only one ranking, 
\begin{equation}\label{eqn:completion formula}
	\completion{\witness{}{}}(u):=\max_{a,b\in u}\witness{}{ab}\,,
\end{equation}
that agrees with $\witness{}{\wild}$ and satisfies $\completion{\witness{}{}}\leq\kappa$ for every ranking $\kappa$ that agrees with $\witness{}{\wild}$.
\end{proposition}
\begin{proof}
See \Cref{proof:completions exist}.\qed
\end{proof}

The upshot of the last proposition is that, henceforth, any 2-ranking may be treated as encoding a ranking. Formally:
\begin{definition}\label[definition]{defn:completion and restriction} Suppose $\witness{}{\wild}$ is a 2-ranking and $\kappa$ is a ranking. The {\em completion} of $\witness{}{\wild}$ is the ranking $\completion{\witness{}{}}$ from the preceding proposition. The {\em 2-restriction} of $\kappa$ is the 2-ranking, denoted $\twores{\kappa}$, obtained from $\kappa$ via the concept representation, that is: $\twores{\kappa}_{ab}:=\witness{\kappa}{ab}$ for all $a,b\in\sens$. The {\em 2-closure} of $\kappa$ is the ranking, denoted $\twoclose{\kappa}$, obtained from $\kappa$ as the completion of its 2-restriction. In particular one has $\kappa\geq\twoclose{\kappa}$.\qed
\end{definition}

\subsubsection{Derived PCRs and their duals.}\label{snapshots:qualitative:derived}
We now introduce the PCR used in the qualitative snapshot structure. As systems of defaults, these PCRs are strengthened (more restrictive) versions of the (ranked) default systems constructed by Goldzmidt and Pearl in~\cite{Goldszmidt_Pearl-qualitative_probas_for_belief_revision}, and they satisfy an analogous characterization.
\begin{proposition} Suppose $\witness{}{\wild}$ is a 2-ranking. For $0\leq\delta<\infty$, let its {\em derived PCR} be defined by:
\begin{equation}\label{eqn:derived graph}
	ab\in \derived{\witness{}{\wild};\delta}\IFF\left\{\begin{array}{l}
		\{a,a\com\}\cap\{b,b\com\}=\varnothing\,,\text{ and }\\[.25em]
		\witness{}{ab\com}=\infty\text{ or }\witness{}{ab\com}>\delta+\max(\witness{}{ab},\witness{}{a\com b\com})\,,
	\end{array}\right.
\end{equation}
and for $\delta=\infty$ let it be defined by:
\begin{equation}
	ab\in \derived{\witness{}{\wild};\infty}\IFF\left\{\begin{array}{l}
		\{a,a\com\}\cap\{b,b\com\}=\varnothing\,,\text{ and }\\[.25em]
		\witness{}{ab\com}=\infty\text{ and }\witness{}{ab},\witness{}{a\com b\com}<\infty\,.
	\end{array}\right.
\end{equation}
Then, $\derived{\witness{}{\wild};\delta}$ is a non-degenerate PCR for all $\delta\in\NNH$.
\end{proposition}
\begin{proof} Let $G=\derived{\witness{}{\wild};\delta}$ and $R=\residual{\witness{}{\wild};\delta}$. Once again, the basic properties of a PCR are baked into the definition of $G$. Furthermore, observe that $ab\in G$ implies $ab\in R$ (though not the other way around). In particular, we have $N(G)\subseteq N(R)$ and it follows that $N(G)\cap N(G)\com\subseteq N(R)\cap N(R)\com=\varnothing$, as required.
\qed\end{proof}
\begin{definition}\label[definition]{graphs of a ranking} \Cref{prop:completions exist} and \Cref{defn:completion and restriction} make it possible for us to abuse notation and talk about the residual and derived PCRs of a ranking by setting $\residual{\kappa;\delta}:=\residual{\twores{\kappa};\delta}$, and $\derived{\kappa;\delta}:=\derived{\twores{\kappa};\delta}$, dropping all mention of $\delta$ when $\delta=0$, as before.
Of course, $\kappa$ may be replaced with its 2-closure $\twoclose{\kappa}$ throughout
%
.\qed
\end{definition}

We proceed to study properties of derived PCRs and their duals, to verify their utility to our representation problem.
Specifically, we are interested in the geometry of level sets, as we try to answer the question: how well does the 2-restriction of a ranking $\kappa$ capture the set of global minimum points of $\kappa$ (the most meaningful states according to $\kappa$)?
\begin{definition}\label[definition]{defn:minset} Given an integer $\epsilon\geq 0$ and a 2-ranking $\witness{}{\wild}$, denote:
\begin{equation}
	\minset{\witness{}{\wild};\epsilon}:=\left\{
    	a\in\sens\,\big|\,\witness{}{a}<\witness{}{a\com}-\epsilon
    \right\}
\end{equation}
The set $\minset{\witness{}{\wild}}:=\minset{\witness{}{\wild};0}$ will be referred to as the {\em minset of $\witness{}{\wild}$}. By virtue of \Cref{prop:completions exist}, this notion extends to rankings as follows:
\begin{equation}
	\minset{\kappa;\epsilon}:=\left\{
    	a\in\sens\,\big|\,\witness{\kappa}{a}<\witness{\kappa}{a\com}-\epsilon
    \right\}\,,
\end{equation}
with $\minset{\kappa}:=\minset{\kappa;0}$ being the {\em minset of $\kappa$}.\qed
\end{definition}

It is clear that a global minimum point of a ranking $\kappa$ must contain $\minset{\kappa}$.
Hence, $\half{\minset{\kappa}}$ contains all global minima of $\kappa$, but what does this have to do with the derived PCR and its dual?
The main result is as follows:
\begin{proposition}\label[proposition]{prop:global minima} Let $\kappa$ be a ranking on $\ham$ and set $G=\derived{\kappa}$ and $M=\minset{\kappa}$.
Let $F$ and $\widehat{F}$ be the sets of global minima of $\kappa$ and $\completion{\kappa}$, respectively.
Then $F\subseteq\widehat{F}\subseteq G^\circ$ and $\widehat{F}=\half{M;G}$.
Moreover, $\widehat{F}$ is the convex hull of $F$ in $\dual{G}$
\end{proposition}
\begin{proof}
See \Cref{appendix:snapshots:minima}.\qed
\end{proof}
Upon inspection, the details of the proof generate the impression that $\completion{\kappa}$ is, for lack of a better word, a form of convex smoothing of $\kappa$, the last proposition showing how the collection of possibly disparate minimum points of $\kappa$ coalesces into a convex plateau of minimum points of $\completion{\kappa}$ in the dual space of the derived PCR. 

\subsubsection{A Snapshot Structure to Learn a Ranking.}
We return to our learning problem.
Suppose $\val{}$ is a {\em fixed} ranking on $\ham$, and we are given a sequence of samples $\val{}\at{t}=\val{}(\observe{}\at{t})$, where $\observe{}\at{t}\in\ham$ are the observations made by our agent.
We will assume $\val{}\at{t}<\infty$ for all $t$, reserving $\val{}(u)=\infty$ for the impossible observations.

We must define the {\em weight update} taking place in response to an incoming observation; and the {\em weight extension} in response to a query being added to the sensorium.

\paragraph{Weight update (static case).} For our snapshot structure, we propose the following update rule for the snapshot weights:
\begin{equation}\label{eqn:weight update constant ranking}
	\witness{}{}\at{0}=\twores{\pointmass{\observe{}\at{0}}{\val{}(\observe{}\at{0})}}\,,\quad
    \witness{}{}\at{t+1}:=\min\left\{\witness{}{}\at{t},\twores{\pointmass{\observe{}\at{t+1}}{\val{}(\observe{}\at{t+1})}}\right\}\,.
\end{equation}
By \Cref{ex:minimum}, $\witness{}{\wild}\at{t}$ is a 2-weight for every $t\geq 0$, giving rise to a non-degenerate PCR in the form of 
\begin{equation}
    G\at{t}:=\derived{\witness{}{\wild}\at{t}}\,.
\end{equation}
Since the sequence of weights is pointwise non-increasing, its convergence is guaranteed. 
Moreover, exposure to (at most) $N=\binom{\card{\sens}}{2}-\tfrac{\card{\sens}}{2}$ observations covering all pairs $\{a,b\}$ with $\{a,a\com\}\cap\{b,b\com\}=\varnothing$, sampling a minimum rank world in $\half{ab}$ for each pair $\{a,b\}$ at least once, will result in $\witness{}{\wild}$ coinciding with $\twores{\val{}}$. This motivates the question {\em ``How much less exposure is required for delivering the same result on average, in, say, an appropriately formulated PAC setting?''}, and emphasizes the good fit of ranking-based snapshot structures to settings featuring a teacher.

\subsection{Statistical Integrators of a Real-Valued Signal.}\label{snapshot:real-valued} The original suggestion of~\cite{GK-Allerton_2012} for maintaining a system of weights in the role of a snapshot structure was based on the idea that $\witness{}{ab}\at{t}$ should be the empirical estimate at time $t$ of the probability of the event $a\wedge b$, so that $ab\in G\at{t}$ could be put on record if and only if $\witness{}{ab\com}\at{t}<\min(\witness{}{ab}\at{t},\witness{}{a\com b\com}\at{t},\witness{}{a\com b}\at{t},\tau_{ab}\at{t})$, where $\tau_{ab}\at{t}$ is a fixed threshold.
That is, the implication $a\rightarrow b$ is put on record whenever the event $a\wedge\neg b$ has sufficiently low empirical probability.
We have since found out that the improved formalization provided by \Cref{prop:when maximal coherent is complete,prop:canonical quotient} enables the use of a far more general weight update scheme that is capable of incorporating a value signal into the learner's reasoning while also taking into account the observed frequency of events.

\subsubsection{Real-valued 2-weights.}
Once again, the learner is presented with a sequence of observations $u_t\in\ham$, accompanied by the signal $\val{}\at{t}=\val{}(u_t)$.
This time we require that the value signal $\val{}\at{t}$ presented to the agent at time $t$ is a real number greater than or equal to $1$, where a higher value of $\val{}$ indicates a more meaningful state of the observed system.
\begin{definition}\label[definition]{defn:real-valued 2-weight} A {\em real-valued 2-weight} on a PCS $\sens$ is a symmetric, real-valued function $\witness{}{\wild}=(\witness{}{ab})_{a,b\in\sens}$ on $\sens\times\sens$, satisfying the following requirements for all $a,b,c\in\sens$:
\begin{enumerate}
    \item $\witness{}{ab}\geq 0$, $\witness{}{\zero a}=0$ and $\witness{}{aa\com}=0$;
    \item $\witness{}{\varnothing}:=\witness{}{a}+\witness{}{a\com}=\witness{}{b}+\witness{}{b\com}$;
    \item $\witness{}{a}:=\witness{}{aa}=\witness{}{ab}+\witness{}{ab\com}$;
    \item $\witness{}{ab\com}+\witness{}{bc\com}+\witness{}{ca\com}=\witness{}{a\com b}+\witness{}{b\com c}+\witness{}{c\com a}$;
    \item $\witness{}{ac\com}+\witness{}{a\com c}\leq\witness{}{ab\com}+\witness{}{a\com b}+\witness{}{bc\com}+\witness{}{b\com c}$.
\end{enumerate}
When $\witness{}{ab}=0$ for all $a,b\in\sens$, we say $\witness{}{\wild}$ is {\em trivial}.\qed
\end{definition}
The following example provides motivation for the definition:
\begin{example}\label[example]{ex:2-weights from integrals} Suppose $(\spc,\mathscr{B},\mu)$ is a measure space and $\val{}:\spc\to\RR$ is a non-negative function in $L_1(\mu)$.
Suppose $\rho:\sens\to\mathscr{B}$ is a PCS morphism, when $\mathscr{B}$ is viewed as a sub-PCS of $\power{\spc}$ (recall \Cref{ex:power set as PCS}).
Then $\witness{}{ab}:=\int_{\rho(a)\cap\rho(b)}\val{}\dd\mu$ is a real-valued 2-weight.
Indeed, since the integral of a non-negative function is non-negative, the requirements 1.-5. become corollaries of various set-theoretic identities applied to $A=\rho(a)$, $B=\rho(b)$ and $C=\rho(c)$, respectively:
\begin{enumerate}
    \item $\varnothing\cap A=\varnothing$, $A\cap(\spc\minus A)=\varnothing$.
    \item $\spc=A\cup(\spc\minus A)=B\cup(\spc\minus B)$,
    \item $A=(A\cap B)\cup(A\minus B)$,
    \item $(A\minus B)\cup(B\minus C)\cup(C\minus A)=(B\minus A)\cup(C\minus B)\cup(A\minus C)$ (see \Cref{fig:weird identity}),
    \item $A\symplus C\subseteq (A\symplus B)\cup(B\symplus C)$,
\end{enumerate}
where $A\symplus B:=(A\minus B)\cup(B\minus A)$, for short.
\end{example}
\begin{figure}[H]
  \centering
  \includegraphics[width=.75\textwidth]{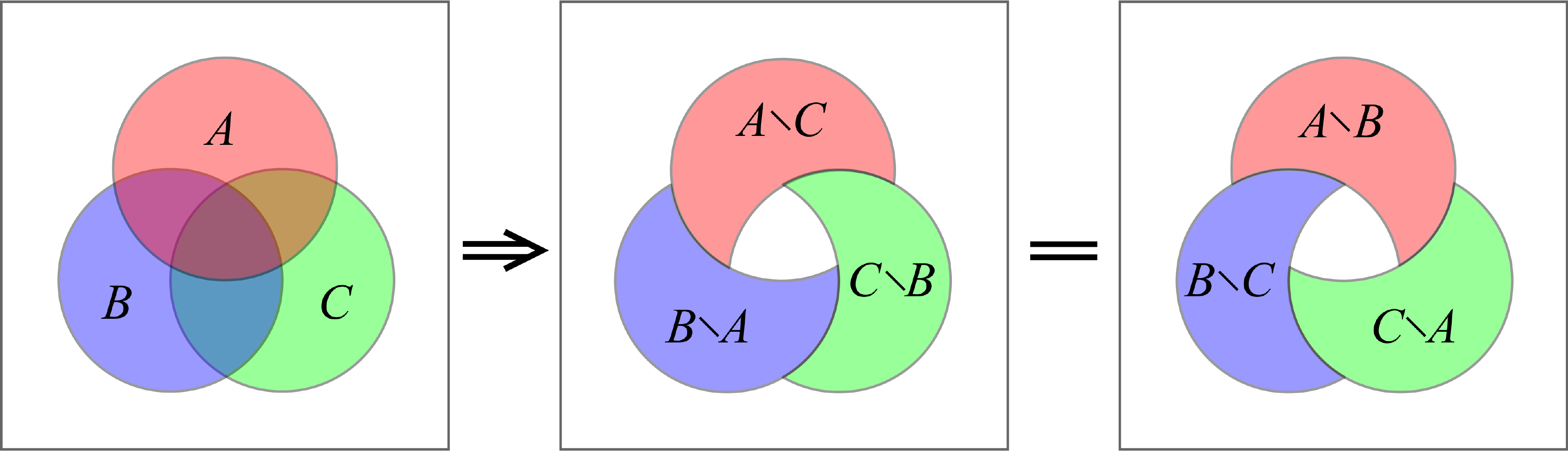}
  \caption{The set-theoretic identity underlying requirement 4. of a real-valued 2-weight (\Cref{defn:real-valued 2-weight}), as explained in \Cref{ex:2-weights from integrals}.}
  \label{fig:weird identity}
\end{figure}
\begin{example}[point mass weight]\label[example]{ex:point-mass weight} Similarly to the qualitative setting, the simplest example of a weight of this form is given by a point-mass measure on $\ham$:
\begin{equation}
    \witness{}{ab}=r\cdot\delta_u(\half{ab})\,,\quad
    \delta_u(F):=\left\{\begin{array}{cl}
        1   &\text{if }u\in F\\
        0   &\text{if }u\notin F\,,
    \end{array}\right.
\end{equation}
where $F\subset\ham$ (Compare with \Cref{ex:point-mass ranking}).\qed
\end{example}

\subsubsection{Derived PCRs and their duals.} 
The resulting notion of a derived PCR requires a system $\tau_\wild$ of threshold values, denoted $\tau_{ab}\at{t}\in(0,1)$, $a,b\in\sens$, satisfying the identities
\begin{equation}\label{eqn:threshold symmetries}
    \tau_{ab}\at{t}=\tau_{ba}\at{t}=\tau_{a\com b}\at{t}
\end{equation}
for all $a,b\in\sens$ and $t\geq 0$.
This makes it possible to construct a non-degenerate PCR as follows:
\begin{proposition}\label[proposition]{prop:statistical derived PCR}
For any choice of threshold values $\tau_\wild$ satisfying \Cref{eqn:threshold symmetries}, if $\witness{}{\wild}$ is non-trivial, then
\begin{equation}\label{eqn:real-valued derived PCR}
    ab\in \derived{\witness{}{\wild},\tau_\wild}\at{t}\DEFF\left\{\begin{array}{c}
    \witness{}{ab\com}\at{t}<min(
        \tau_{ab}\cdot\witness{}{\varnothing},
        \witness{}{ab},
        \witness{}{a\com b\com},    \witness{}{a\com b}
    )\\[.25em] 
    \texttt{or}\\[.25em]
    \witness{}{ab\com}=\witness{}{a\com b}=0\,.
    \end{array}\right.
\end{equation}
defines a non-degenerate PCR.
\end{proposition}
\begin{proof}
See \Cref{proof:statistical derived PCR}.\qed
\end{proof}
Let $G=\derived{\witness{}{\wild}}$ for a real-valued 2-weight $\witness{}{\wild}$.
A notion analogous to that of a minset may be considered in the real-valued setting, taking into account the reversal of the value hierarchy (now, bigger values of $\val{}$ are considered the most significant):
\begin{equation}
    \minset{\witness{}{\wild}}:=\set{a\in\sens}{\witness{}{a}>\witness{}{a\com}}\,,
\end{equation}
The argument that $M=\minset{\witness{}{}}$ is $G$-coherent and forward-closed, for any choice of the thresholds $\tau_\wild$, is the same as the one given for minsets in the qualitative setting (\Cref{lemma:minsets are forward-closed coherent} in \Cref{appendix:snapshots:minima}), upon reversing the relevant inequalities.
This time around, however, the non-empty convex subset $F=\half{M;G}$ of $\dual{G}$ does not directly relate to extreme points of the value signal $\val{}$ in $\ham$, but, rather, to a notion of {\em center of mass} of $G^\circ$ with respect to $\witness{}{\wild}$, seen as a representation of the distribution of the value signal over $\ham$.

\subsubsection{Snapshot update.}
Similarly to the qualitative setting, in the real-valued setting we will also be assembling our estimate of the [integrals of the] observed value signal from point-masses, this time replacing minimization with linear combinations.
The update rule for a {\em discounted integrator snapshot} takes the form:
\begin{equation}\label{eqn:general discounted update}
    \left\{\begin{array}{rcl}
    \witness{}{ab}\at{0}&:=&\val{}\at{0}\cdot\delta_{u_0}(\half{ab})\,,\\[.5em]
    \witness{}{ab}\at{t+1}&:=&q\at{t}\cdot\witness{}{ab}\at{t}+(1-q\at{t})\cdot\val{}\at{t+1}\cdot\delta_{u_{t+1}}(\half{ab})\,,
    \end{array}\right.
\end{equation}
where the $q\at{t}\in(0,1]$ are the {\em discount coefficients}, $t\geq 1$.
The fact that $\witness{}{\wild}\at{t+1}$ is a convex combination of real-valued 2-weights ensures that $\witness{}{\wild}\at{t+1}$ is a real-valued 2-weight as well.

\paragraph{Types of update.} We studied two variants of the discounted integrator snapshot:
\begin{enumerate}
    \item{\bf Empirical Snapshot. } In this case, one sets $q\at{t}:=\tfrac{t+1}{t+2}$, resulting in
    \begin{equation}
        \witness{}{ab}\at{t}=\frac{1}{t}\sum_{s=0}^t\val{}\at{s}\cdot\delta_{u(s)}(\half{ab})\,,
    \end{equation}
    which is the empirical estimate for the integral of $\val{}$ over $\rho(a)\cap\rho(b)$.
    For this snapshot type, we used fixed thresholds $\tau_{ab}=\tau$.
    \item{\bf Fixed Discount Snapshot.} Here one sets $q\at{t}:=q$, a constant, playing the role of a rate at which information acquired about the signal `fades' unless continually reinforced by incoming observations:
    \begin{equation}
        \witness{}{ab}\at{t}=q^t\val{}\at{0}\cdot\delta_{u(0)}(\half{ab})+(1-q)\sum_{s=1}^{t} q^{t-s}\val{}\at{s}\cdot\delta_{u(s)}(\half{ab})\,.
    \end{equation}
    The eventual purpose of using an update of this form is to accommodate settings where $\val{}$ has multiple peaks, as well as, possibly, the dynamic setting, provided the value signal changes sufficiently slowly.
\end{enumerate}

\paragraph{PAC learning guarantees.} The notion of {\em probably approximately correct (PAC) learning} introduced by Valiant~\cite{Valiant-PAC_paper} is one framework within which the quality of UMAs based on real-valued snapshots could be discussed.
The assumptions of this setting are that the observations $(u_t)_{t\geq 0}$ are i.i.d. samples of a fixed distribution on $\ham$, in which case, for any fixed pair $a,b\in\sens$ with $\{a,a\com\}\cap\{b,b\com\}=\varnothing$, one could think of the sequence of input values $\mathrm{X}_{ab}\at{t}:=\val{}\at{t}\cdot\delta_{u_t}(\half{ab})$ as a sequence of i.i.d. samples of a random variable $\mathrm{X}_{ab}\in[0,A]$, where $A$ is an upper bound on the value signal $\val{}$. 
Equation~\Cref{eqn:general discounted update} then lets us think of $\witness{}{ab}\at{t}$ as random variables $\mathrm{Y}_{ab}\at{t}$ constructed according to $\mathrm{Y}_{ab}\at{t+1}=q\at{t}\mathrm{Y}_{ab}\at{t}+(1-q\at{t})\mathrm{X}_{ab}\at{t+1}$.
Applying induction one immediately verifies that $\expect{\rv{Y}_{ab}\at{t}}=\expect{\rv{X}_{ab}}$ for all $t\geq 0$.
It thus becomes reasonable to ask how many samples are required in order to bring the probability that $\left|\mathrm{Y}_{ab}\at{t}-\expect{\mathrm{X}_{ab}}\right|>\varepsilon$ below a specified threshold.
Valiant~\cite{Valiant-PAC_paper} had long ago observed that Chernoff bounds are a powerful tool for answering such questions.
Computing Chernoff bounds for our setting yields:
\begin{proposition}[PAC learning in empirical snapshots]\label[proposition]{prop:PAC real-valued} Given $\delta>0$, the empirical snapshot learning mechanism attains a precision of $\delta$ on all weights, with probability $1-\delta$ from a number of i.i.d randomized samples that is at most linear in $\tfrac{1}{\delta}$, at a rate depending only on the value signal.
\end{proposition}
\begin{proof}
See \Cref{proof:PAC real-valued}.\qed
\end{proof}
Our simulation results indicate that similar guarantees could be expected for the discounted setting, but the standard Chernoff-inspired approaches for leveraging the independence of the observations do not seem to work.
Since discounted snapshot learning makes it easier for the representation to recover from false implications, it is important to ascertain whether or not a result of the form \Cref{prop:PAC real-valued} could be proved, and if not---in what circumstances it might fail.

\paragraph{Other learning scenarios.}
The PAC learning guarantees of the preceding paragraph are predicated on the assumption that the sequence of observations is statistically independent.
This assumption becomes unreasonable for an observer of a system whose state evolves continuously over time, subject to some internal dynamics, in which case it is often unlikely that contiguous observations will be uncorrelated.

A fairly general model of such settings is provided by Markov chains~\cite{Rivest_Schapire-diversity}, where the underlying Markov process models the (uncertain) dynamics of the observed system.
In our setting, one regards $\rho\com(\spc)$---the set of observable possible worlds in $\ham(\sens)$ (\Cref{models:first notions})---as the set of states of a fixed (albeit unknown) Markov process.
Then, by the ergodic theorem for Markov chains~\cite{Feller-the_book}, one has:
\begin{proposition}\label[proposition]{prop:convergence for random walks} Suppose the sequence of observations $u_t\in\ham$ is sampled from an a-periodic, irreducible, positive-recurrent Markov chain with limiting distribution $\pi$. Then the empirical snapshot weights $\witness{}{ab}\at{t}$ learned from the constant value signal $\val{}\at{t}=1$ converge to the marginals $\int_{\rho(a)\cap\rho(b)}\dd\pi$, for all $a,b\in\sens$.\qed
\end{proposition}
In particular, any thresholded implications derived from the real-valued 2-weight $\witness{}{ab}:=\int_{\rho(a)\cap\rho(b)}\dd\pi$ will be recovered in this process.

Finally, it follows from the decomposition theorem for Markov chains~\cite{Feller-the_book} that the ergodicity assumption in the above proposition does not impose undue restrictions on our model, as we only expect an agent to learn implications from recurring observations anyway.
We also note that the special case of lazy random walks guarantees an exponential rate of convergence to the limiting distribution in many interesting cases (see Theorem 5.1 of~\cite{Lovasz-random_walks_survey} and Theorem 9 of \cite{Rivest_Schapire-diversity}).

\section{Simulations.}\label{simulations}
We present two kinds of simulation studies.
\Cref{sim:1} illustrates the preceding results about learning with different snapshot types in a sample of `toy' settings. \Cref{sim:BUAs} explains how to construct simple UMA-based binary agents, whose performance is considered in~\Cref{sim:2}.

\subsection{Simulation settings.}\label{sim:settings} Each setting considered in \Cref{sim:1} consists of an observer/agent $\mathscr{A}$ situated in a discrete environment, $\env$.
For simplicity, the queries assigned to $\mathscr{A}$ are functions of the agent's current position in the environment, which we denote by $\pos(t)\in\env$.
Let $[N]:=\{0,\ldots,N\}$.
The environments and sensory endowments we consider are:
\begin{itemize}
    \item{\bf Discretized interval with GPS. } Here $\env=[N]$, and $\mathscr{A}$ has queries $\alphabet=\{a_1,\ldots,a_N\}$, with $a_i$ holding true at time $t$ iff $\pos(t)<i$;
    \item{\bf Discretized circle with beacons. } Now set $\env=[N-1]$ with $a_i$ ($i=0,\ldots,N-1$) holding true iff $\pos(t)$ is close enough to $i$, modulo $N$;
    \item{\bf Discretized interval with random position sensors. } $\env=[N]$ again, and $\alphabet=\{a_1,\ldots,a_N\}$, with $a_i$ true at time $t$ iff $\pos(t)\in A_i$, where $A_i\subsetneq\env$ are chosen uniformly at random ahead of each simulation run.
\end{itemize}

We consider different value signals, all set to be functions of the position, depending on snapshot type:
\begin{itemize}
    \item{\bf Qualitative Snapshots. } Two natural choices of the signal are considered,
    \begin{equation}\label{eqn:sniffy qualitative signals}
        \val{d}(p):=\left\{\begin{array}{rl}
            0 &\text{if }p=T\,,\\
            1 &\text{if }p\neq T\,,
        \end{array}\right.
        \quad\text{and}\quad
        \val{s}(p):=\dist{p}{T}\,,
    \end{equation}
    where $T$ should be regarded as a ``target'' position of high significance. 
    \item{\bf Real-valued Snapshots. } To parallel the ``sharp peak''/``dull peak'' signal variants from the qualitative setting, we pick: 
    \begin{equation}\label{eqn:sniffy real-valued signals}
        \val{d}(p):=1+\diam{\env}-\dist{p}{T}\,,\quad\text{and}\quad
        \val{s}(p):=\val{d}(p)^4\,,
    \end{equation}
    respectively.
    For discounted snapshots, the discount coefficients were picked to be $q=0.999$. 
    Learning thresholds are constant, where relevant, and are chosen to equal $\tfrac{1}{2N}$ to ensure correct learning of implications among the initial sensors by the real-valued snapshots.
\end{itemize}

\subsection{Simulation results for observers.}\label{sim:1} 
To assess the speed and quality of PCR learning, we track the error-rate of the learned PCR representation---the fraction of correctly learned PCR implications---over time.

\subsubsection{Repeated i.i.d. sampling (PAC-style setting).}\label{sim:1:iid}
\Cref{fig:sniffy_OBS_iid_error_rates} compares logarithmic plots of two mean error rates over $100$ observation sequences generated by repeated i.i.d. uniform sampling of positions from the environment, for the settings described in \Cref{sim:settings} for $N=20$:
\begin{enumerate}
    \item{\bf Solid lines. } The mean fraction of incorrect implications in the learned PCR relative to the expected PCR for the given learner in each setting, as a function of time;
    \item{\bf Dashed lines. } The mean fraction of incorrect implications in the transitive closure of the learned PCR relative to the poc set of actual implications among the provided sensors, as a function of time;
    \item{\bf Shaded regions } depict the mean$\pm$standard deviations for the corresponding quantities.
\end{enumerate}
\begin{figure}[ht]
    \includegraphics[width=1\textwidth]{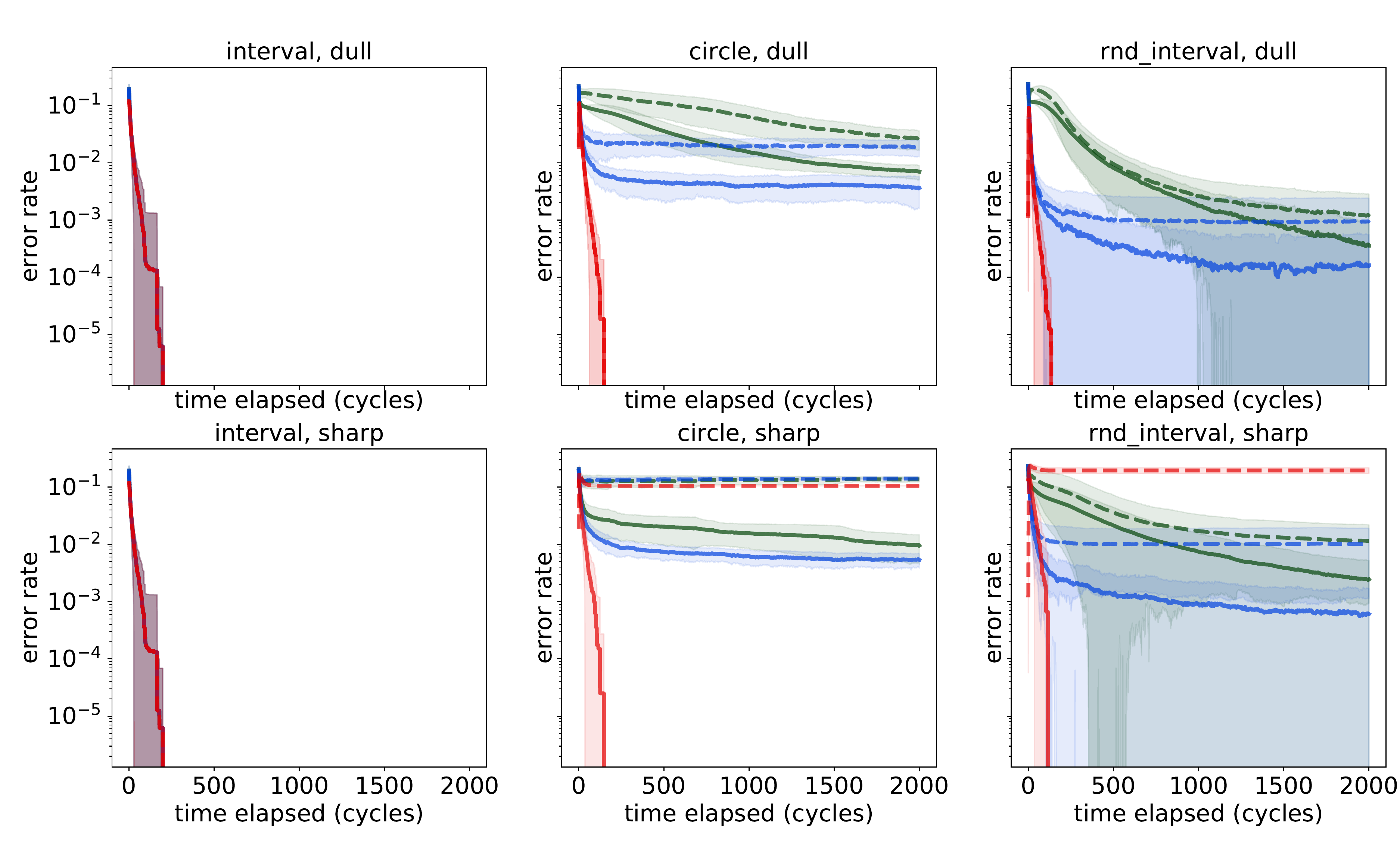}
    \caption{Evolution over time of the log mean error rates (curves in bold, logarithmic scale), plus/minus standard deviation (shading), over $100$ runs, of PCR representations acquired by i.i.d. uniform random sampling, for different snapshot types: empirical (red), discounted (green), qualitative (blue).
    Columns, left to right: interval w/GPS sensors; circle w/beacon sensors; and interval w/random sensors.
    Rows, top to bottom: $\val{1}$ (``dull peak'' value signal), $\val{2}$ (``sharp peak'' value signal), as in \Cref{eqn:sniffy qualitative signals,eqn:sniffy real-valued signals}.
    See \Cref{sim:1:iid} for more details.
    }
    \label{fig:sniffy_OBS_iid_error_rates}
\end{figure}

The first most notable feature of the figures---beyond confirming (and, in fact, exceeding) the theoretical results---is the complete agreement of the curves for all six learners on the interval (left column).
Since the sensors in this case are nested, the poc set of true implications coincides with the derived PCR induced by the expected weights and is recovered quickly and completely.

Next, on the circle we begin to see the difference between the quality of the learned PCR and the quality of the inferred system of implications as compared to the real ones.
This deterioration in quality was to be expected, as transitive closure enables the deduction of implications from chains of {\em approximate} implications recorded in the PCR.
Observe that the discrepancy is bigger for the sharp peak settings, in which a very small degree of significance is assigned to positions farther away from the target.
This difference is most notable in the qualitative learners: while completely absent in the dull peak setting, it is very visible in the sharp peak setting.
We account for these differences, among other things, in the detailed analysis of the true PCR provided in~\Cref{app:sniffy on the circle}.

A similar discrepancy is visible, but less pronounced in the third column, though we must keep in mind that, in this column, each run was executed with a different random collection of sensors.
The differences are less pronounced than on the circle because in the sensorium we have chosen for the circle there is very little nesting, while in a random sensorium, the probability of nesting is non-negligible.
Nesting relations $\rho(a)\subset\rho(b)$ in the sensorium forces $\witness{}{ab\com}\at{t}=0$ in real-valued snapshots, and $\witness{}{ab\com}\at{t}=\infty$ in qualitative snapshots at all times $t$, guaranteeing that $a<b$ will be learned with sufficient exposure. 
The rotation-invariant sensorium we chose for the circle has very little nesting, and hence much more room for error if the provided value signal happens to discount too many positions as being insignificant.
Deeper differences arise as a result of the circle's non-trivial homotopy type, which we discuss in~\Cref{sim:2:sniffy:circle} and further in~\Cref{discussion}.

Finally, let us remark that we do not yet have a good explanation for the good behavior of the discounted learners.
We were unable to prove any concentration inequalities for the discounted weight update to parallel the ones obtained for the empirical one.
Moreover, the quality of learning appears to be very sensitive to the choice of discount parameter.
In fact, it was this difficulty with appropriately selecting and controlling the discount parameter that motivated the construction of qualitative learners in the first place.

\subsubsection{Lazy random walk (learning from ``motor babble'').}\label{sim:1:lazy}
For a robotic system, a more realistic mode of sampling from the environment is ``motor babble'': a random walk on $\env$ generated by repeated i.i.d. sampling from the space of available actions/decisions.
In this mode, each instance of the agent $\mathscr{A}$ is constrained to a small set of available actions, depending on $\env$:
\begin{itemize}
    \item{\bf Discretized interval. } The allowed actions are a single step to the right ($\rt\colon p\mapsto\min\{N,p+1\}$), a step to the left ($\lt\colon p\mapsto\max\{p-1,0\}$), or to remain in place;
    \item{\bf Discretized circle. } Similarly, on the circle $\rt\colon p\mapsto p+1 (\mathrm{mod}\;N)$, or $\lt\colon p\mapsto p-1 (\mathrm{mod}\;N)$ or to do nothing at all.
\end{itemize}
\Cref{fig:sniffy_OBS_lazywalk_error_rates} shows the evolution of the error rates we had considered earlier in \Cref{sim:1:iid}, in the new sampling mode.

\begin{figure}[ht]
    \centering
    \includegraphics[width=1\textwidth]{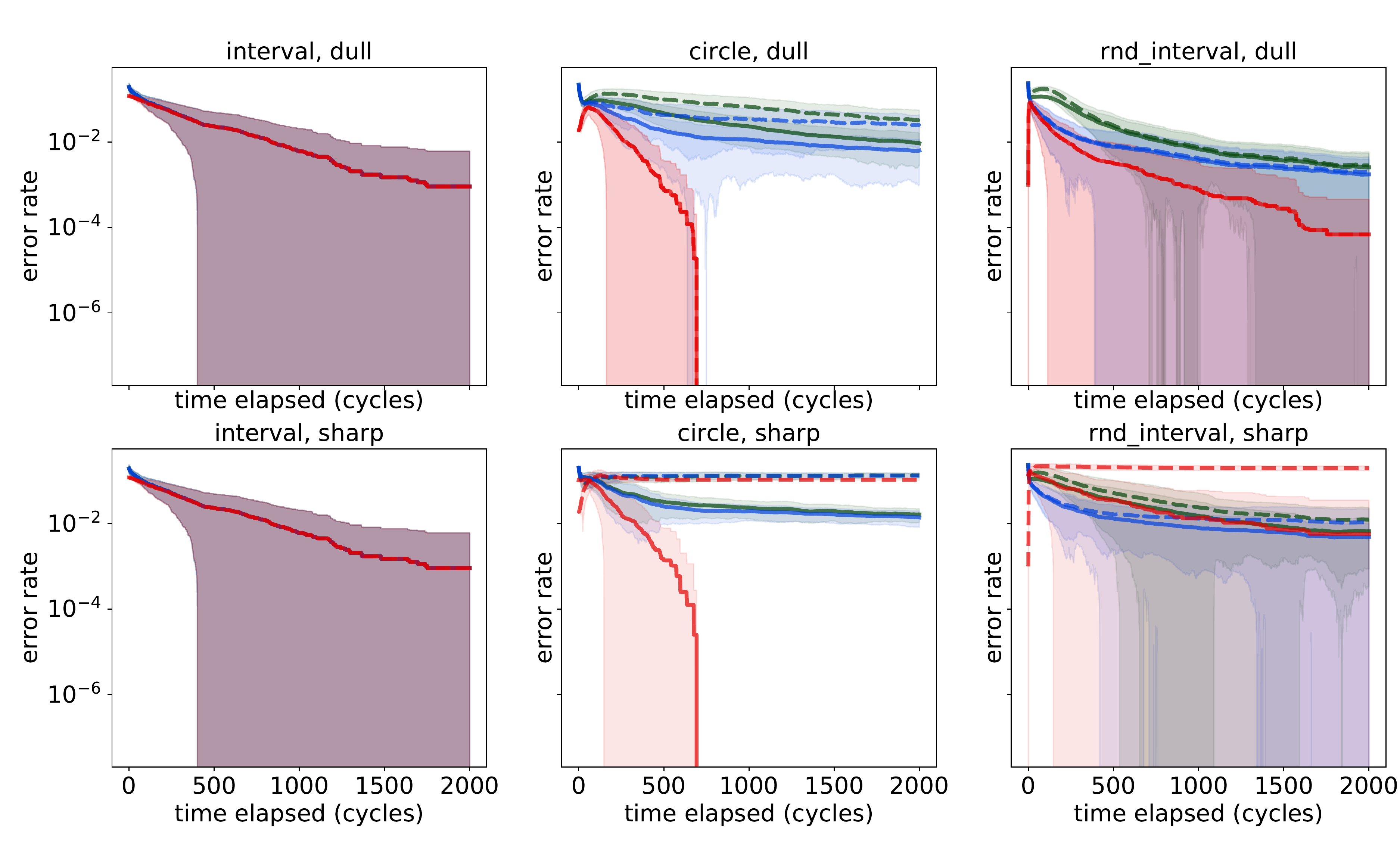}
    \caption{Evolution over time of the mean error rates (curves in bold, logarithmic scale), plus/minus standard deviation (shading), over $100$ runs, of PCR representations acquired by a lazy walk over the environment, for different snapshot types: empirical (red), discounted (green), qualitative (blue).
    Columns, left to right: interval w/GPS sensors; circle w/beacon sensors; and interval w/random sensors.
    Rows, top to bottom: $\val{1}$ (``dull peak'' value signal), $\val{2}$ (``sharp peak'' value signal).
    See \Cref{sim:1:lazy} for details.
    }
    \label{fig:sniffy_OBS_lazywalk_error_rates}
\end{figure}
This set of plots provides a good illustration of the robustness of UMA learning---especially with qualitative snapshots---where the quality of learning improves over time (though now at a much slower pace, due to the change in the sampling process), as the observer gains more exposure to the observed system.

\subsubsection{Learning the target set over time.}\label{sim:1:target}
We compare how UMA learners of different snapshot types develop their notion of the target set, $\minset{\witness{}{\wild}}$, over time.
For this purpose, \Cref{fig:sniffy_OBS_interval_AT_walk} shows this evolution for a single run from a separate batch of lazy random walk observations in a smaller environment ($N=10$), over a shorter period of time ($500$ cycles).
The features observed in this plot are, however, typical of the runs we generated for \Cref{fig:sniffy_OBS_lazywalk_error_rates}.
``Downgrading'' the experiment to a smaller environment enabled faster learning, and hence plotting the run at a lower resolution, without requiring the reader to magnify the plot attempting to discern its significant features.

\begin{figure}[ht]
    \centering
    \includegraphics[width=\textwidth]{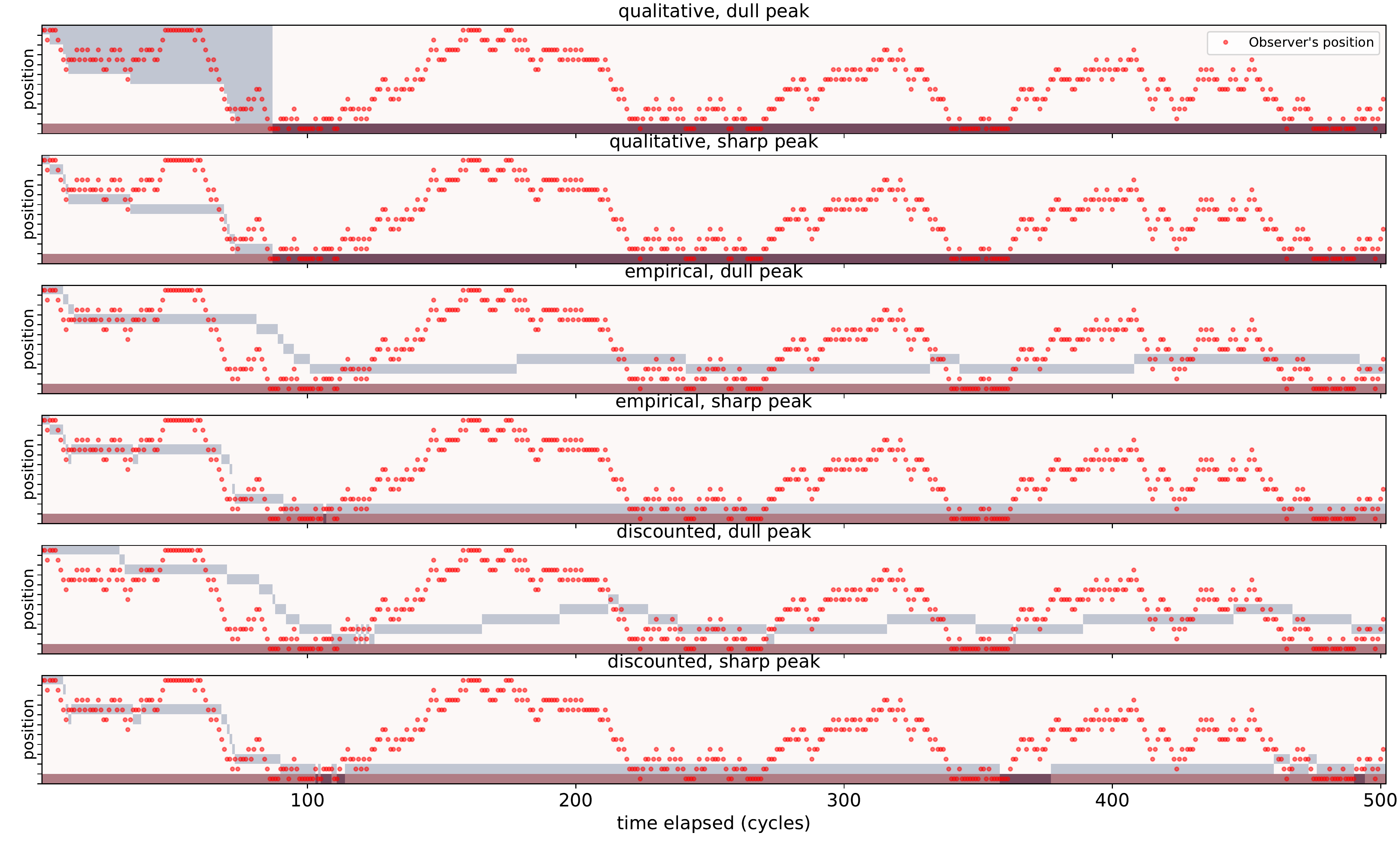}
    \caption{Evolution of an UMA observer's notion of target set (vertical bars drawn in grey) towards the true target (pink bar), in the interval w/GPS setting, $N=10$.
    The agent's position (red dot) evolves as a symmetric random walk.
    See \Cref{sim:1:target} for details.
    }
    \label{fig:sniffy_OBS_interval_AT_walk}
\end{figure}

Observe the eventual precision and efficiency of the qualitative reasoners, compared to the drift (away from the target) clearly noticeable for the real-valued learners.
Also note some initial delay in learning the target (in comparison with other types) in the discounted learners: the value of $q$ places a bound on how quickly an implication may be learned.

Both these observations are typical of all the batches we have observed.
This suggests the qualitative UMA learners as the best bet for upgrading UMAs to perform learning in the dynamic setting.
This also suggests that the real-valued learners could benefit from more careful shaping of the value signal, with significantly sharper peaks, as well as from lower values of the discount parameter (for discounted learners), if learning on shorter time scales is important.

\subsection{Binary UMA agents.}\label{sim:BUAs}
Postponing a more general formal definition of a binary UMA agent to another paper, let us describe just the simple sub-class of these agents considered here.

\paragraph{Actions as agents. } Given the environment $\env$ and the associated set of queries $\alphabet$ as described above in \Cref{sim:settings}, we regard each of the actions $\alpha$ available to $\mathscr{A}$ as an individual agent $\mathscr{A}_\alpha$, in charge of making the decision whether to act ($\alpha$) or not to act ($\alpha\com$).
Any conflicts between decisions made by different $\mathscr{A}_\alpha$ are, at this stage of development, arbitrated by hard-wiring (see example in \Cref{sim:2:sniffy} below).

\paragraph{Extended query set. } For $\mathscr{A}_\alpha$ to be capable of considering the consequences of its decisions, we have to extend $\alphabet$ so as to enable reasoning about the past.
Specifically, each $\mathscr{A}_\alpha$ is assigned a value signal $\val{\alpha}$, and an initial set of queries $\sens_\alpha:=\sens(\alphabet\cup\delay\alphabet)$, where $(\delay)$ is the {\em delay operator}:
the query $\delay q$ holds true at time $t+1$ if and only if $q$ held true at time $t$.

\paragraph{UMA representation conditional on action. } The BUA $\mathscr{A}_\alpha$ maintains {\em two} snapshots, $\witness{\beta}{\wild}$, $\beta\in\{\alpha,\alpha\com\}$.
The $2$-weight $\witness{\beta}{\wild}$ is updated precisely in those transitions in which $\mathscr{A}_\alpha$ acted according to $\beta$.
Thus, at any time $t$, $\witness{\beta}{\wild}\at{t}$ may be used to infer implications conditioned on $\beta$ taking place, by computing a derived graph, $G^\beta\at{t}$.

\paragraph{Prediction. } Given the current state $\current{\beta}\at{t}$ at time $t$ as represented by the $\beta$ snapshot, $\beta\in\{\alpha,\alpha\com\}$, the {\em prediction for time $(t+1)$ given $\beta$}, $\prediction{\beta}\at{t+1}$, is defined to be the coherent projection of $\delay\current{\beta}\at{t}$ with respect to $G^\beta\at{t}$.
This is the collection of sensations which $\mathscr{A}_\alpha$ can {\em prove} will occur if $\beta$ is chosen to take place, provided, of course, $G^\beta\at{t}$ persists into the $(t+1)$-st cycle.

\paragraph{Decision. } At the same time, each of the agent's two snapshots has a notion of where it is that the agent {\em should be}: the subset $\half{\minset{\witness{\beta}{\wild}};G^\beta\at{t}}$.
A simple way for $\mathscr{A}_\alpha$ to make a choice of $\beta\in\{\alpha,\alpha\com\}$ is to pick the value of $\beta$ for which $\divergence{\prediction{\beta}\at{t+1};\minset{\witness{\beta}{\wild}}\at{t}}$ is smaller, and to flip an even coin in the case of a tie (recall~\Cref{defn:divergence}).

\subsection{Simulation results for agents.}\label{sim:2}

\subsubsection{Sniffy: locating a stationary target using ``place field'' sensors.}\label{sim:2:sniffy} Consider an agent $\mathscr{A}$ in one of the two fixed settings described above in \Cref{sim:settings}, with two actions $\rt$ and $\lt$, as defined in \Cref{sim:1:lazy}, implemented as BUAs according to \Cref{sim:BUAs} with the value signals given in Equations~\eqref{eqn:sniffy qualitative signals} and~\eqref{eqn:sniffy real-valued signals}.
To minimize interference between $\rt$ and $\lt$, we impose a hard-wired arbitration mechanism: if $\lt$ and $\rt$ decide to act at the same time, a Bernoulli$-\frac{1}{2}$ random trial decides which one of them to suppress.

At the beginning of each simulation run, Sniffy (our pet agent $\mathscr{A}$) and its target are placed in random positions in $\env$, denoted $\pos(0)$ and $T$, respectively.
The agent then experiences a training period during which every decision by every BUA is overridden by a random one, resulting in a lazy random walk.
Once the training period is over, the BUAs are given control authority, with Sniffy acting according to their decisions.

Finally, following the indications of \Cref{sim:1:target}, we have chosen to replace the discount parameter of $q=0.999$ with $q=1-\tfrac{1}{N+1}$, to enable a faster response by the discounted learners.

\Cref{fig:sniffy_BUAs} reports the results of our simulations.
Each plot shows the mean, plus/minus standard deviation, over $100$ distinct runs, of the distance of the agent to its target as a function of time, in each setting.

\begin{figure}[ht]
    \centering
    \includegraphics[width=\textwidth]{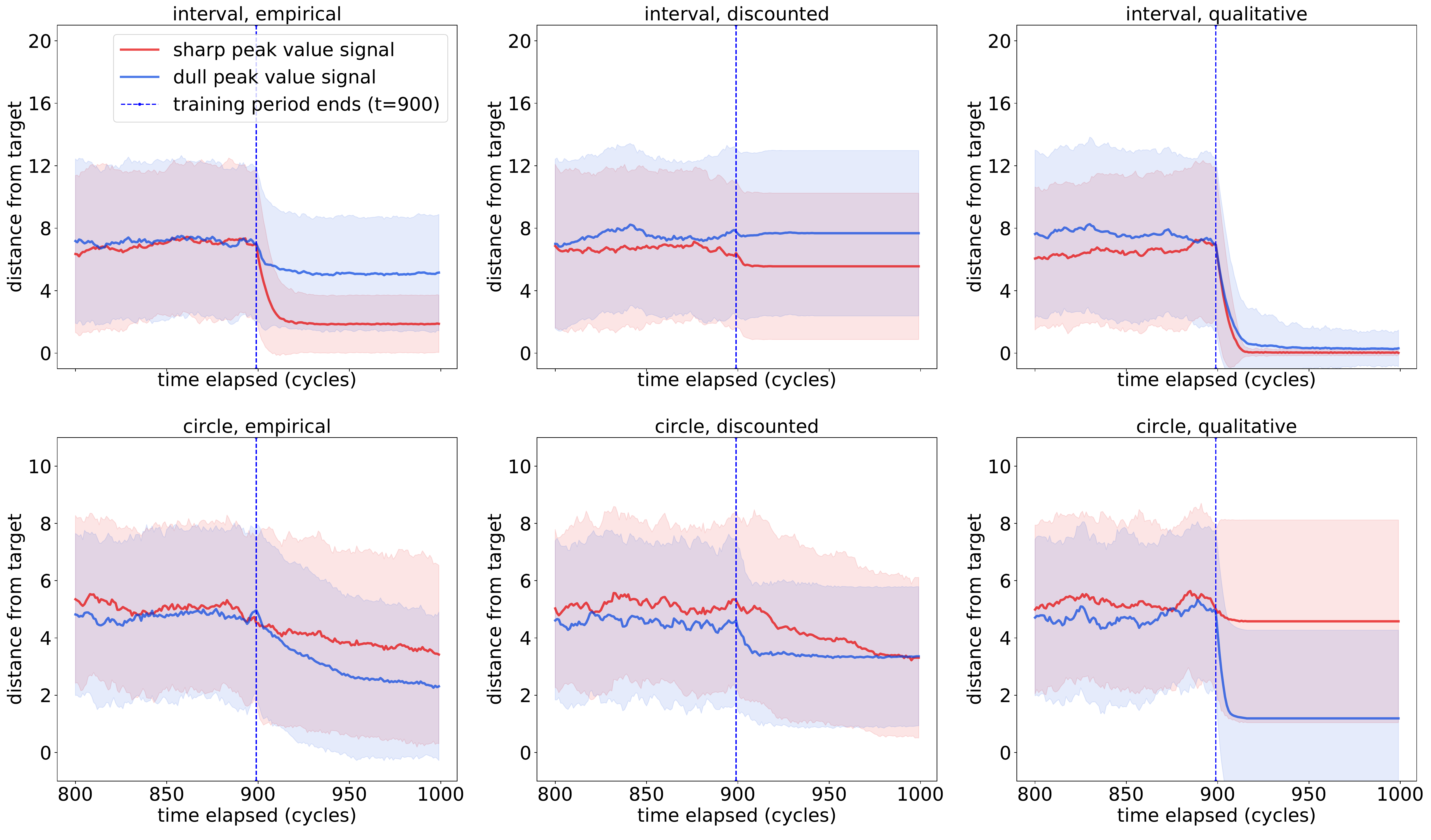}\\
    \caption{BUA implementations of Sniffy (\Cref{sim:2:sniffy}) learning to find a target, with different snapshot types.
    %
    %
    %
    Each graph shows the evolution of Sniffy's mean distance to the target, plus/minus std. deviation, over $100$ runs.
    %
    %
    }
    \label{fig:sniffy_BUAs}
\end{figure}

\Cref{poc:ex:moving bead on an interval} discusses the representations expected to arise in the case of the interval in some detail, explaining Sniffy's success in that environment, shown in the figure.
However, we also notice a deterioration of the results as Sniffy is moved from the interval to the circle.
This is due to subtle interactions between the propagation mechanism generating the BUAs' predictions (which drives decision-making), and the non-trivial homotopy type of the circle, which forces inconsistent states into all the model spaces involved (the latter, we recall, are always contractible).
This discrepancy between the topology of UMA model spaces and the spaces they come to model provides the main motivation for our future project of studying the control of situated agents by {\em networks} of BUAs, where the deliberation among agents is meant to generate an emergent joint representation of reactive behavior patterns with the competence to overcome topological constraints and obstacles (more in~\Cref{discussion}).

\subsubsection{What did Sniffy learn on the circle?}\label{sim:2:sniffy:circle}
All the graphs in \Cref{fig:sniffy_BUAs} indicate a significant change of behavior at the end of training.
It therefore seems sensible to attempt splitting the set of runs in each setting into those finishing closer to the target than to its antipodal point on the circle, as shown in \Cref{fig:split the runs}.
\begin{figure}[ht]
    \centering
    \includegraphics[width=.325\textwidth]{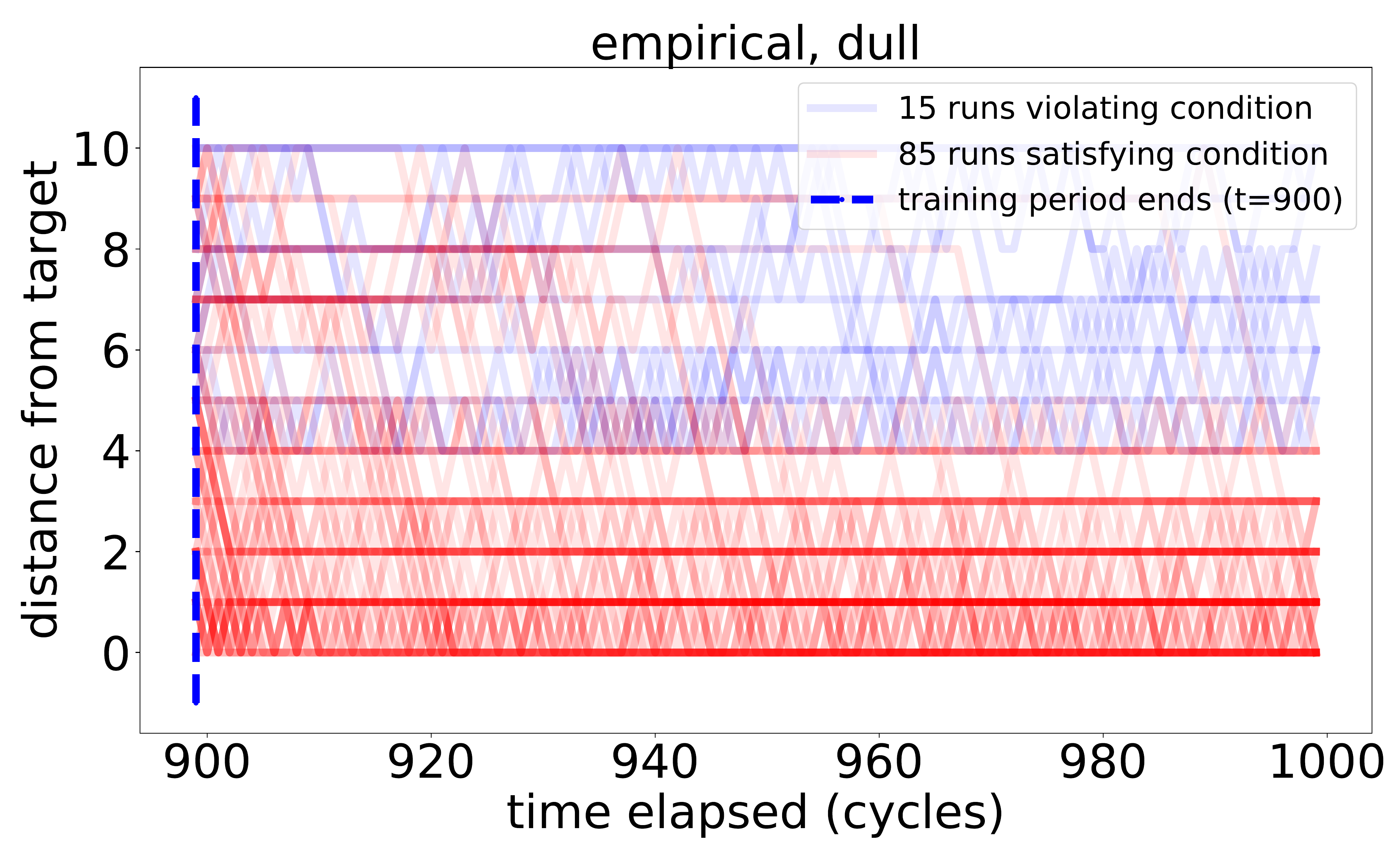}
    \includegraphics[width=.325\textwidth]{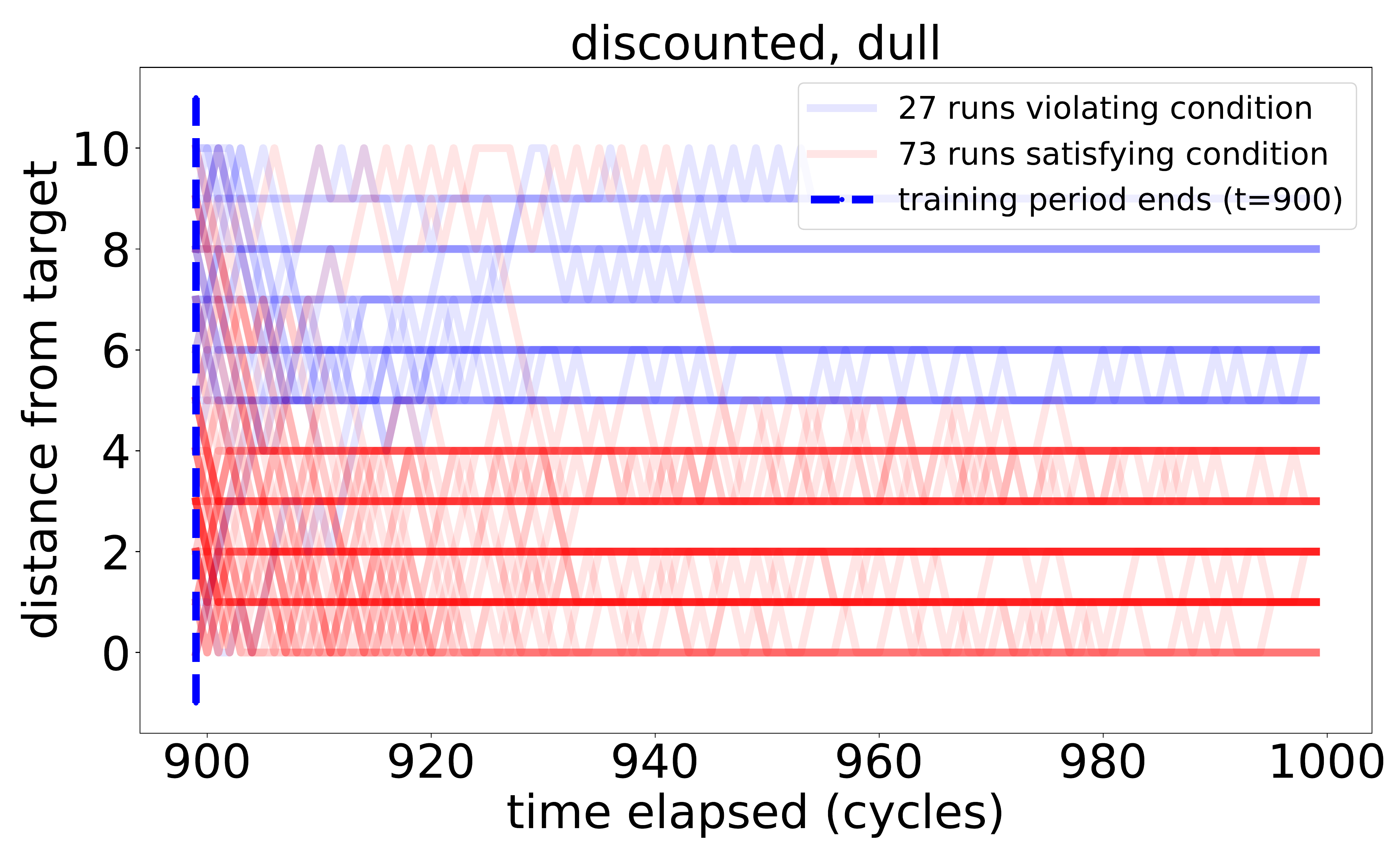}
    \includegraphics[width=.325\textwidth]{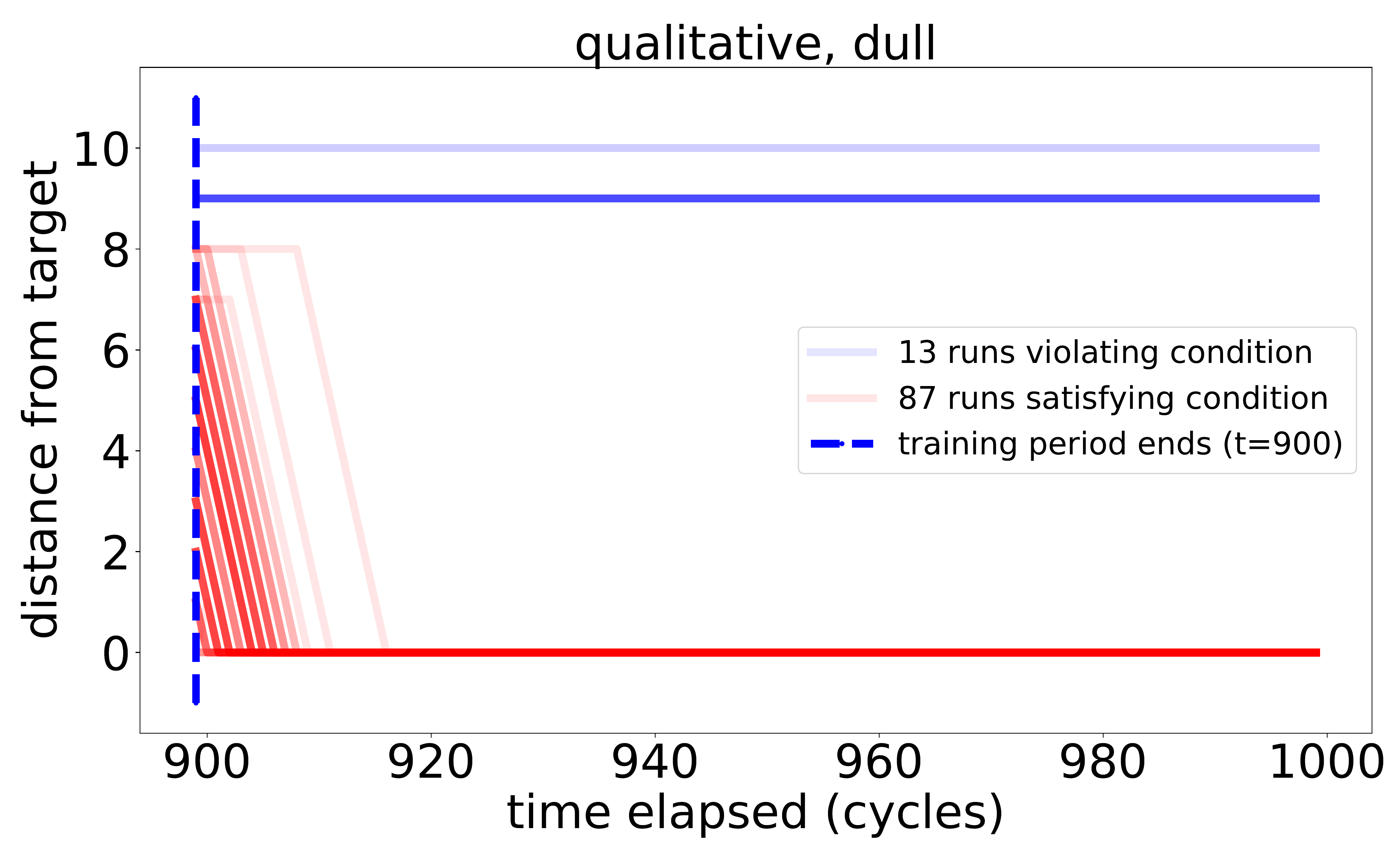}\\
    \includegraphics[width=.325\textwidth]{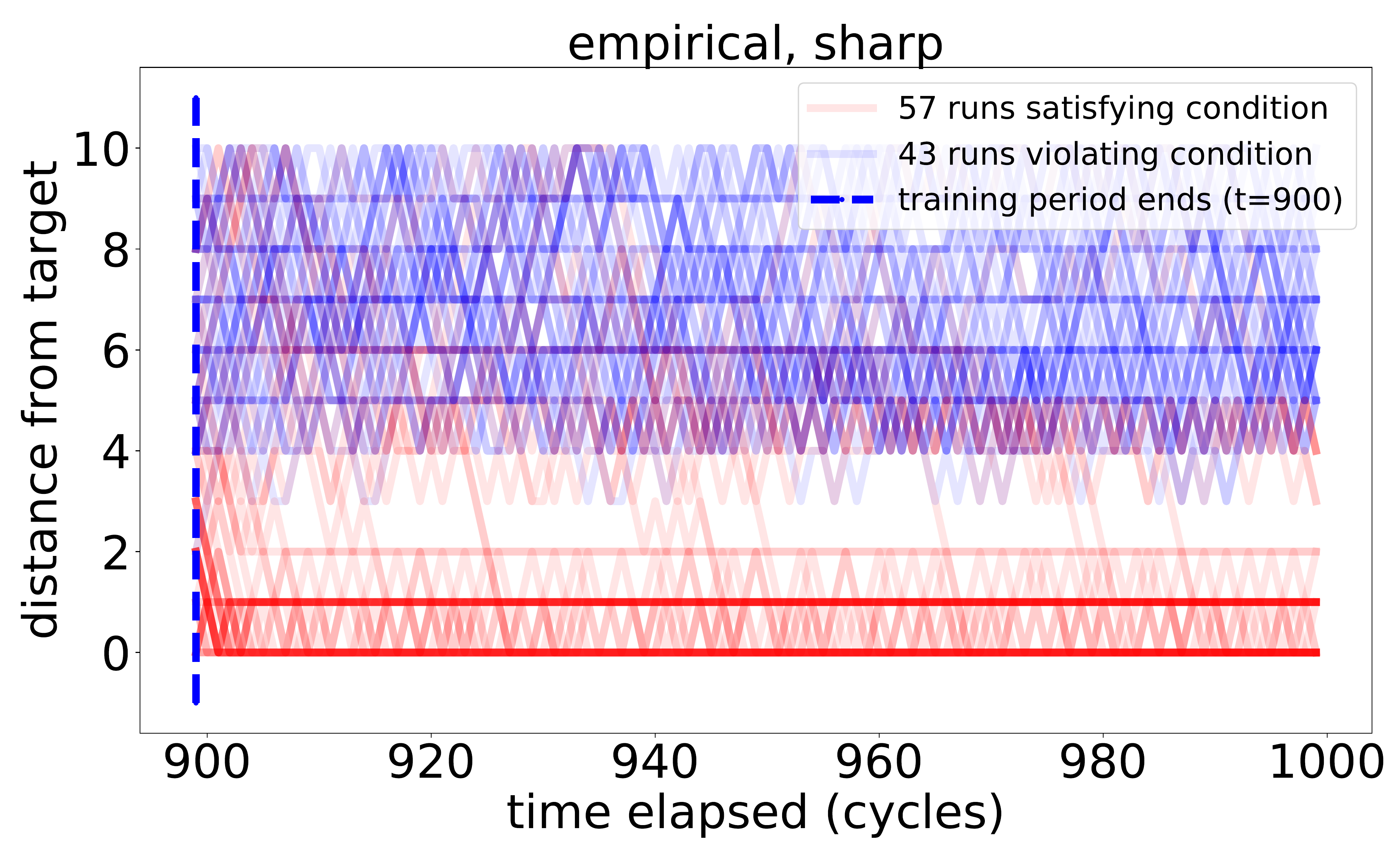}
    \includegraphics[width=.325\textwidth]{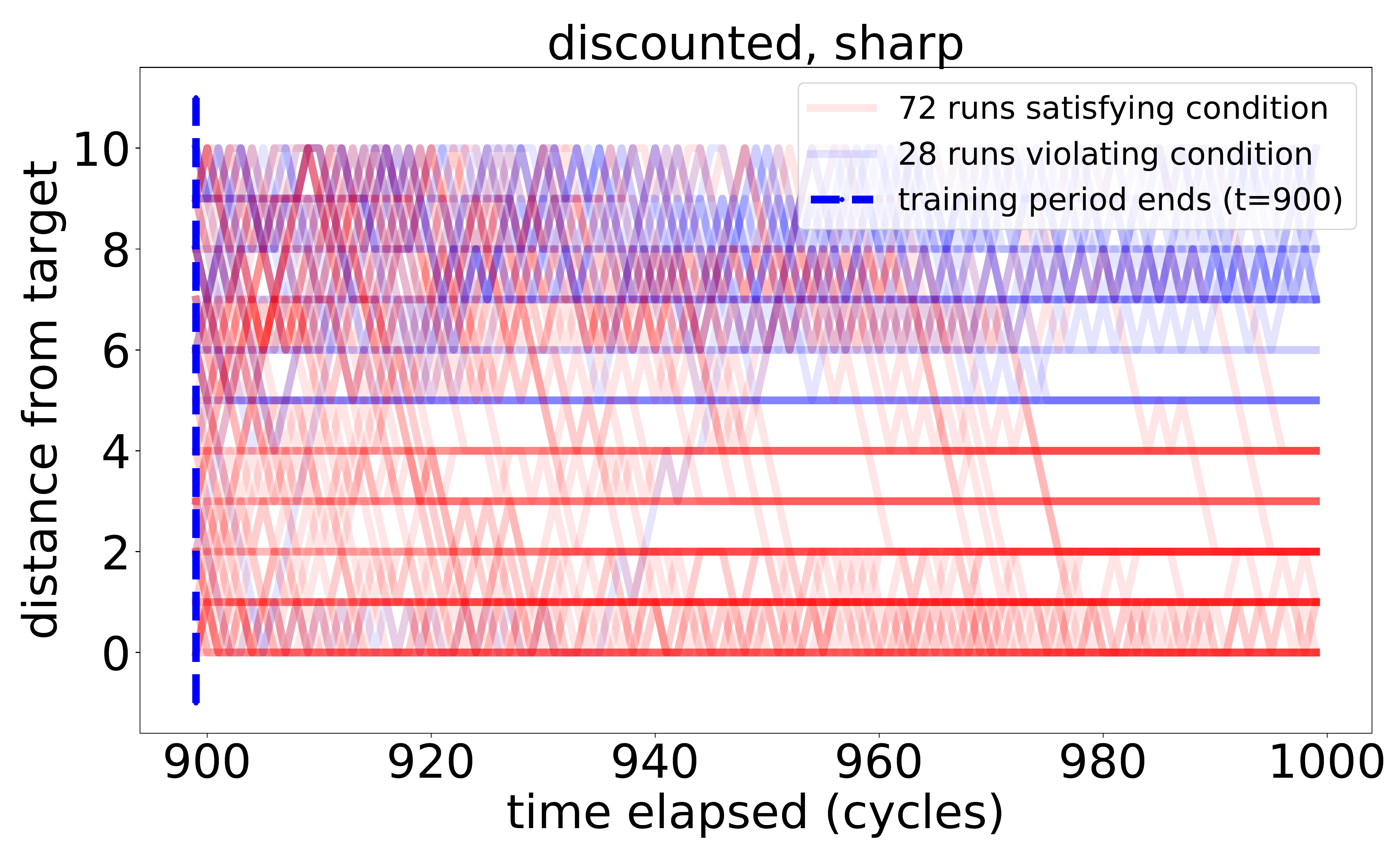}
    \includegraphics[width=.325\textwidth]{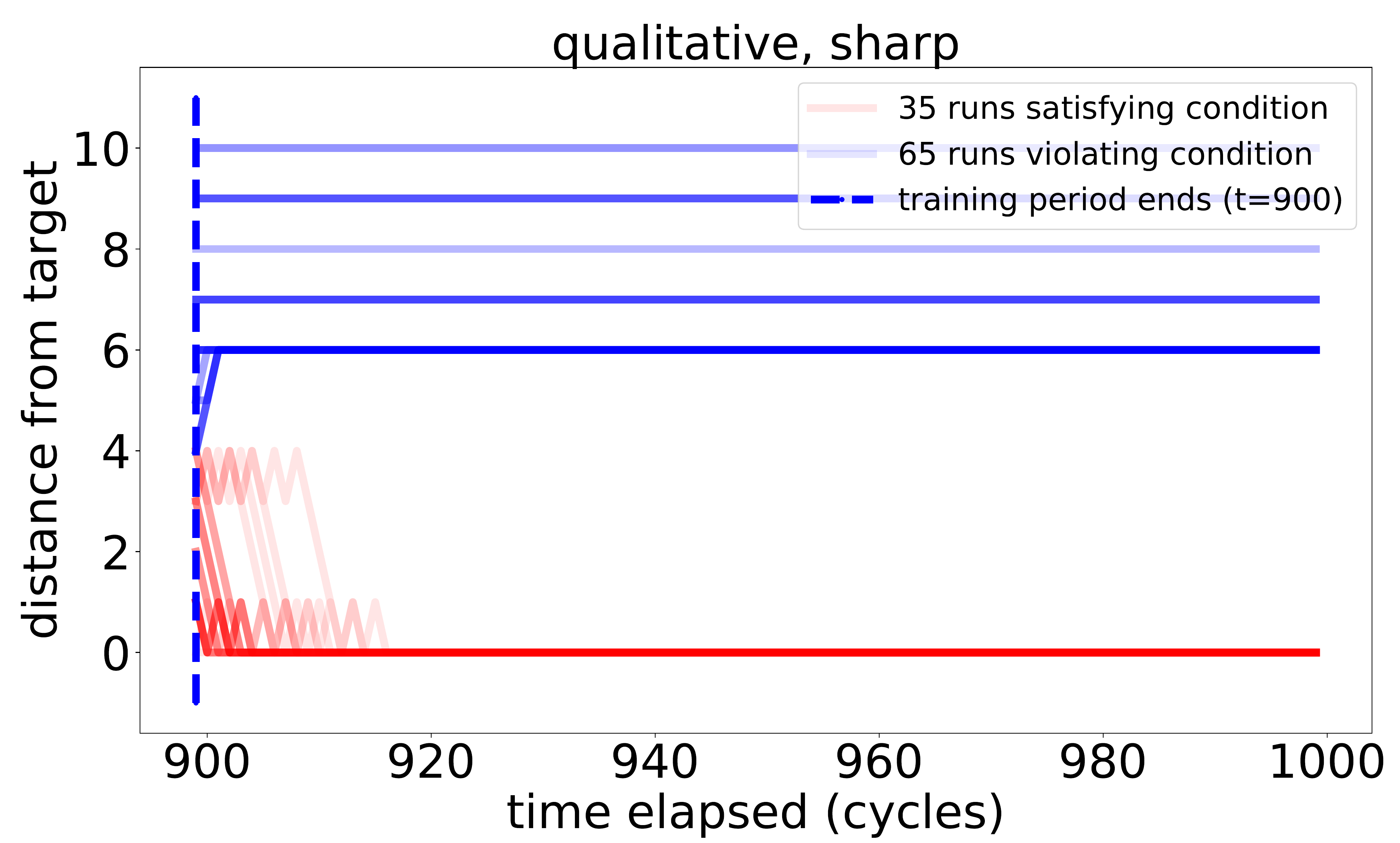}\\
    \caption{Splitting the population of Sniffy's runs on the circle according to their termination points:
    red paths terminate inside $\rho(a_T)$; blue paths terminate outside of $\rho(a_T)$.
    %
    %
    }
    \label{fig:split the runs}
\end{figure}
What emerges is that all the learned representations experience difficulties dealing with the situation loosely characterized as ``Sniffy approaches the point on the circle antipodal to the target''.
Note that the ``dull peak'' qualitative learners emerge as the most apt, both in terms of efficiency and in terms of separation between the desirable and undesirable modes of behavior.
In this setting, the target clearly emerges as an attracting point except for a small neighbourhood of its antipode, which seems to play the role of an unstable equilibrium.
This is reminiscent of gradient descent over the function $f(x)=\dist{x}{T}$ on the unit circle, viewed as a differentiable manifold: the target $T$ is a robust attractive equilibrium, complemented by an unstable equilibrium that is forced by the non-trivial homotopy type of the circle.
Since qualitative snapshots enable direct computation of the eventual values of the snapshot weights, it becomes possible to obtain explicit insights into the behavior learned by Sniffy in this setting.
We refer the reader to \Cref{app:sniffy on the circle} for a detailed discussion proving the preceding claims.
\section{Discussion.}\label{discussion}
Motivated by the goal of implementing well-reasoned general learning on mobile robots, this paper introduces algorithms implementing a simplified version of iterated belief revision and update that is consistent with budgetary constraints on storage space and computational complexity, collectively named ``universal memory architectures'' (UMAs).
We establish and study the mathematical language necessary for the analysis of UMA instances, and show how the standard model-theoretic approach to belief revision gets naturally replaced by the study of the geometry of convex sets in the model spaces represented by UMAs.

By construction, UMA representations are systems of default rules that are closed under counter-positives.
We show that such representations may be learned both by means of sampling and statistical integration of a real-valued signal (empirical and discounted snapshots, \Cref{snapshot:real-valued}), as well as by means of aggregating samples of a ranking function on the space of possible worlds, in the sense of Spohn~\cite{Spohn-OCF_epistemic_state} and Pearl~\cite{Pearl-system_Z} (qualitative snapshots~\Cref{snapshots:qualitative}).
In the latter case, we are able to guarantee the correct encoding of the convex hull, in the learned geometry, of the set of minimum rank worlds, provided sufficient exposure.
Finally, we show the potential of UMA representations for the motivating application by considering its behavior in a pair of simple learning settings simulating a standard task formulation from Robotics: localize a target in the presence of (highly impoverished) sensing in a global frame (Section~\ref{sim:2:sniffy}).

\paragraph{The need for expanding the set of queries (`self-enrichment'). } It is important to state clearly the limitations of UMA learners in the form presented in this paper.
From a practical perspective, attempting to learn a PCR structure for a fixed sensorium will yield {\em no learning at all} in the case of an arbitrary and/or `unstructured' binary sensorium such as the pixel grid of a B/W video camera, where no two pixels are {\it a-priori} correlated.
Consider an even simpler example: the situation of $a,b,c\in\sens$ satisfying $\rho(a)\cap\rho(b)\subseteq\rho(c)$, $\rho(a)\nsubseteq\rho(c)$ and $\rho(b)\nsubseteq\rho(c)$ cannot be encoded by a PCR unless the query set $\sens$ explicitly contains an element whose realization is $\rho(a)\cap\rho(b)$.
Finally, it is clear that PCRs are not geared for studying temporal interactions unless explicitly outfitted with appropriate queries (as in the example of BUAs in \Cref{sim:BUAs}).

Accepting the above as the price of the radical reduction in computational costs achieved by UMA-based learning (as compared to unrestricted iterated belief revision), a natural avenue for increasing the descriptive power of an UMA representation is to allow the set of queries $\sens$ to expand (by adding `meaningful' queries) and contract (by coalescing related queries, or deleting uninformative ones) over time, in a controlled fashion, {\em at a known and minimal cost in computational resources}.
The fixed sensorium $\sens$ should be replaced with a sequence $\sens\at{t}$, as the map $\rho:\sens\to\power{X}$ is replaced with a sequence $\rho\at{t}:\sens\at{t}\to\power{X}$.
Still, the advantage of UMA representations over others is in their efficiency at encoding a model space and reasoning about it in terms of its convex subspaces.
This motivates the search for an enrichment method that meets the lower complexity bound for representing the observed system.

Looking for such a method, one must be mindful that the expansion steps cannot be arbitrary, as it is necessary for each map $\rho\at{t+1}:\sens\at{t+1}\to\power{\spc}$ to be uniquely determined by its predecessor $\rho\at{t}$ and the limited information that was available to the UMA at time $t$.
This suggests two natural elementary expansion operations, which also happen to interact well with our detailed understanding of the geometry of duals:
\begin{description}
    \item[\bf Append a conjunction.] Adding a query of the form $q=a_1\wedge\cdots\wedge a_k$ for some $a_1,\ldots,a_k\in\sens\at{t}$, to form $\sens\at{t+1}=\sens\at{t}\cup\{q,q\com\}$, forces the extension of $\rho$ via $\rho(q)=\rho(a_1)\cap\ldots\cap\rho(a_k)$.
    \item[\bf Append a delayed sensor.] Let $\delay:\spc\to\spc$ denote the operation of truncating the last state from a given history; Then it is possible to introduce a query of the form $q=\delay a$ for $a\in\sens$, where $\delay a$ reports the value of $a$ preceding the current one, or, in other words: $x\in\rho(\delay a)\IFF\delay x\in\rho(a)$.
\end{description}
Observing $\delay(a\wedge b)=\delay a\wedge\delay b$ for all $a,b\in\sens$, we conclude that any composition of the above extension operations determines a unique extension of the original $\rho$.
Hence, an UMA endowed with these enrichment operations is capable, in principle, of eventually representing very rich theories of the observed system, both in terms of Boolean relations among the original sensors and in terms of temporal properties---provided we are willing to accept the cost in resources.
Clearly, the burden is on us to decide when an extension is in order; for what purpose; and how to prevent the population of added sensors from exploding to a prohibitive size.

In the presence of delayed queries, the situation lends itself to the formation of a prediction operator, extending the simplistic one constructed in \Cref{sim:BUAs}.
This makes it possible to formulate learning objectives concerning the quality of prediction.
Our ongoing work exploring analogies with perceptron learning~\cite{Minsky_Papert-perceptrons} is directed towards studying the problem of optimizing prediction through gradual extension of the sensorium using the operations just formulated.

\paragraph{Agents. } The stated motivation for this project was that of producing computationally efficient agents whose reasoning is grounded in a suitably relaxed---though still formally reasoned---form of iterated BR.
At the same time, the model spaces encoded by UMAs are uniquely suited for {\em reactive control}: the selection of a control instruction in direct response to a localized (in time, as well as in space) perception of the task.
At all times that the goal set is represented by a coherent selection on $\sens$ (that is, the goal set is non-empty and convex in the relevant model space), propagation may be used to produce the nearest point projection paths from the current state to the goal set, within the model space, helping determine the appropriate actions as those provably propelling the agent roughly along one of these paths, using the mechanism described in \Cref{sim:BUAs}.

A pertinent question for our current research is whether or not it is possible to employ self-enrichment procedures (see preceding paragraph) to guarantee---at least for some classes of problems---the emergence of a representation with the property that an agent's predictions from time $t$ never fall outside the perceived current state at time $(t+1)$, for all $t$ large enough.
If, and when, that becomes possible, one will have to conclude that any planning failure is due to an obstacle in the relevant UMA model space(s) originating from an attempt to navigate into an impossible perceptual class.
This would open the door to methods for efficient representation of such classes, as well as the leveraging of such representations for correcting the simplistic control scheme of navigation along geodesics. 

\paragraph{Improving representation using multiple agents. } The possible presence of obstacles focuses our attention on another important deficiency of UMA representations.
While the concept representations they encode are {\em always} contractible when regarded as cubical complexes (see \Cref{app:poc:homotopy} for more details), the concept representations corresponding to the ground truth will, more often than not, possess cavities/holes, serving the role of obstacles to navigation along geodesics in the UMA model space, and driving up the complexity of continuous planning~\cite{Farber-topological_complexity}.
An example of this phenomenon is already encountered in our simulations of target localization on the circle, in \Cref{sim:2:sniffy:circle}, and investigated in detail in \Cref{app:sniffy on the circle}.

A possible solution to this problem might lie with the accumulation of a flexible collection of specialized agents, each with its own sensors and its own value signal; each correctly representing some aspects of the `physical' agent's tasks, while having to rely on others in regions of its model space where its predictions fails.
Fairly detailed descriptions of communities of this form have been proposed as possible models of human cognition by Minsky~\cite{Minsky-K_lines,Minsky-SOM}, and studying the dynamics of such communities, charged with governing a situated agent, poses many interesting challenges.

In this context it is important to note that very recent results~\cite{Reverdy_Koditschek-motivational_dynamics}, demonstrating smooth(!) reactive switching between different control alternatives (behaviors/actions) using value-based motivational dynamics, provide a basis for speculation that (1) such methods may be applicable to our setting, too; and (2) formal understanding of the dynamics of the putative Minskian ``societies'' of UMAs just mentioned may be well within our reach.

Developing this approach will require the study of multi-agent systems incorporating means for the formation of ``BUA coalitions'', for lack of a better term, to be recruited for action under appropriate circumstances.
This is also where we expect the mathematical theory behind UMAs to prove most useful. Its categorical underpinnings (the fact that model spaces arise as dual spaces; see \Cref{models:duals}) provide a rigorous framework for comparing different models of the same system, and for studying the interaction between different perceptual components of a single model (see \Cref{poc:ex:moving bead on an interval,app:sniffy on the circle} where we carry out such detailed analysis).

\section*{acknowledgements}
This research was developed in part with funding from Air Force Research Lab (AFRL) grant FA865015D1845 (subcontract 669737-1), and in part with funding from the Defense Advanced Research Projects Agency (DARPA) and the Air Force Research Lab (AFRL) under agreement number FA8650-18-2-7840.
The views, opinions and/or findings expressed are those of the authors and should not be interpreted as representing the official views or policies of the Department of Defense or the U.S. Government. 
The authors are grateful to Siqi Huang, a Penn CGGT Master's graduate, for his relentless work developing a hardware-accelerated implementation of the UMA architecture, making the simulations in this study possible.
We also thank Kostas Karydis for helping proof-read some of the initial material on ranking-based UMAs during his last months as a post-doctoral fellow at Penn's GRASP lab.
%

\bibliographystyle{plain}
\bibliography{memory}

\appendix

\section{Appendix: The Duality Theory of Finite Poc Sets.}\label{poc}
The purpose of this appendix is to review known results about the geometry of duals of finite poc sets, while illustrating them with simple examples which emphasize our application.
An additional goal is to provide a sufficient technical background for proofs of new results in the appendices that follow.

The concept presentation of the dual of a poc set leads to more intuitive understanding of the geometry of poc set duals.
Recall from \Cref{models:concept presentation} that the concept representation of a subset $V\subset\ham$ of vertices of the Hamming cube over a PCS $\sens$ encodes the set of (cubical) faces of the Hamming cube obtained by deleting all faces containing at least one vertex of $\ham\minus V$.
The resulting structure is a (rather special) {\em cubical complex}\footnote{See~\cite{Kozlov-combi_alg_top}, Chapter 2, for a very brief introduction to polyhedral (in particular, cubical) complexes.}.
One way in which such cubical complexes are special is that they are completely determined by their {\em 1-dimensional skeleton}---their collections of vertices and edges.

The resulting freedom to consider a higher dimensional ``enveloping structure'' for $\dual{\ppoc}$ when $\ppoc$ is a poc set over $\sens$ turns out to be useful in many ways, some of which we intend to explore in this section.
\begin{definition}[Dual Cubing]\label[definition]{defn:dual cubing} Let $\ppoc$ be a poc set structure over a finite PCS $\sens$. The {\em dual cubing} $\cube{\ppoc}$ is the cubical complex obtained as the concept representation of the subset $\ppoc^\circ\subset\ham(\sens)$.\qed
\end{definition}
In the very least, the ability to refer to $\cube{\ppoc}$ will make it easier to visualize the graph $\dual{\ppoc}$, exposing its higher dimensional structure and bringing order to what otherwise would have been a chaos of edges (e.g. \Cref{fig:six_way_compass}). 
The notion of a dual cubing also makes it easier to understand cartesian products of dual graphs (\Cref{poc:ex:cartesian product} below).
Finally, we will use the dual cubing to explain some fundamental properties and limitations of PCR presentations (\Cref{poc:ex:circle} below) relating to their universality (\Cref{prop:universality}).

\subsection{Nesting, Transversality and Cubes.}
Fix a poc set $\ppoc$ over a finite PCS $\sens$.
The purpose of this section is to present the known characterizations of the cubes arising in $\cube{\ppoc}$.
Some additional standard terminology will be needed.
The following are from~\cite{Roller-duality}, Section 1.4:
\begin{definition}[proper elements, proper pairs] Let $\sens$ be a PCS.
A {\em proper element} of $\sens$ is any element $a\in\sens$ such that $a\notin\{\mbf{0},\mbf{0}\com\}$.
A pair $\{a,b\}$ of proper elements in $\sens$ is said to be proper, if $b\notin\{a,a\com\}$.\qed 
\end{definition}
\begin{definition}[nesting, transversality]\label[definition]{defn:nesting,transversality} Let $\ppoc$ be a poc set.
For any proper $a,b\in\ppoc$ at most one of the following holds:
\begin{equation}
    a\geq b\,,\quad
    a\geq b\com\,,\quad
    a<b\,,\quad
    a<b\com\,.
\end{equation}
If any one of the above relations holds, we will say that {\em $a$ and $b$ are nested}.
Otherwise, we say that {\em $a$ and $b$ are transverse}.
Furthermore, for any $A\subset\ppoc$, we say that {\em $A$ is nested (transverse)}, if every two elements of $A$ are nested (resp. transverse).\qed 
\end{definition}
Recalling that a poc set is, first and foremost, a partially ordered set, for any subset $S\subset\ppoc$ it makes sense to consider
\begin{equation}\label{eqn:poc minset}
    \min(S):=\set{a\in S}{(\forall b\in S)(b\leq a \rightarrow b=a) }\,.
\end{equation}
Since $\ppoc$ is finite, $\min(S)$ is non-empty whenever $S$ is.
The following is Proposition 10.1 of~\cite{Roller-duality}, restricted to the finite case and parsed into more elementary language:
\begin{lemma}[when a vertex meets a cube]\label[lemma]{lemma:when a vertex meets a cube} Let $\ppoc$ be a finite poc set and let $v\in\ppoc^\circ$.
Let $Q$ be a $d$-dimensional cube of $\cube{\ppoc}$.
Then $v\in Q$ if and only if $\min(v)$ contains a transverse subset $T$ of $\ppoc$ with the property that every vertex $u\in Q$ is of the form $u=\flip{v}{S}:=(v\minus S)\cup S\com$ for some $S\subseteq T$.\qed
\end{lemma}
\begin{figure}[t]
    \centering
    \includegraphics[width=.8\columnwidth]{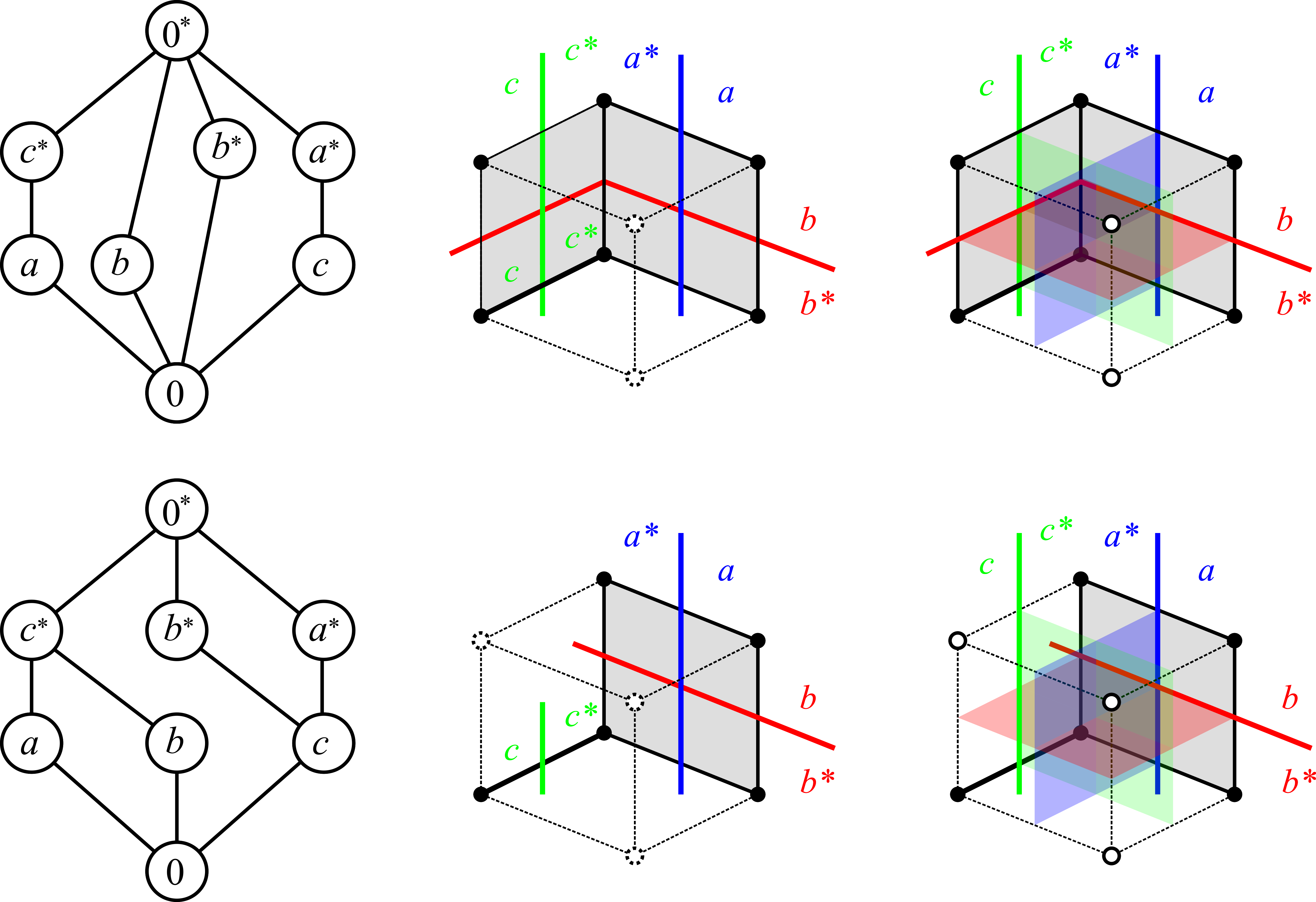}
    \caption{\small Half-spaces of $\ppoc^\circ$ and the hyperplanes of $\cube{\ppoc}$ that generate them (see discussion following \Cref{lemma:when a vertex meets a cube}).
    Introducing the nesting relation $a<c\com$ (top,a) produces the dual cubing (top,b), whose hyperplanes arise as the restrictions of the canonical hyperplanes of the Hamming cube (top,c). Additional relations (bottom,a) further reduce the dual and its hyperplanes (bottom,b), as fewer faces of the dual each hyperplane of the Hamming cube (bottom,c).\normalsize}
    \label{fig:hyperplanes}
\end{figure}
In particular:
\begin{itemize}
    \item every edge ($1$-cube) containing $v$ is spanned by $v$ and a vertex of the form $\flip{v}{a}:=(v\minus\{a\})\cup\{a\com\}$ for some $a\in\min(v)$;
    \item every square ($2$-cube) containing $v$ is spanned by $v$, $\flip{v}{a}$, $\flip{v}{b}$ and $\flip{v}{ab}$ for some transverse pair $\{a,b\}\subseteq\min(v)$.
\end{itemize}
These properties give rise to a new understanding of how the half-spaces $\{\half{a;\ppoc}\}_{a\in\ppoc}$ in $\dual{\ppoc}$ interact with the geometry of $\cube{\ppoc}$:%
one could think of the splitting of $\ppoc^\circ$ in the form $\half{a;\ppoc}\cup\half{a\com;\ppoc}$ as the result of cutting the cubing $\cube{\ppoc}$ along the {\em hyperplane} arising as the union of perpendicular bisectors of edges of the form $\{v,\flip{v}{a}\}$---see \Cref{fig:hyperplanes}.

In addition, the last lemma plays a crucial role in deducing some fundamental properties of $\cube{\ppoc}$ (Proposition 10.2 of~\cite{Roller-duality}):
\begin{theorem} Let $\ppoc$ be a finite poc set. Then $\cube{\ppoc}$ is contractible.\footnote{Contractibility of a topological space is a fundamental notion in Topology, formalizing the idea of a ``space with no holes''. See~\cite{Hatcher-alg_top_textbook}, Chapter 0 for a quick and very intuitive introduction.}\qed
\end{theorem}
Moreover, the lemma implies that $\cube{\ppoc}$ is non-positively curved (see~\cite{Wise-riches_to_raags}, Section 2.1).
This produces a characterization of complexes of the form $\cube{\ppoc}$ (Theorem 10.3 of~\cite{Roller-duality}):
\begin{theorem}[characterization of cubings]\label{thm:cubings are poc set duals}
A cubical complex arises as the dual of a finite poc set if and only if it is contractible and non-positively curved.\qed
\end{theorem}
All the above apply in far more general settings than the finite one: the interested reader should consult~\cite{Roller-duality}.

\subsection{Examples of Duals.}\label{poc:examples of duals} To improve the reader's intuition regarding dual graphs of poc sets, as well as to illustrate one of the example simulations (\Cref{sim:2:sniffy}), we consider a sequence of examples in light of the results of \Cref{models:convexity}.

\subsubsection{Example: a bead on a string.}\label{poc:ex:interval}
Suppose the system being observed consists of a bead strung on a tight piece of string.
The observed state of the system is modeled by the interval $[0,1]$ in the obvious way, so the space of histories $\spc$ is the set of sequences $x=(x_n)_{n=-\infty}^0$, where $\pos(x):=x_0$ corresponds to the current position of the bead given $x$, $x_{-1}$ is the previous position of the bead, and so on.
Let us set $\sens=\{\minP,\minP^\ast,a_1,a_1^\ast,\ldots,a_L,a_L^\ast\}$ with two different poc set structures, $\mbf{P}$ and $\mbf{Q}$, defined by the relations $a_k<a_{k+1}$, $1\leq k<L$ in $\mbf{P}$ and $a_i<a_j^\ast$, $1\leq i< j\leq L$ in $\mbf{Q}$.
These may be regarded as PCR representations of two different sensoria constructed as follows. Let $p_1<\ldots<p_L$ in $(0,1)$ be points that are pairwise at least $\epsilon$ apart, $0<\epsilon<\tfrac{1}{2(L+1)}$.
Then $\mbf{P}$ may be realized by setting $x\in\rho(a_k)\IFF\pos(x)<p_k$ (``threshold sensors''), while $\mbf{Q}$ may be realized, for example, by $x\in\rho(a_k)\IFF\dist{\pos(x)}{p_k}<\epsilon$ (``beacon sensors'').

\begin{figure}[t]
	\begin{center}
		\includegraphics[width=\columnwidth]{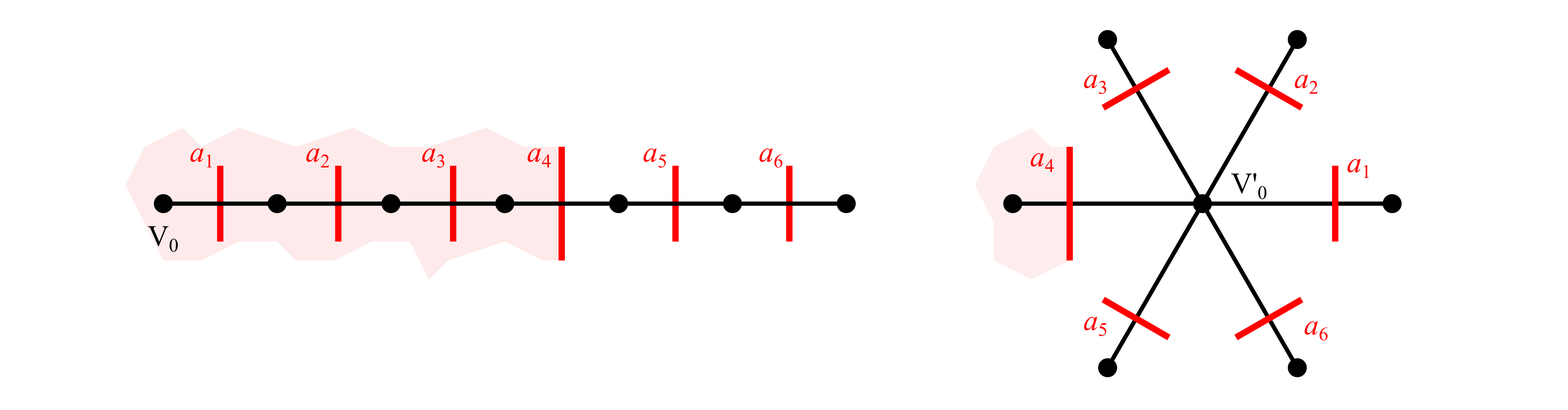}
		\caption{\small The dual graphs $\dual{P}$ and $\dual{Q}$ for the two arrangements of sensors along the real line described in \Cref{poc:ex:interval}: `threshold' sensors encoding a path (left, satisfying $\mbf{P}$), and `beacon' sensors encoding a starfish (right, satisfying $\mbf{Q}$), respectively.\normalsize\label{fig:L-path}}
	\end{center}
\end{figure}

The vertices of $\dual{\mbf{P}}$ have the form $V_k=\{\maxP\}\cup\{a_j^\ast\}_{j>k}\cup\{a_i\}_{i\geq k}$, $0\leq k\leq L$, with an edge joining $V_k$ to $V_{k+1}$ for all $k<L$ (recall that edges in $\dual{\mbf{P}}$ are edges of the Hamming cube $\ham=\ham(\sens)$).
The graph $\dual{\mbf{Q}}$ has a different collection of vertices, dictated by the fact that all pairs $\{a_i,a_j\}$ with $i\neq j$ are incoherent:
there is a `special' vertex $V'_0=\{\maxP,a_1^\ast,\ldots,a_L^\ast\}$ and a collection of `generic' ones, $V'_k=\{\maxP,a_k\}\cup\{a_j^\ast\}_{j\neq k}$;
all the $V'_k$, $k>0$, are adjacent to $V'_0$, and no other pair of vertices are adjacent.
\Cref{fig:L-path} shows $\dual{\mbf{P}}$ (left), which is an $L$-path, and $\dual{\mbf{Q}}$ (right), which we will refer to in the future as a {\it starfish}.
Note how, of the two model spaces, $\dual{\mbf{P}}$ seems to provide the better discretization of $[0,1]$.
Note that both duals are trees.
This is a manifestation of the well-known fact that $\dual{\ppoc}$ is a tree if and only if $\ppoc$ is nested (that is, any two elements of $\ppoc$ are nested).

\subsubsection{Example: Cartesian products of duals.}\label{poc:ex:cartesian product}
The easiest way to join two poc sets together is to form their direct sum:
\begin{definition}\label[definition]{defn:direct sum of poc sets} Let $\mbf{P}$ and $\mbf{Q}$ be discrete poc sets.
Their \emph{direct sum} $\mbf{P}\vee\mbf{Q}$ is defined to be the quotient of their external disjoint union $P\sqcup Q$ by the identification $\minP_\mbf{P}=\minP_\mbf{Q}$ and $\maxP_\mbf{P}=\maxP_\mbf{Q}$, endowed with the following:
\begin{itemize}
	\item $a\leq b\text{ in }\mbf{P}\vee\mbf{Q}\IFF
	(\{a,b\}\subseteq\mbf{P}\text{ and }a\leq b\text{ in }\mbf{P})
	\text{ or }
	(\{a,b\}\subseteq\mbf{Q}\text{ and }a\leq b\text{ in }\mbf{Q})$;
	\item $a=b\com\text{ in }\mbf{P}\vee\mbf{Q}\IFF
	(\{a,b\}\subseteq\mbf{P}\text{ and }a=b\com\text{ in }\mbf{P})
	\text{ or }
	(\{a,b\}\subseteq\mbf{Q}\text{ and }a=b\com\text{ in }\mbf{Q})$.
\end{itemize}
We abuse notation by identifying each element of $\mbf{P}\cup\mbf{Q}$ with the equivalence class in $\mbf{P}\vee\mbf{Q}$ of its natural representative in $\mbf{P}\sqcup\mbf{Q}$.\qed
\end{definition}
Consider the two inclusion maps, $p\colon\mbf{P}\into\mbf{P}\vee\mbf{Q}$ and $q\colon\mbf{Q}\into\mbf{P}\vee\mbf{Q}$, each of which is an injective poc morphism.
The dual maps $p^\circ$ and $q^\circ$ give rise to the median morphism $\mu:(\mbf{P}\vee\mbf{Q})^\circ\to\mbf{P}^\circ\times\mbf{Q}^\circ$ defined by $\mu(w)=(p^\circ(w),q^\circ(w))$, where $p^\circ(w)=w\cap\mbf{P}$ and $q^\circ(w)=w\cap\mbf{Q}$, by definition.
Since every proper pair $a,b\in\mbf{P}\vee\mbf{Q}$ with $a\in\mbf{P}$ and $b\in\mbf{Q}$ satisfies $a\pitchfork b$, it follows that $u\cup v$ is coherent for any $u\in\mbf{P}^\circ$ and $v\in\mbf{Q}^\circ$, and we conclude that $\mu$ is bijective.

Finally, recall that an edge in $\dual{\mbf{P}\vee\mbf{Q}}$ joining $w=\mu(u,v)$ with $w'=\mu(u',v')$ occurs iff $\card{w\minus w'}=1$. Since the intersection of $\mbf{P}$ with $\mbf{Q}$ in $\mbf{P}\vee\mbf{Q}$ is trivial, in terms of $u,u',w,w'$ we obtain:
\begin{equation*}
    \card{w\minus w'}=\card{u\minus u'}+\card{v\minus v'}\,,
\end{equation*}
so that $w,w'$ span an edge if and only if exactly one of the pairs $\{u,u'\}$ or $\{v,v'\}$ spans an edge. Thus, $\mu$ is a median isomorphism of the dual graphs and we have:
\begin{corollary} Let $\mbf{P,Q}$ be discrete poc sets. Then the mapping
\begin{equation}\label{eqn:cartesian product}
    \mu\colon\left\{\begin{array}{ccc}
        \dual{\mbf{P}\vee\mbf{Q}}   &\to&   \dual{\mbf{P}}\times\dual{\mbf{Q}}\\[.25em]
        w   &\mapsto&   (w\cap\mbf{P},w\cap\mbf{Q})
    \end{array}\right.
\end{equation}
is a median-preserving graph isomorphism.\qed
\end{corollary}
For an alternative argument, note that for any $u\in\mbf{P}^\circ$ and $v\in\mbf{Q}^\circ$, if $S\subset\min(u)$ and $T\subset\min(v)$ are transverse sets, then $S\cup T\subset\min(u\cup v)$ and is a transverse set in $\mbf{P}\vee\mbf{Q}$.
Therefore, by \Cref{lemma:when a vertex meets a cube}, every cube in $\cube{\mbf{P}}\times\cube{\mbf{Q}}$ corresponds to a unique cube in $\cube{\mbf{\mbf{P}\vee\mbf{Q}}}$.
Thus $\mu$ from the corollary is much more than an isomorphism of graphs: it extends to an isomorphism of cubical complexes from $\cube{\mbf{P}\vee\mbf{Q}}$ onto $\cube{\mbf{P}}\times\cube{\mbf{Q}}$.

\subsubsection{Example: representing a circle.}\label{poc:ex:circle} Similarly to the example of a bead on a straigh piece of string (\Cref{poc:ex:interval}), one could consider a bead on a circular bracelet, replacing the interval $[0,1]$ with the unit circle $\SS^1\subset\mathbb{C}$ in the complex plane.
This time, let $p_0,\ldots,p_{L-1}\in\SS^1$ be a cyclically ordered collection of marker points, say, $p_k:=\exp(\tfrac{2\pi k\mbf{i}}{L})$.

We will compare several different representations over the PCSs:
\begin{equation}
    \sens=\sens(L)=
    \{\mbf{0},\mbf{0}\com\}\cup
    \{a_0,a_0\com\ldots,a_{L-1},a_{L-1}\com\}\,,\quad
    L\in\NN\,.
\end{equation}
We regard $\sens$ as a sensorium whose realization $\rho=\rho(L,\epsilon)$ is defined by setting $x\in\rho(a_k)$ for a history $x$ if and only if the currect state $x_0$ lies in the open circular arc segment of $\SS^1$ centered at $p_k$ and having radius $\epsilon$.
Depending on the choice of $L$ and $\epsilon$, different PCRs (and duals) may arise.
Specifically, We consider the examples with $L=4$ and $\epsilon=\tfrac{\pi}{4},\tfrac{\pi}{3}$; with $L=6$ and $\epsilon=\tfrac{\pi}{3},\tfrac{\pi}{2}$, to illustrate possible differences and shared qualities.

\paragraph{Jack Sparrow's compass, $L=4$.} Rather than keep track of the indices modulo 4 in this example, let us identify it with a day-to-day object: a compass. We denote
\begin{equation}
    \north:=a_0\,,\quad
    \south:=a_2\,,\quad
    \west:=a_1\,,\quad
    \east:=a_3\,.
\end{equation}
\Cref{fig:compass}(left) depicts the subsets of $\SS^1$ which determine $\rho(a_i)$, $i=0,1,2,3$, for the realizations in the cases $\epsilon=\tfrac{\pi}{3}$ (A) and $\epsilon=\tfrac{\pi}{4}$ (B).
Thinking of $\SS^1$ as the space of all possible positions of a compass needle---the needle of {\em this} compass points in the direction of your heart's greatest desire and that may not be a visit to the magnetic north pole---one should think of, e.g., $N_\alpha:=\rho_\alpha(\north)$, $\alpha\in\{A,B\}$, as the set of positions of the needle with which observer $\alpha$  associates an affirmative answer to the question ``Is the needle pointing North?''.
The difference between the two examples is that $\rho_B(\north),\rho_B(\south),\rho_B(\west),\rho_B(\east)$ are pairwise disjoint, while $\rho_A$ realizes the major directions so that only opposites are disjoint.

\begin{figure}[t]
	\begin{center}
		\includegraphics[width=\columnwidth]{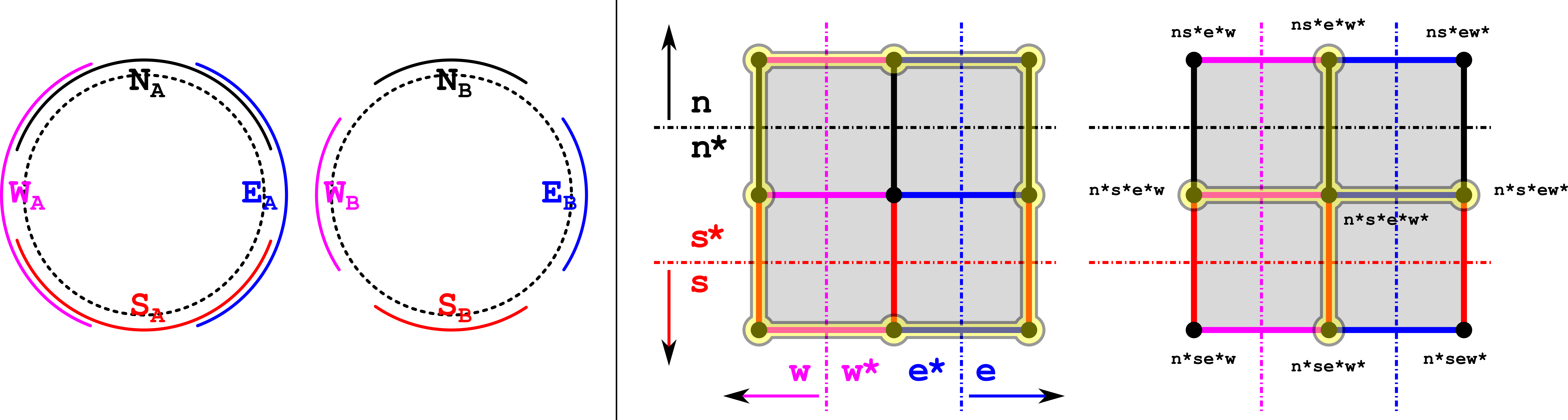}
		
		\caption{\small Illustrating the case $L=4$ in the example of \Cref{poc:ex:circle}: images of the realization maps (left), and the corresponding perceptual classes in the model space $\cube{\ppoc}$ dual to a common PCR representation (right).\normalsize\label{fig:compass}}
	\end{center}
\end{figure}

Let $G$ denote the poc set structure on $\sens$ with relations of the form\footnote{We regard the indices in this example and any arithmetic operations on them as being defined modulo $L$.}
\begin{equation}
    \north<\south\com\,,\quad
    \west<\east\com\,.
\end{equation}
Then $\rho$ is a poc morphism for either choice of $\epsilon$, and the right hand side of \Cref{fig:compass} illustrates the perceptual classes of $\rho$ (yellow highlighting) together with the edges they induce in the ambient structure, $\dual{G}$.
Note how case (A) produces an embedded cycle sub-graph in $\dual{G}$---a coarse but topologically faithful reconstruction of $\SS^1$, which is homotopically non-trivial---while case (B) produces a tree, a space homotopically equivalent to a point.

While providing an illustration for Proposition~\Cref{prop:universality}, this example also highlights the necessity in discussing what properties of the realization map $\rho$ could guarantee a degree of fidelity of the observer's reconstruction of the observed space (the space of histories $\spc$? the `environment' $\SS^1$?) as, say, the sub-graph of $\dual{G}$ induced by the perceptual classes.

\paragraph{Higher dimensions, $L=6$.}
\Cref{fig:six_way_compass} compares the dual graphs/cubings of two poc set representations, each optimal for its corresponding choice of the value of $\epsilon$.
The case $\epsilon=\tfrac{\pi}{3}$ again has the property that non-consecutive $\rho(a_j)$ are disjoint, implying that $\rho$ is a poc isomorphism of the poc set structure
\begin{equation}
    a_i\pitchfork a_j\IFF |i-j|\leq 1\,,\quad
    a_i<a_j\com \IFF |i-j|>1\,,
\end{equation}
onto its image in $\power{\SS^1}$. Denote this poc set structure on $\sens$ by $\ppoc_1$.

The case $\epsilon=\tfrac{\pi}{2}$ has fewer nesting relations among the $\rho(a_j)$, because every three consecutive sets of this form have a point in common.
Formally, $\rho$ is a poc isomorphism of the poc set structure described by:
\begin{equation}
    a_i\pitchfork a_j\IFF |i-j|\leq 2\,,\quad
    a_i<a_j\com \IFF |i-j|>2\,,
\end{equation}
onto its image. Denote this poc set structure on $\sens$ by $\ppoc_2$.
Note that the identity map $\mathrm{id}\colon\ppoc_2\to\ppoc_1$ is a poc morphism, while its inverse is not: a poc morphism $f$ is allowed to map a transverse pair to a nested one, but not the other way around.
The dual of this map embeds $\dual{\ppoc_1}$ in $\dual{\ppoc_2}$. This embedding can be seen clearly in \Cref{fig:six_way_compass}(right).

\begin{figure}[t]
	\begin{center}
		\includegraphics[width=\columnwidth]{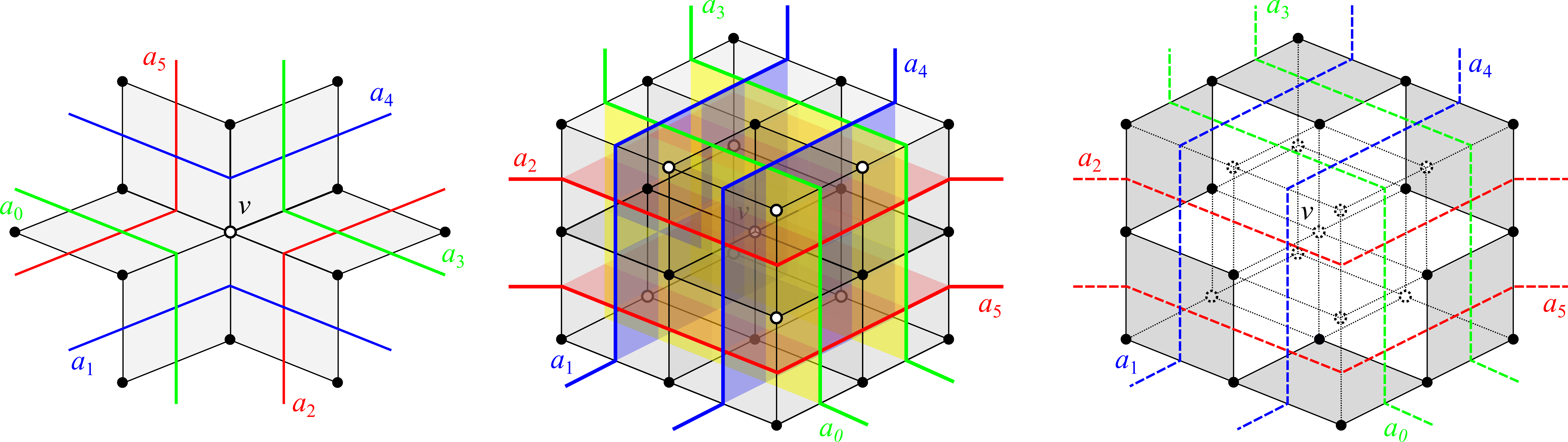}
		
		\caption{\small Duals of `optimal' poc set representations for two distinct realizations of the example of \Cref{poc:ex:circle}, in the case $L=6$: $\cube{\ppoc_1}$ is 2-dimensional (left), while $\cube{\ppoc_2}$ is 3-dimensional (center). Vertices painted black represent consistent models, while vertices painted white represent coherent, but inconsistent models. The diagram to the right shows what is left of $\cube{\ppoc_2}$ when the inconsistent models are deleted from the complex, leaving a sub-complex homotopy-equivalent to the circle $\SS^1$.\normalsize\label{fig:six_way_compass}}
	\end{center}
\end{figure}

\subsubsection{Example: moving bead on an interval.}\label{poc:ex:moving bead on an interval} Returning to the `thresholds' example of \Cref{poc:ex:interval}, we would like to consider it from the point of view of the agents described in \Cref{sim:2:sniffy}.

The interval $[0,1]$ from \Cref{poc:ex:interval} will now be replaced with $[0,L]$, $L$ a positive integer, for convenience.
Once again we are given position sensors $a_1,\ldots,a_L\in\sens$---more sensors will be added to $\sens$ in a moment---with realizations $x\in\rho(a_j)\IFF \pos(x)<j$, where we recall that $x=(x_t)_{t=-\infty}^0$, $x_t\in[0,L]$ is our current notion of a history, and $\pos(x):=x_0$ is the current position of the bead on $[0,L]$.

This time we are interested in reasoning about the possible motion of the bead along the interval, so we introduce delayed sensors $\delay a_j$ into $\sens$ alongside the original position sensors.
Formally, the delay operator acts on histories via $(\delay x)_{t}:=x_{t-1}$, and acts on sensors via $x\in\rho(\delay a)\IFF\delay x\in\rho(a)$, that is: at any time, the current value of $\delay a$ coincides with value of $a$ in the previous cycle.

The bead is endowed with two actuators. One, named $\rt$, whose action at time $t$ pushes the bead one unit to the right along the interval.
The only exception is the position $L$: if $\rt$ is applied there, its contribution to the motion of the bead will be nil.
Similarly, an actuator named $\lt$ pushes the bead one unit toward the endpoint $0$ of $[0,L]$, with no effect when the bead is already there.
Finally, turning on both actuators at the same time results in no motion of the bead in either direction.

Two agents, also named $\rt$ and $\lt$, are each in charge of deciding, respectively, whether to act (turn on their assigned actuator for the duration of one time interval), or not.
Each agent $\alpha\in\{\rt,\lt\}$ maintains two PCR representations: $G^\alpha$ is updated conditioned on the agent having acted, and $G^{\alpha\com}$ is updated conditioned on $\alpha$ resting.

Here we will consider the poc set representations we would like each agent to learn, as we attempt to draw their dual cubings.
For this purpose, we analyze nesting relations in the sensorium
\begin{equation}
    \sens_0:=\{\mbf{0},\mbf{0}\com\}\cup\{a_j,a_j\com\}_{j=1}^L\cup\{\delay a_j,\delay a_j\com\}_{j=1}^L\,.
\end{equation}
In the absence of any additional assumptions, the following relations are consistent with the selected realization (and are, therefore, desirable as part of any learned poc set structure on $\sens_0$):
$a_1<a_2<\cdots<a_L$ encodes the geometry of the interval, and further implies also $\delay a_1<\delay a_2<\cdots<\delay a_L$;
not knowing anything about the actions taken by the actuators one may only be certain of the relations $(\ddagger)$ $\delay a_j<a_{j+1}\,,\;a_j<\delay a_{j+1}$, for $j<L$.
Denote the resulting poc set structure on $\sens_0$ by $\ppoc_0$.

Leveraging our understanding of cartesian products (\Cref{poc:ex:cartesian product}), we set $\mbf{Q}$ to be the sub-poc set of $\ppoc_0$ restricted to just the position sensors $a_j$, while $\delay\mbf{Q}$ will be the sub-poc set of $\ppoc_0$ over the delayed position sensors $\delay a_j$.
Then the identity mapping $\id\colon\mbf{Q}\vee\mbf{\delay Q}\to\ppoc_0$ is a poc morphism, whose dual map is a median-preserving embedding of cubical complexes, of $\cube{\ppoc_0}$ in the square $(L+1)\times(L+1)$ grid arising as $\cube{\mbf{Q}}\times\cube{\mbf{\delay Q}}$.
We conclude that $\cube{\ppoc_0}$ is the cubical complex shown in \Cref{fig:interval with motion}(left), by applying the relations $(\ddagger)$ to erase redundant squares from the grid.

Now suppose that all the correct relations have been (somehow) learned and represented in the collection of PCRs $G^{\rt}$, $G^{\rt\com}$, $G^{\lt}$ and $G^{\lt\com}$.
Because the synchronous application of $\rt$ and $\lt$ yields no motion, there is no way to discriminate between $G^{\rt}$ and $G^{\lt\com}$, as well as between $G^{\lt}$ and $G^{\rt\com}$.
However, the former two poc set presentations will have obtained the relations $a_j<\delay a_j$ in addition to those of $\ppoc_0$, causing the dual cubing to grow even smaller, as highlighted on the right-hand side of \Cref{fig:interval with motion} (of course, a symmetric situation arises for the other pair of indistinguishable representations).

To end this section we note that, by hard-wiring the actuators to never execute both $\rt$ and $\lt$ at the same time, it is possible to disambiguate the representations.
In this regime, the (optimally learned) representations $G^{\rt\com}$ and $G^{\lt\com}$ will remain the same, while the PCRs $G^{\rt}$ and $G^{\lt}$ will each experience a collapse to a non-trivial canonical quotient:
the PCR $G^{\rt}$ witnesses $\delay a_j$ if and only if it witnesses $a_{j+1}$ (the diagonal vertices in \Cref{fig:interval with motion} are inconsistent given $\rt$ is active).
The situation is symmetric (but not identical) for $G^{\lt}$, and we obtain four distinct ``world views'' for each of the observers.

\begin{figure}[t]
    \centering
    \includegraphics[width=\columnwidth]{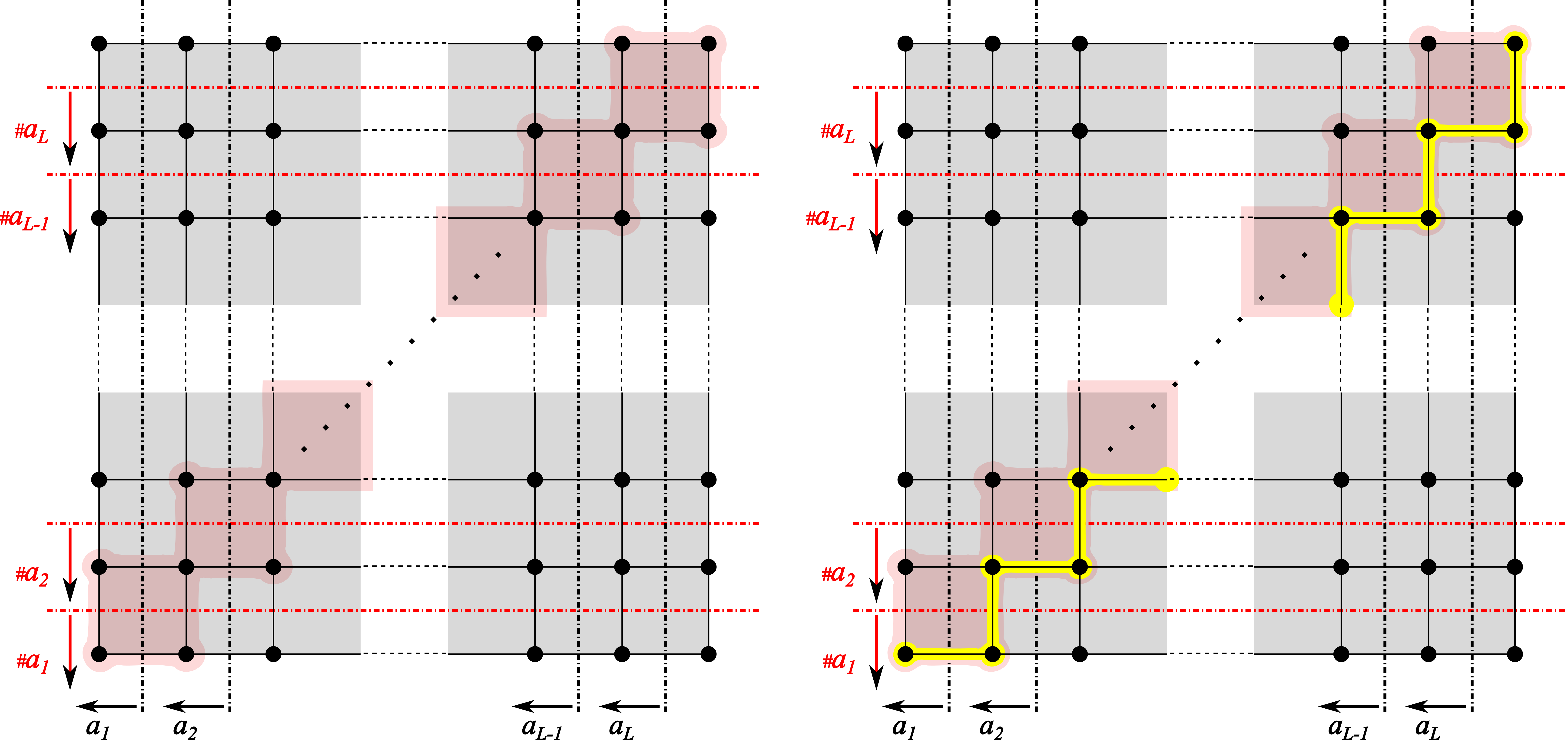}
    \caption{\small Discrete motion of a bead along the interval $[0,L]$, illustrating \Cref{poc:ex:moving bead on an interval}: $\cube{\ppoc_0}$ (red highlighting) shown as a sub-cubing of the 2-dimensional grid $\cube{\mbf{Q}}\times\cube{\mbf{\delay Q}}$ (left); the cubing representing the correct implications for $\rt$ and for $\lt\com$ shown as a sub-complex of $\cube{\ppoc_0}$ (yellow highlighting, right), while the cubing for $\lt$ and for $\rt\com$ is its diagonal mirror image.\normalsize}
    \label{fig:interval with motion}
\end{figure}

\subsection{Homotopy type of the observed space.}\label{app:poc:homotopy} The phenomenon witnessed by the examples of \Cref{poc:ex:circle} is very general, and brings to bear on the capabilities and limitations of knowledge representation using PCRs.

For a fixed PCS $\sens$, a fixed space $\spc$ and PCS morphism $\rho:\sens\to\power{\spc}$, recall (\Cref{models:first notions}) the subset $\model(\rho)$ of the Hamming cube $\ham=\ham(\sens)$ consisting of those models $u\in\ham$ for which $\bigcap_{a\in u}\rho(a)$ is non-empty---the set of possible worlds with respect to $\rho$.
Let $\cube{\rho}$ denote the cubical complex corresponding to the concept presentation of $\model(\rho)$---the set of cubical faces of $\ham$ all of whose vertices lie in $\model(\rho)$.

The authors proved in~\cite{GK-Allerton_2012} that, for sufficiently tame topological spaces $\spc$ and PCS morphisms $\rho:\sens\to\power{\spc}$, the following holds:
\begin{theorem}[Recovery of Homotopy Type]\label{thm:homotopy type representation} 
Suppose that, for every cube $C\in\cube{\rho}$, the set
\begin{equation}
    Z_C:=\bigcup_{u\in C}\bigcap_{a\in u}\rho(a)
\end{equation}
is contractible.
Then $\cube{\rho}$ is homotopy equivalent to $\spc$.\qed
\end{theorem}
In other words, if the collection of queries available to the observer is sufficiently rich that obviously contractible subspaces of $\cube{\rho}$ (cubes) are witnessed by contractible subspaces of $\spc$, then $\cube{\rho}$ has, in the formal sense provided by algebraic topology, the same shape as the observed space $\spc$.

{\em In particular, under the condition of the theorem, if $\ppoc$ is a poc set structure on $\sens$ and $\rho$ is a poc morphism, then the universality of representation by PCRs (\Cref{prop:universality}) implies that $\cube{\rho}\subseteq\cube{\ppoc}$, making $\cube{\ppoc}$ into a minimal contractible model space for $\ppoc$ housing a homotopy model of the observed space, and the discrepancy between the two is precisely the set of unobservable perceptual classes.}

To illustrate the theorem, let us return to the examples of the preceding paragraph to observe that none of the phenomena we have encountered there had happened by accident.
For any $L\geq 2$, a choice of $0<\epsilon\leq\tfrac{\pi}{L}$ leads to $\cube{\ppoc}$ being a tree (a `starfish') containing the vertex $v=\{a_j\com\}_{j=0}^{L-1}$.
Since the set of points in $\SS^1$ witnessing this vertex is disconnected (see \Cref{fig:compass}), the hypothesis of the last theorem fails, making it possible for $\cube{\rho}$ to be contractible, which is exactly what happened for $L=4$, $\epsilon=\tfrac{\pi}{4}$.
At the same time, any choice of $\tfrac{\pi}{L}<\epsilon\leq\tfrac{\pi}{2}$ results in the hypothesis of the theorem being fulfilled, which is why, in the three other cases considered here, $\cube{\rho}$ is homotopy-equivalent to the circle.

Finally, we would like to emphasize that---similarly to the examples considered above---neither the tameness assumptions on $\spc$ and $\rho$ nor the hypothesis of the last theorem are excessive in standard Robotics settings.
First, since the sensor values are often functions of merely the last few visited states, the realization map $\rho:\sens\to\power{\spc}$ will often factor, up to sufficient approximation, through $\power{E\times\cdots\times E}$ where $E$ is the configuration space of the robotic system (similarly to the role played by the circle and the interval in all the preceding examples).
Second, $E$ is often a manifold, possibly with corners, or a cellular complex; in the absence of chaotic behavior, and provided sufficient sensing, it becomes possible to construct a sufficiently fine mesh of sensor values for ``chopping up'' the reduced history space $E\times\cdots\times E$ into small contractible regions as required by our theorem. 
\section{Appendix: Basic Results about PCRs.}\label{proofs:models}

\subsection{Proof of \Cref{prop:when maximal coherent is complete}.}\label{proof:when maximal coherent is complete}
Suppose $G$ is non-degenerate.
Take any $S\in G^\circ$ and any $a\in\sens$. One of the following holds:
\begin{itemize}
	\item $S\cup\{a\}$ is coherent, and hence $a\in S$ (by the maximality property of $S$) and $a\not\leq a\com$;
	\item $S\cup\{a\com\}$ is coherent, in particular $a\com\in S$ and $a\com\not\leq a$;
	\item Neither of the above.
\end{itemize}
In the third case there are two possibilities.
Either $S$ is empty, in which case the statement is that neither $\{a\}$ nor $\{a\com\}$ are coherent, which means that both $a\leq a\com$ and $a\com\leq a$ hold, and putting $a$ inside $N(G)\cap N(G)\com$~ ---~ a contradiction; or there exist $b,c\in S$ such that $a\leq b\com$ and $a\com\leq c\com$. But then $c\leq b\com$~ ---~ a contradiction to $S$ being coherent. Thus we are left with $a\in S$ or $a\com\in S$ for each $a\in\sens$, as desired.

The second assertion trivially implies the third, and the third implying the first follows from the remark preceding \Cref{defn:negligible_degenerate}.\qed

\subsection{Proof of \Cref{prop:function form of dual}.}\label{proof:function form of dual}
It is clear that $\chi$ is injective.
Any $f\in\morph{gc}{G}{\power{}}$ is a function of $\sens$ to $\power{}$, a two-point set, and is therefore characterized by the (possibly empty) set of points on which it obtains the value $\mbf{1}$.

Now let us verify that $\chi$ is well-defined, that is: that the set $S=f\inv(\mbf{1})$ is a maximal coherent subset of $\sens$ with respect to $G$.
Indeed, were $a,b\in S$ such that $a\leq b\com$, this would force $\mbf{1}=f(a)\leq f(b)\com=\mbf{0}$~ ---~ a contradiction.

Finally, we prove the surjectivity of $\chi$. Given a maximal coherent set $S$, \Cref{prop:when maximal coherent is complete} implies $S$ is a selection on $\sens$. This means that the function $f:\sens\to\power{}$ defined by $f(a)=\mbf{1}\IFF a\in S$ satisfies the identity $f(a\com)=f(a)\com$. We claim that $f$ is a PCR morphism. Since $\chi(f)=S$, proving this claim will finish the proof of the current proposition.

Suppose $f$ is not a morphism. Then there is $ab\in G$ satisfying $f(a)\not\leq_G f(b)$. In the current setting this is tantamount to $f(b)=\mbf{0}$ and $f(a)=\mbf{1}$, or, equivalently, $f(b\com)=f(a)=\mbf{1}$. In turn, this means $a,b\com\in S$. However, $S$ is forward-closed (as is any maximal coherent set), so $a\in S$ and $ab\in G$ imply $b\in S$. With $b\com\in S$ we obtain a contradiction.
\qed

\subsection{Proof of \Cref{prop:universality}.}\label{proof:universality}
The proof extends a standard argument from Sageev-Roller duality theory.
Given $\spc$ and $\rho$, pick any point $x\in\spc$. By definition, $\xi=\rho\com(x)$ belongs in $G^\circ$ if and only if no $a,b\in\xi$ satisfy $a\leq b\com$ in $G$. Since $\rho$ is order-preserving, having $a\leq b\com$ for $a,b\in\xi$ would imply $\rho(a)\cap\rho(b)=\varnothing$ while $x\in\tilde\rho(a)\cap\tilde\rho(b)$ at the same time---contradiction. Thus, $\xi\in G^\circ$ for all choices of $x\in\spc$, proving the first assertion of the proposition.
To verify the second one, consider the choice of $\spc=G^\circ$ with $\rho:\sens\to\power{\spc}$ given by $\rho(a)=\set{U\in G^\circ}{a\in U}$. It is easily verified that $\rho$ is a morphism and that $\rho\com:\spc\to G^\circ$ is the identity map (and hence surjective), finishing the proof.\qed

\subsection{Proof of \Cref{prop:canonical quotient}.}\label{proof:canonical quotient}
Let $G$ be a fixed non-degenerate PCR over $\sens$. For every $a\in\sens$, recall $[a]_G=\up{a}\cap\down{a}$, and recall the definition of $\pi=\pi_G$:
\begin{equation}
	\pi_G(a):=\left\{\begin{array}{ll}
		[a]_G		&\text{if }a\notin N(G)\cup N(G)\com\\
		N(G)	&\text{if }a\in N(G)\\
		N(G)\com&\text{if }a\in N(G)\com		
	\end{array}\right.
\end{equation}
Here are a few natural observations:
\begin{itemize}
	\item For any $a\in\sens$, $[a]_G\com=[a\com]_G$, where for any set $S\subseteq\sens$ we remember that $S\com:=\set{a\com}{a\in S}$.
	\item Since $N(G)$ is backwards-closed, $N(G)$ is a union of strong components of $G$: indeed, if $a\in N(G)$ then every $b\in[a]_G$ satisfies $b\leq a$, which implies $b\in N(G)$; hence $[a]_G\subseteq N(G)$.
	\item Analogously for $N(G)\com$, since it is forward-closed.
\end{itemize}
This allows for the construction of a new PCR $\widehat{G}$ over the PCS $\widehat{\sens}:=\set{\pi(a)}{a\in\sens}$, by setting $\widehat{G}=\set{\pi_G(a)\pi_G(b)}{a\leq_G b}$. We claim that $\widehat{G}$ induces on $\widehat{\sens}$ the structure of a poc set. For this it will suffice to show that $\widehat{G}$ is a non-degenerate PCR and a partial order.

First we show that $\widehat{\sens}$ is a PCS. The identity $\pi(a\com)=\pi(a)\com$ yields  $\pi(a)\com\com=\pi(a\com\com)=\pi(a)$ for all $a\in\sens$.
Suppose some $a\in\sens$ satisfied $\pi(a)\com=\pi(a)$.
Since $G$ is non-degenerate, this means $a\notin N(G)$ and $a\notin N(G)\com$.
But then $\pi(a)=[a]_G$, at the same time, and the equality $[a\com]_G=[a]_G$ implies both $a\leq a\com$ and $a\com\leq a$---contradicting non-degeneracy.
We conclude that $\pi(a)\com\neq\pi(a)$ for all $a\in\sens$, and, since $\pi$ is surjective, $\widehat{\sens}$ is a PCS.

It is clear now that $\widehat{G}$ is a PCR, by construction. Suppose now that $A\in\widehat{\sens}$ lay in $N(\widehat{G})$.
Then $A\leq A\com$, and writing $A=[a]_G$, $a\in\sens$ we obtain $a\in N(G)$, showing that $A=N(G)$.
Thus, $N(\widehat{G})$ is trivial, as desired.

$\widehat{G}$ is partially ordered by general considerations, so to conclude that $\widehat{G}$ is a poc set, it remains to verify that $N(G)$ is its minimum.
%
Now, $\mbf{0}\in N(G)$ and $\widehat{\mbf{0}}=\pi(\mbf{0})=N(G)\in\widehat{\sens}$ imply that the edge $\widehat{\mbf{0}}\pi(a)\in\widehat{G}$ for all $a\in\sens$.
Since $\pi$ is surjective, $\widehat{\mbf{0}}$ is the minimum element of $\widehat{\sens}$ with respect to the new partial order.

Finally, let $\ppoc$ be any poc set, and let $f:G\to\ppoc$ be any PCR morphism.
Then $f$ is constant on $\pi(a)$ for all $a\in\sens$, which defines the injective set map $\widehat{f}:\Gamma\to\ppoc$ via $\widehat{f}([a])=f(a)$.
This map is a PCR morphism of complemented graphs by construction, and is, therefore, a morphism of $\widehat{G}$ into $\ppoc$.
If $f':\widehat{G}\to\ppoc$ is any poc morphism satisfying $f=f'\circ\pi$, then for any $a\in\sens$ we have $f'(\pi(a))=f(a)=\widehat{f}(\pi(a))$. Since $\pi$ is surjective, $f'$ coincides with $\widehat{f}$.\qed

\subsection{Proof of \Cref{cor:all duals are poc set duals}.}\label{proof:all duals are poc set duals}
Let $G$ be a PCR over a PCS $\sens$, and let $\pi_G:G\to\widehat{G}$ be the canonical quotient map. We apply \Cref{prop:canonical quotient} with $\ppoc=\power{}$.

For any morphism $\varphi:G\to\power{}$ there exists one and only one morphism $\widehat{\varphi}:\widehat{G}\to\power{}$ satisfying $\varphi=\widehat{\varphi}\circ\pi_G$. Now, thinking of $\widehat{\varphi}$ as an element of $\widehat{G}^\circ$, we may write, by \Cref{prop:function form of dual,defn:dual map}, $\pi_G^\circ(\widehat{\varphi})=\widehat{\varphi}\circ\pi_G=\varphi$.

On the other hand, for any $\psi:\widehat{G}\to\power{}$, we may write $\widehat{\pi_G^\circ(\psi)}=\widehat{\psi\circ\pi_G}=\psi$, the last inequality following from the uniqueness assertion of \Cref{prop:canonical quotient}, applied to the morphism $\psi\circ\pi_G$.

We conclude that the map $\morph{cg}{G}{\power{}}\to\morph{cg}{\widehat{G}}{\power{}}$ defined by $\varphi\mapsto\widehat{\varphi}$ is an inverse of $\pi_G^\circ:\morph{cg}{\widehat{G}}{\power{}}\to\morph{cg}{G}{\power{}}$, as desired.\qed

\subsection{Proof of \Cref{cor:naturality of canonical quotients}.}\label{proof:naturality of quotients} We apply \Cref{prop:canonical quotient} again, to the PCR $G$, the poc set $\ppoc=\widehat{H}$ and the morphism $g:=\pi_H\circ f:G\to\widehat{H}$, to conclude there exists one and only one morphism $\widehat{g}:\widehat{G}\to\widehat{H}$ satisfying $g=\widehat{g}\circ\pi_G$. Substituting $g=\pi_H\circ f$ we see that $\widehat{g}$ is the required morphism.\qed

\section{Appendix: Convexity theory of PCR duals.}\label{app:convexity}
The purpose of this section is to provide a self-contained account of our results regarding coherent projection and the use of propagation for the computation of nearest-point projections in poc set duals.
Throughout this section, $G$ will be a fixed non-degenerate PCR over a finite PCS $\sens$.
Moreover, without loss of generality (through replacing $G$ with its canonical poc quotient $\widehat{G}$), we may assume $G$ is a poc set.
Since $G$ is fixed, we will simplify notation by writing $(\leq)$ instead of $(\leq_G)$, `coherent' instead of `$G$-coherent', and so on, throughout this section.

\subsection{Proof of \Cref{lemma:halfspace properties}:}\label{proof:halfspace properties}
Property (1) is a restatement of the fact that, if $G$ is non-degenerate, then every coherent subset of $\sens$ is contained in a coherent complete $\ast$-selection.
Items (2,3) are straightforward from the definition.

For item (4), observe that, if $u\in\half{a;G}$ and $a\leq b$, then $b\in u$ as well: indeed, if $b\notin u$, then $b\com\in u$ and $a\leq(b\com)\com$ means $u$ is incoherent.
We conclude that $a\leq b$ implies $\half{a;G}\subseteq\half{b;G}$, from which (4) readily follows.

To prove item (5), we observe that $S=\up{S}$ implies $S=\up{\min(S)}$, and then we apply (4).

Finally, one direction of (6) amounts to (3).
To prove the converse, suppose $S_1,S_2\in\mbf{C}(G)$ are such that $(\dagger)$ $\half{S_1;G}\subseteq\half{S_2;G}$, and suppose there exists $s\in S_2\minus S_1$. If there is a $u\in\half{S_1;G}$ such that $s\in\min(u)$, then its neighbor $v:=\flip{u}{s}$ in $\dual{G}$ (recall \Cref{lemma:when a vertex meets a cube}) contains $S_1$ but not $S_2$, contradicting $(\dagger)$.
We are left to prove that $\half{S_1;G}$ must contain such a $u$.

Indeed, pick any $v\in\half{S_1;G}$ such that the number $N(v)$ of $a\in v$ satisfying $a<s$ is smallest possible. Since $s\in S_2$, we have $s\in v$, by $(\dagger)$.
If $s\notin\min(v)$ (otherwise we are done), then $N(v)>0$ and we may find an $a\in\min(v)$ with $a<s$. Consider the vertex $\flip{v}{a}$: since $S_1=\up{S_1}$, we conclude that $a\notin S_1$, so $S_1\subset\flip{v}{a}$ and $s\in\flip{v}{a}$, with $N(\flip{v}{a})=N(v)-1$. This contradicts the choice of $v$, and we are done.\qed

\subsection{Proof of \Cref{prop:coherent approximation}:}\label{proof:coherent approximation}
Suppose $B\in G^\circ$ is such that $\ellone{A}{B}\leq\ellone{A}{u}$ for all $u\in G^\circ$. We must show that $\coh{}{A}\subseteq B$.

Suppose $a\in\coh{}{A}\minus B$. Then $a\com\in B$ and there is an element $b\in\min(B)$ with $b\leq a\com$.
Note that $u:=\flip{B}{b}$ is then also an element of $G^\circ$, by \Cref{lemma:when a vertex meets a cube}.

Now, if $b\in A$, then $a\in\up{A}\com$, contradicting $a\in\coh{}{A}=\up{A}\minus\up{A}\com$.
Therefore, $b\com\in A$ (since $A$ is a complete $\ast$-selection), but then $u$ satisfies:
\begin{equation}
    \ellone{A}{u}=\card{u\minus A}=\card{B\minus A}-1<\ellone{A}{B}\leq\ellone{A}{u}
\end{equation}
---a contradiction again. We conclude that $\coh{}{A}\subseteq B$, as desired.\qed

\subsection{Proof of \Cref{prop:coherent projection}:}\label{proof:coherent projection}
Recall that $A\subseteq\up{A}$, $\up{\up{A}}=\up{A}$ and $\down{A\com}=\up{A}\com$ for all $A\subseteq\sens$.
We check that $\coh{}{A}$ is coherent for all $A$.
For suppose that $b,c\in\coh{}{A}$ satisfy $b\leq c\com$.
Then $b\in\up{A}$ implies $c\com\in\up{\up{A}}=\up{A}$, and therefore $c\in\up{A}\com$. But then $c$ cannot lie in $\coh{}{A}$.

Next, we verify that $\coh{}{A}$ is forward-closed.
It suffices to verify $\up{\coh{}{A}}\subseteq\coh{}{A}$.
By definition we have $\coh{}{A}\subseteq\up{A}$, hence $\up{\coh{}{A}}\subseteq\up{\up{A}}=\up{A}$, and it remains to check that no $b\in\up{\coh{}{A}}$ belongs to $\up{A}\com$;
were there such a $b$, there would have been $a\in\coh{}{A},\,c\in A$ with $a\leq b$ and $c\leq b\com$, implying $a\leq c\com$~ ---~ a contradiction to $a\notin\down{A\com}=\up{A}\com$.
This proves (a).

Now let us calculate: $\coh{}{\coh{}{A}}=\up{\coh{}{A}}\minus\up{\coh{}{A}}\com=\coh{}{A}\minus\coh{}{A}\com=\coh{}{A}$, the last equality due to $\coh{}{A}$ being coherent. At the same time, if $A$ itself is coherent then $\coh{}{A}=\up{A}\supseteq A$. Moreover, this shows $\coh{}{A}=A$ whenever $A\in\mbf{C}(G)$. Finally, if $A=\coh{}{A}$ then $A\in\mbf{C}(G)$ because $\coh{}{A}$ always is.\qed

\subsection{Proof of \Cref{prop:projection formula}:}\label{proof:projection formula}
The proof of the projection formula will require additional notions and results from \cite{Roller-duality}, which we now recall.

\subsubsection{Separators and Gates} 
\begin{definition}\label[definition]{defn:separator} For any $K,L\subseteq G^\circ$, the set
\begin{equation}\label{eqn:separator}
	\sepr{K,L}:=\set{a\in G}{K\subseteq \half{a;G})\,,\;L\subseteq \half{a\com;G}}
\end{equation}
is called the {\emph separator of $K$ and $L$} in $G^\circ$.\qed
\end{definition}
The inequality $\ellone{u}{v}\geq\card{\sepr{K,L}}$ follows immediately for all $u\in K$ and $v\in L$. This motivates:
\begin{definition}\label[definition]{defn:gate}
Let $K,L\subseteq G^\circ$. A \emph{gate for $K,L$} is a pair of points $u\in K$, $v\in L$ such that $\ellone{u}{v}=\card{\sepr{K,L}}$.\qed
\end{definition}
The following result is well known in our setting:
\begin{proposition}\label[proposition]{prop:gates exist} Let $K,L$ be non-empty convex subsets of a median graph and let $u\in K$ and $v\in L$. Then $u,v$ form a gate for $K,L$ if and only if $\proj{K}{v}=u$ and $\proj{L}{u}=v$. Moreover, the pair $K,L$ has a gate.\qed
\end{proposition}
We will apply this proposition without proof. An important consequence for us is the following:
\begin{lemma}\label[lemma]{lemma:supporting halfspace} Suppose $S\subset\sens$ is coherent, and $K=\half{S;G}$. Then, for any $a\in\sens$, if $K\subseteq\half{a;G}$ then there exists $s\in S$ such that $s\leq a$.
\end{lemma}
\begin{proof} Let $u\in K$ and $v\in L:=\half{a\com;G}$ form a gate.
Since $v\notin K$, there exists $s\in S$ such that $v\in\half{s\com;G}$. 

Suppose there existed a $w\in L$ with $w\in\half{s;G}$, and consider $m=\med{u}{v}{w}$.
Then, $a\com\in v,w$ implies $a\com\in m$, but the inequality
\begin{equation}
	\ellone{u}{v}=\ellone{u}{m}+\ellone{m}{v}\geq\ellone{u}{m}
\end{equation}
implies $m=v$, since $v=\proj{L}{u}$.
On the other hand, $s\in u,w$ implies $s\in m$---hence $s\in v$, a contradiction.

We have shown that $L=\half{a\com;G}$ is contained in $\half{s\com;G}$.
Equivalently, $a\com\leq s\com$, which is the same as $s\leq a$.\qed
\end{proof}
\begin{lemma}\label[lemma]{lemma:general projection_intersecting} Suppose $K,L$ are non-empty convex subsets of $\dual{G}$. If $K\cap L\neq\varnothing$, then $\proj{K}{L}=\proj{L}{K}=K\cap L$.
\end{lemma}
\begin{proof} Clearly, if $v\in K\cap L$ then $\proj{L}(v)=v$, so $K\cap L\subset\proj{L}{K}$. For the reverse inclusion, suppose $v\in\proj{L}{K}$ and write $v=\proj{L}{u}$, $u\in K$. Pick any point $w\in K\cap L$. Setting $m=\med{w}{v}{u}$ we note that $m\in L$ (because $w,v\in L$) and
\begin{equation*}
	\ellone{u}{v}=\ellone{u}{m}+\ellone{m}{v}\geq\ellone{u}{m}\,.
\end{equation*}
The uniqueness of projection forces $v=\proj{L}{u}$ to coincide with $m$. However, since $w,u\in K$ we also have $m\in K$, showing $v\in K\cap L$.\qed
\end{proof}
We are now ready for the proof of one more lemma.
\subsubsection{Proof of \Cref{lemma:divergence}.}\label{proof:divergence} Since $T\in\mbf{C}(G)$ and $u$ is a complete $\ast$-selection, we have $\sepr{\half{T},u}=u\com\cap T=T\minus u$. Since $S\subseteq u$, we have $T\minus u\subseteq T\minus S$. Overall, this yields $\ellone{u}{\half{T}}=\card{\sepr{u,\half{T}}}\leq\card{T\minus S}$, as required.\qed

\subsubsection{Computing Nearest Point Projection Maps}
We now offer an explicit construction of a geodesic path in $\dual{G}$ emanating from a given vertex $u$ and terminating at its unique nearest point in a specified convex target set:
\begin{proposition}\label[proposition]{prop:constructing geodesics} Suppose $u\in G^\circ$ is a vertex.
Let $T\subseteq\sens$ be a coherent subset.
Then the following algorithm constructs a shortest path in $\dual{G}$ from $u$ to $K=\half{T;G}$:
\begin{enumerate}
	\item Find an element $b\in T\minus u$; if no such element, stop and output $u$.
	\item Find an element $c\leq b\com$ with $c\in\min(u)$;
	\item Replace $u$ by $\flip{u}{c}$ and return to the first step.
\end{enumerate}
\end{proposition}
\begin{proof} We have $u\in K$ if and only if $T\subset u$, which provides the stopping condition for the algorithm.
Now, if $u\notin K$ and $b\in T\minus u$ then for all $v\in K$ one has $v\in\half{b;G}$ and $u\in\half{b\com;G}$.
Since $c\leq b\com$, we have $u\in\half{c;G}\subseteq\half{b\com;G}$, implying $v\in\half{c\com;G}$ and $c\in u\minus v$.
As a result:
\begin{equation}
	\ellone{v}{\flip{u}{c}}=\ellone{v}{u}-1
\end{equation} 
Having reduced $\ellone{u}{v}$ by a unit for all $v\in K$, we have reduced $\ellone{u}{K}$ by a unit as well.\qed
\end{proof}
\begin{corollary}[Projection of a Point]\label[corollary]{cor:projection of a point} Let $G$ and $T$ be as above.
Then the closest point projection to $K=\half{T;G}$ is given by the formula:
\begin{equation}
	\proj{K}{u}=(u\minus\down{T\com})\cup\up{T}=(u\cup\up{T})\minus\down{T\com}\,.
\end{equation}
\end{corollary}
\begin{proof} Note that the second equality follows from the DeMorgan rules and the fact that $\up{T}\cap\down{T\com}=\varnothing$, since $T$ is coherent.
We now prove the first equality.

Set $K=\half{T;G}$ and proceed by induction on $\ellone{u}{K}$. If $\ellone{u}{K}=0$, then $u\in K$ and therefore $T\subset u$. In addition, $u$ is coherent and we conclude $\down{T\com}\cap u=\varnothing$, leaving us with
\begin{equation}
	u\minus\down{T\com}\cup T=u\cup T=u\,,
\end{equation}
as desired. Now suppose $n:=\ellone{u}{K}>0$. By the preceding proposition, there is $a\in\down{T\com}\cap u$ such that $v:=\flip{u}{a}\in G^\circ$, $\ellone{v}{K}=n-1$, and $\proj{K}{u}=\proj{K}{v}$. We thus have:
\begin{equation}
	\proj{K}{u}=\proj{K}{v}=(v\minus\down{T\com})\cup\up{T}=(u\minus\down{T\com})\cup\up{T}\,,
\end{equation}
the last equality being due to $a\in T\com$ and $a\com\in T$.\qed
\end{proof}

\subsubsection{Projecting a Convex Set to a Convex Set}
\begin{proposition}\label[proposition]{prop:general projection} Let $K,L$ be non-empty convex subsets of $\dual{G}$ with $L=\half{S;G}$ and $K=\half{T;G}$. Then:
\begin{equation}\label{eqn:general projection}
	\begin{array}{rcl}
	\proj{K}{L}&=&\half{(\up{S}\cup\up{T})\minus\down{T\com};G}\\
		&=&\half{T;G}\cap\half{\up{S}\minus\up{T}\com;G}\,.
	\end{array}
\end{equation}
\end{proposition}
\begin{proof} Since $T$ is coherent, $\up{T}$ and $\down{T\com}=\up{T}\com$ are disjoint. This allows us to write:
\begin{eqnarray*}
	\half{(\up{S}\cup\up{T})\minus\up{T}\com;G}
		&=&	\half{\up{T}\cup(\up{S}\minus\up{T}\com);G}\\
		&=&	\half{\up{T};G}\cap\half{\up{S}\minus\up{T}\com;G}
\end{eqnarray*}
and the second equality in \Cref{eqn:general projection} follows from the identity $\half{T;G}=\half{\up{T};G}$. Denote $R=\up{S}\minus\up{T}\com$ and $N=\half{R;G}$.

For every $u\in L=\half{S;G}$ we have $\up{S}\subset u$, implying $\proj{K}{u}$ contains $\up{T}\cup R$, by \Cref{cor:projection of a point}. Thus, $\proj{K}{L}\subset K\cap N$, as required.

For the converse, observe that the case $K\cap L\neq\varnothing$ was already dealt with in \Cref{lemma:general projection_intersecting}: if $K\cap L\neq\varnothing$, then 
\begin{equation*}
	\proj{K}{L}=K\cap L=\half{\up{S};G}\cap\half{\up{T};G}=\half{\up{S}\cup\up{T};G}
\end{equation*}
In particular, $\up{S}\cup\up{T}$ is coherent, and hence does not intersect $\up{T\com}$, and the formula \Cref{eqn:general projection} holds.

Thus we may henceforth assume $K\cap L=\varnothing$. Equivalently, $\up{S}\cap\down{T\com}\neq\varnothing$. In fact, by \Cref{lemma:supporting halfspace} we have $\up{S}\cap\down{T\com}=\sepr{A,B}$.

Starting with $v\in K\cap N$ we must show $v\in\proj{K}{L}$. Set $u=\proj{L}{v}$, $w=\proj{K}{u}$, and $m=\med{u}{v}{w}$. Then $m\in K$ since $v,w\in K$. Since $K\cap L=\varnothing$, we have $\ellone{u}{v}>0$ and $\ellone{u}{w}>0$. Consider the point $m$: we have $m\in I(u,w)$ and $m\in K$; by the choice of $w$, $m$ must equal $w$ and therefore $w\in I(u,v)$. Thus, $w=\proj{K}{u}\in I(u,v)$ and $u=\proj{L}{w}$. By \Cref{prop:gates exist}, the pair $u,w$ is a gate for $K,L$ and we have
\begin{equation*}
	u\minus w=\sepr{L,K}=\up{S}\cap\down{T\com}\,.
\end{equation*}

Consider an element $a\in v\minus u$. If $\half{a;G}\cap L\neq\varnothing$, pick $u'\in\half{a;G}\cap L$. Then $m=\med{u}{v}{u'}$ will satisfy $m\in\half{a;G}\cap L$ as well as
\begin{equation*}
	\ellone{v}{L}=\ellone{v}{u}=\ellone{v}{m}+\ellone{m}{u}\,.
\end{equation*}
Now, $\ellone{u}{m}>0$ since $u\in\half{a\com;G}$ and a contradiction to $u\proj{L}{v}$ is obtained. Thus, $\half{a;G}\cap L$ must be empty, which means $L\subseteq\half{a\com;G}$. Applying \Cref{lemma:supporting halfspace} we obtain $a\com\in\up{S}$.

Overall, we have shown that  $v\minus u\subseteq\up{S}\com$. We will now verify that $v\minus w=\varnothing$, finishing the proof. Indeed, were it not so, there would have been $h\in v\minus w$. On one hand, $w\in I(u,v)$ implies $v\minus w\subset v\minus u$, and hence $h\com\in\up{S}$. On the other hand, $h\notin w$ means $h\com\in w$ and therefore $h\com\notin\sepr{L,K}=\up{S}\cap\up{T}\com$, which forces $h\com\in R$. Since  $R\subset v$ (by choice of $v$), we have $h\com\in v$, contradicting our choice of $h$.\qed
\end{proof}

We will need the following technical corollary for the purposes of propagation:
\begin{corollary}\label[corollary]{cor:projection by propagation} Let $S,T\subset P$ be subsets and suppose $S$ is coherent. Let $L=\half{S;G}$ and $K=\half{\coh{T};G}$. Then:
\begin{equation}
	\proj{K}{L}=(\up{S}\cup\up{T})\minus\up{T}\com=(\up{S}\minus\up{T}\com)\cup\coh{}{T}\,.
\end{equation}
\end{corollary}
\begin{proof} Recall that $\coh{}{T}=\up{T}\minus\up{T}\com$, and set $J=\up{T}\cap\up{T}\com$, so that $\up{T}=\coh{}{T}+J$ and $\up{T}\com=\coh{T}\com+J$. Then,
\begin{eqnarray*}
	(\up{S}\cup\up{T})\minus\up{T}\com
		&=& ((\up{S}\cup\coh{}{T}\cup J)\minus\coh{}{T}\com)\minus J\\
		&=& (\up{S}\cup\coh{}{T})\minus\coh{}{T}\com
\end{eqnarray*}
Since $\up{\coh{}{T}}=\coh{}{T}$, the last expression equals $\proj{K}{L}$, by the preceding proposition. The proof of the second equality is similar.\qed
\end{proof}

\section{Appendix: Qualitative Snapshots (proofs)}\label{proofs:qualitative snapshots}

\subsection{Proof of \Cref{lemma:triangle inequality}.}\label{proof:triangle inequality} For all $a,b,c\in\sens$ one has $\half{ac\com}=\half{abc\com}\cup\half{ab\com c\com}\subseteq\half{bc\com}\cup\half{ab\com}$.
Thus, either the minimum of $\kappa$ over $\half{ac\com}$ is attained at a point of $\half{bc\com}$ or it is attained at a point of $\half{ab\com}$ (or both).
Therefore one has $\witness{\kappa}{ac\com}\geq\witness{\kappa}{ab\com}$ or $\witness{\kappa}{ac\com}\geq\witness{\kappa}{bc\com}$, as required.
\qed

\subsection{Proof of \Cref{prop:residual PCR}.}\label{proof:residual PCR}
Denote $G=\residual{\witness{}{\wild};\delta}$ for the rest of this proof.
First, we need to show that $ab\in G$ implies $b\com a\com\in G$.
This is baked into the definition, as $\witness{}{b\com a\com\com}=\witness{}{ab\com}$.
Also, $\mbf{0}a\in G$ is satisfied because $\witness{}{\mbf{0}a}=\infty$.
Finally, applying \Cref{lemma:triangle inequality} we conclude that, for all $a,b\in\sens$ one has $a\leq_G b\THEN\witness{}{ab\com}>\delta$ when $\delta<\infty$, and $a\leq_G b\THEN\witness{}{ab\com}=\infty$ when $\delta=\infty$.
In particular, were $a\in N(G)\cap N(G)\com$, then $a\leq_G a\com$ would have implied $\witness{}{a}=\witness{}{aa\com\com}>\delta$ (or equals $\infty$ if $\delta=\infty$), while $a\com\leq_G a$ would have given $\witness{}{a\com}=\witness{}{a\com a\com}>\delta$ (or $\infty$, respectively)). But that would have meant $\witness{}{\varnothing}\in(\delta,\infty]$~ ---~ a contradiction.
\qed

\subsection{Proof of \Cref{cor:point mass clamp}.}\label{proof:point mass clamp}
With $r=\witness{}{ab}\geq\witness{}{\varnothing}$, we consider the PCR $G=\residual{\witness{}{\wild};r}$, for which we have $p\leq_G q$ if and only if $\witness{}{pq\com}>r$.
By \Cref{prop:residual PCR}, $G$ is non-degenerate, hence there exists $u\in G^\circ\subseteq\ham$.
We set $\nu=\pointmass{\nu}{r}$.
For any $p,q\in u$, since $p\not\leq q\com$, we must have $\witness{}{pq}\leq r=\witness{\nu}{pq}=\witness{}{ab}$.
At the same time, if $\{p,q\}\nsubseteq u$, then $\witness{}{pq}\leq\infty=\witness{\kappa}{pq}$ again. Thus $u$ is the desired vertex of $\ham$.
\qed

\subsection{Proof of \Cref{prop:completions exist}.}\label{proof:completions exist}
Sufficiency follows from \Cref{lemma:triangle inequality} and the observations following \Cref{defn:ranking,defn:concept representation}.
Now, suppose $\witness{}{\wild}$ is a 2-ranking, and consider the set $\mathscr{K}$ of all rankings $\kappa$ satisfying $\witness{\kappa}{ab}\geq\witness{}{ab}$ for all $a,b\in\sens$.
By \Cref{ex:minimum}, the family $\mathscr{K}$ is closed under taking pointwise minima. Since $\sens$ is finite, $\mathscr{K}$ must have a minimum element. 

Let $\completion{\witness{}{}}$ be given by \Cref{eqn:completion formula}.
To prove that it coincides with the minimum of $\mathscr{K}$ it suffices to verify that (a) $\completion{\witness{}{}}\leq\kappa$ for all $\kappa\in\mathscr{K}$, and that (b) $\completion{\witness{}{}}$ agrees with $\witness{}{\wild}$.

Fix $\kappa\in\mathscr{K}$. Then, for any $u\in\ham$ and $a,b\in u$ we have $\witness{\kappa}{ab}\leq\kappa(u)$ because $\kappa$ is a ranking, and $\witness{}{ab}\leq\witness{\kappa}{ab}$ by the particular choice of $\kappa$, proving (a).
Finally, to prove (b), it suffices to verify that, for every $a,b\in\sens$, there exists a ranking $\nu\in\mathscr{K}$ with $\witness{\nu}{ab}=\witness{}{ab}$. Indeed, \Cref{cor:point mass clamp} provides just such a ranking $\nu$, setting $r=\witness{}{ab}$, which finishes the proof.
\qed

\subsection{Proof of \Cref{prop:global minima}.}\label{appendix:snapshots:minima} The first immediate observation regarding minsets is the following observation:
An immediate result is this: 
\begin{lemma}\label[lemma]{lemma:minsets are forward-closed coherent} Let $\delta,\epsilon\geq 0$, and let $\witness{}{\wild}$ be a 2-ranking.
Let $G=\derived{\witness{}{\wild};\delta}$ and $M=\minset{\witness{}{\wild};\epsilon}$.
Then $M\in\mbf{C}(G)$ and
$\half{M;G}\neq\varnothing$.\qed
\end{lemma}
\begin{proof} Let $G=\derived{\witness{}{\wild},\delta}$ and $M=\minset{\witness{}{\wild},\epsilon}$ for some $\delta,\epsilon\geq 0$.  
Clearly, $M$ is a $\ast$-selection. Now suppose that $a,b\in\sens$ satisfy $ab\in G$. If $a\in M$, then $\witness{}{a}<\witness{}{a\com}-\epsilon$ and $\witness{}{ab\com}>\witness{}{ab},\witness{}{a\com b\com}$. But $\witness{}{b\com}=\min\{\witness{}{ab\com},\witness{}{a\com b\com}\}$ then implies $\witness{}{b\com}=\witness{}{a\com b\com}$, and hence also $\witness{}{a\com}\leq\witness{}{a\com b\com}=\witness{}{b\com}$. Similarly, $\witness{}{a}=\min\{\witness{}{ab\com},\witness{}{ab}\}=\witness{}{ab}\geq\witness{}{b}$, and we have $\witness{}{b}\leq\witness{}{a}<\witness{}{a\com}-\epsilon\leq\witness{}{b\com}-\epsilon$, proving $b\in M$, and we conclude that $M$ is a forward-closed $\ast$-selection. In particular, it is $G$-coherent.\footnote{For, suppose $a,b\in M$ and we had $a\leq_G b\com$; then we would have also had $b\com\in M$ because $M$ is forward-closed, contradicting $M$ being a $\ast$-selection.}\qed
\end{proof}

To prove \Cref{prop:global minima}, we need to analyze the relationship between level sets of the rankings $\kappa$ and $\completion{\kappa}$.
We have the following lemma:
\begin{lemma}\label[lemma]{lemma:level set supports}
Let $\kappa$ be a ranking in $\ham$ and fix a value $r\in\NNH$, $\witness{\kappa}{\varnothing}\leq r<\infty$.
For the sub-level sets $F=[\kappa\leq r]$ and $\twoclose{F}=[\twoclose{\kappa}\leq r]$ one has: {\rm (a)} $F\subseteq\completion{F}$, {\rm (b)} $F\supp=\completion{F}\supp$, and {\rm (c)} $\completion{F}\subseteq\half{F\supp}$.
\end{lemma}
\begin{proof} Since $\twoclose{\kappa}\leq\kappa$, we have $F\subseteq\twoclose{F}$, which, in turn, implies $\completion{F}\supp\subseteq F\supp$.
Conversely, if $a\in F\supp$, then $F\subseteq\half{a}$ implies $\kappa(u)>r$ for all $u\in\half{a\com}$; equivalently, $\kappa(\half{a\com})>r$.
Since this information carries over to the 2-restriction of $\kappa$, we conclude that $\completion{\kappa}(\half{a\com})>r$ as well, which means $a\in\completion{F}\supp$, verifying (b).
Assertion (c) follows directly from (b) via $\completion{F}\subseteq\half{\completion{F}\supp}=\half{F\supp}$.
\qed\end{proof}
One can say more about the lowest level sets of $\kappa$:
\begin{lemma}\label[lemma]{lemma:which levels in Gcirc} In the notation of \Cref{lemma:level set supports}, let $G=\derived{\kappa;\delta}$ with $\delta\in\NNH$. Then:
\begin{enumerate}[label={\rm (\alph*)}]
	\item If $r\leq\delta+\witness{\kappa}{\varnothing}\in\NNH$, then $\completion{F}\subseteq G^\circ$;
    \item If $\delta=0$, then $\half{F\supp}\cap G^\circ\subset\completion{F}$.
\end{enumerate}
\end{lemma}
\begin{proof} To prove (a), consider an arbitrary $u\in\completion{F}$.
Since $u$ is a complete $\ast$-selection it suffices to show that it is forward-closed with respect to $G$.
Take any $x\in u$ and $y\in\sens$. If $xy\in G$, then:
\begin{equation*}
	\witness{\kappa}{xy\com}>\delta+\max\{\witness{\kappa}{xy},\witness{\kappa}{x\com y\com}\}\geq\delta+\witness{\kappa}{\varnothing}\,.
\end{equation*}
However, if $y\notin u$, then we have $\{x,y\com\}\subseteq u$.
Since $\completion{\kappa}(u)\leq r$, we must have $\witness{\kappa}{xy\com}\leq r$, by \Cref{eqn:completion formula}.
This, however, contradicts our assumption regarding $r$.

To verify (b), suppose $u\in\half{F\supp}\cap G^\circ$, but $\completion{\kappa}(u)>r$. Applying \Cref{eqn:completion formula} again, we conclude there is a pair $x,y\in u$ with $\witness{\kappa}{xy}>r$.
In particular, no element of $F$ is contained in $\half{xy}$, which leads to the following two complementary cases:
\begin{itemize}
	\item {\bf Both $\half{xy\com}$ and $\half{x\com y}$ contain elements of $F$. } Then we have $\witness{\kappa}{xy}>r\geq\witness{\kappa}{xy\com},\witness{\kappa}{x\com y}$, which means $xy\com\in G$ and contradicts $u\in G^\circ$.
    \item {\bf Either $F\cap\half{x}=\varnothing$ or $F\cap\half{y}=\varnothing$. } In other words, either $x\com\in F\supp$ or $y\com\in F\supp$. Since $F\supp\subseteq u$, we conclude that one of $x\com,y\com$ must lie in $u$. This contradicts the fact that $x,y\in u$, since $u$ is a complete $\ast$-selection.
\end{itemize}
Thus, either case yields a contradiction, finishing the proof.
\qed\end{proof}

\bigskip
We are finally ready to prove \Cref{prop:global minima}.

\medskip
Setting $r=\kappa(\ham)$---the minimum value of $\kappa$---we apply \Cref{lemma:which levels in Gcirc}(a) to conclude that $\completion{F}\subseteq G^\circ$.
From \Cref{lemma:which levels in Gcirc}(b) and \Cref{lemma:level set supports}(c) we obtain that $\half{F\supp;G}=\half{F\supp}\cap G^\circ$ coincides with $\completion{F}$.
Since $F\subseteq\completion{F}$ by \Cref{lemma:level set supports}(a) we conclude that $F\subset G^\circ$ and we may apply \Cref{cor:convex hull formula} to deduce that $\completion{F}$ is the convex hull of $F$ in $\dual{G}$.
Finally, consider the set $M=\minset{\kappa}$. For any $a\in\sens$, at least one of $\kappa(\half{a})$, $\kappa(\half{a\com})$ equals $\kappa(\ham)$. Therefore, we have $a\in M$ if and only if $\kappa(\half{a})=\kappa(\ham)$ and $\kappa(\half{a\com})>\kappa(\ham)$, if and only if $a\in F\supp$ for our current choice of $F=[\kappa\leq\kappa(\ham)]$. This finishes the proof.\qed

\section{Appendix: Real-Valued Snapshots (proofs).}

\subsection{Proof of \Cref{prop:statistical derived PCR}.}\label{proof:statistical derived PCR}
We must show that the derived PCR $G=\derived{\witness{}{\wild}}$ of a real-valued 2-weight $\witness{}{\wild}$ satisfying the requirements 1.-5. of \Cref{defn:real-valued 2-weight} is non-degenerate. Recall \Cref{eqn:real-valued derived PCR}, defining $G$ given an assignment of thresholds $\tau_{\wild}$:
\begin{equation*}
    ab\in G\quad\IFF\quad
    \witness{}{ab\com}\at{t}<\min(
        \tau_{ab}\cdot\witness{}{\varnothing},
        \witness{}{ab},
        \witness{}{a\com b\com},
        \witness{}{a\com b}
    )
    \;or\;
    \witness{}{ab\com}=\witness{}{a\com b}=0\,,
\end{equation*}
Define functions $\omega,\partial:\sens\times\sens\to\RR$ via $\ori{a}{b}:=\witness{}{a\com b}-\witness{}{ab\com}$ and $\tri{a}{b}:=\witness{}{a\com b}+\witness{}{ab\com}$.
From properties 1. and 4. of the 2-weight $\witness{}{\wild}$, one has the identities $\ori{a}{b}=-\ori{b}{a}$, $\ori{a}{a}=0$ and $\ori{a}{c}=\ori{a}{b}+\ori{b}{c}$.
From properties 2. and 5. of the 2-weight $\witness{}{\wild}$ one also obtains the identities $\tri{a}{b}=\tri{b}{a}\geq 0$, $\tri{a}{a}=0$, $\tri{a}{c}\leq\tri{a}{b}+\tri{b}{c}$, and $\tri{a}{a\com}=\witness{}{\varnothing}$.

We are ready to prove the proposition.
Suppose $a\leq_G a\com\leq_G a$ for some $a\in\sens$, and find a sequence $a_0,\ldots,a_m,\ldots,a_n$ with $a_0=a$, $a_m=a\com$, $a_n=a$ and $a_{k-1}a_k\in G$ for $k=1,\ldots,n$.
We then must have $\ori{a}{a}=\sum_{k=1}^n\ori{a_{k-1}}{a_k}\geq 0$, with equality if and only if $\ori{a_{k-1}}{a_k}=0$ for all $k=1,\ldots,n$.
By the definition of $G$, this implies $\witness{}{a_{k-1}\com a_k}=\witness{}{a_{k-1}a_k\com}=0$ for all $k=1,\ldots,n$.

But then we also have $\witness{}{\varnothing}=\tri{a}{a\com}\leq\sum_{k=1}^m \tri{a_{k-1}}{a_k}=0$, which is only possible when $\witness{}{\wild}$ is trivial---a contradiction.\qed

\subsection{Proof of \Cref{prop:PAC real-valued}. }\label{proof:PAC real-valued}
The proof of this proposition follows a standard scheme, widely attributed to Chernoff, and is only included here for the sake of completeness.

Recall that the Kullback-Leibler divergence of a $\mathbf{Ber}(q)$ random variable relative to a $\mathbf{Ber}(p)$ random variable is given by\footnote{We will always mean the natural base, $\ee$, of the logarithm when using the notation $\log(\wild)$.}:
\begin{equation}\label{eqn:KLdiv}
    \kldiv{q}{p}:=q\log\tfrac{q}{p}+(1-q)\log\tfrac{1-q}{1-p}\,.
\end{equation}
We require the following standard lemma:
\begin{lemma}[KL-divergence bound]\label[lemma]{lemma:KL divergence bound} Let $p,q\in(0,1)$, and consider the function $f(\zeta):=\ee^{-q\zeta}\left(1-p+p\ee^{\zeta}\right)$ over the interval $(0,\infty)$.
Then
\begin{equation*}
    \inf_{\zeta>0}f(\zeta)
    =
    \left\{\begin{array}{cl}
        \ee^{-\kldiv{q}{p}}
    &\text{if }q>p\\[1em]
    1 &\text{otherwise,}
    \end{array}\right.
    \qquad
    \kldiv{q}{p}:=q\log\tfrac{q}{p}+(1-q)\log\tfrac{1-q}{1-p}
\end{equation*}
\end{lemma}
\begin{proof} Differentiating $f$ one obtains:
\begin{align*}
    f'(\zeta)&=
    \ee^{-q\zeta}\left(
        -q(1-p)+p(1-q)\ee^{\zeta}
    \right)\,,\\
    f''(\zeta)&=
    \ee^{-q\zeta}\left(
        q^2(1-p)+p(1-q)^2\ee^\zeta
    \right)>0\,.
\end{align*}
The function $f$ has only one critical point:
\begin{align*}
    f'(\zeta_0)=0\IFF p(1-q)\ee^{\zeta_0}=q(1-p)
    \IFF\zeta_0=\log\tfrac{q(1-p)}{p(1-q)}\,,
\end{align*}
and the value of $f$ at $\zeta_0$ is the claimed value, $f\left(\zeta_0\right)=\ee^{-\kldiv{q}{p}}$.

Finally, $\zeta_0>0$ if and only if $q(1-p)>p(1-q)$, which is tantamount to $q>p$. 
\qed
\end{proof}

\bigskip
The setting for learning snapshot weights described in \Cref{prop:PAC real-valued} simplifies to the following.
Suppose $\rv{X}\in[0,A]$ is a non-constant random variable, and let $\alpha:=\tfrac{\expect{\rv{X}}}{A}$.
We posit a sequence $\rv{X}\at{t}$, $t\geq 0$ of i.i.d. random variables $\rv{X}\at{t}\sim\rv{X}$.

We have the following probability bounds:
\begin{lemma}\label[lemma]{lemma:Chernoff bounds}
Let $t$ be a non-negative integer and let $\delta>0$.
Then, for $\rv{Y}:=\tfrac{1}{t+1}\sum_{s=0}^t\rv{X}\at{s}$ one has:
\begin{align}
    \prob{\left|\rv{Y}-\expect{\rv{X}}\right|\geq\delta}
    &\leq
    \ee^{-(t+1)\kldiv{\beta}{\alpha}}
    +
    \ee^{-(t+1)\kldiv{1-\gamma}{1-\alpha}}\,,
\end{align}
%
where $\beta=\alpha+\tfrac{\delta}{A}$ and $\gamma:=\alpha-\tfrac{\delta}{A}$.\qed
\end{lemma}

To prove \Cref{prop:PAC real-valued}, observe that the first bound---the standard Chernoff bound---guarantees exponentially fast convergence in probability of the empirical snapshot weights $\witness{}{ab}\at{t}$ to the mean value of the signal over the domain $\rho(a)\cap\rho(b)$ when we take $\rv{X}\at{t}:=\val{}\at{t}\cdot\delta_{u_t}(\half{ab})$.

\bigskip
We are left to verify the Chernoff bound.
\paragraph{Proof of \Cref{lemma:Chernoff bounds}: } Recall $\alpha:=\tfrac{\expect{\rv{X}}}{A}$, and observe that $0<\alpha<1$ because $\rv{X}$ is non-constant.
We proceed in the standard way to obtain a bound for the empirical estimate of the sample mean. For every {\em fixed value} of $t$, and recalling $\beta=\alpha+\tfrac{\delta}{A}$ one has:
\begin{align*}
    P&=
    \prob{\rv{Y}\geq\expect{\rv{X}}+\delta}
    =
    \prob{\tfrac{\rv{Y}}{A}\geq\alpha+\tfrac{\delta}{A}}
    =
    \prob{\sum_{s=0}^t \tfrac{\rv{X}\at{s}}{A}\geq\beta(t+1)}\\
    &=
    \prob{\exp{\left(\zeta\sum_{s=0}^t \tfrac{\rv{X}\at{s}}{A}\right)}\geq \ee^{\zeta\beta(t+1)}}
    \leq
    \ee^{-\zeta\beta(t+1)}
    \expect{\exp{\left(\zeta\sum_{s=0}^t \tfrac{\rv{X}\at{s}}{A}\right)}}\\
    \intertext{by Markov's inequality; also, since the $\rv{X}\at{s}$ are independent we have:}
    &=
    \ee^{-\zeta\beta(t+1)}
    \prod_{s=0}^t
        \expect{
            \ee^{\zeta\tfrac{\rv{X}\at{s}}{A}}
        }
    \leq
    \ee^{-\zeta\beta(t+1)}
    \prod_{s=0}^t
    \expect{
        1-\tfrac{\rv{X}\at{s}}{A}+\tfrac{\rv{X}\at{s}}{A}\ee^{\zeta}
    }\,,
\end{align*}
using the inequality $(\dagger)\;\ee^{x\lambda}\leq 1-x+x\ee^{\lambda}$ for $x\in[0,1]$, $\lambda>0$.
Finally, this yields:
\begin{equation*}
    P\;\leq\;
    \ee^{-\zeta\beta(t+1)}
    \left(
        1-\alpha+\alpha\ee^{\zeta}
    \right)^{t+1}\,.
\end{equation*}
Using \Cref{lemma:KL divergence bound} to minimize the right hand side over $\zeta>0$, we obtain, for every fixed $t$:
\begin{equation*}
    P\leq \exp\left(-(t+1)\kldiv{\beta}{\alpha}\right)\,,
\end{equation*}
as claimed.
Now replace $\rv{X}$ with $A-\rv{X}$, $\rv{Y}$ with $A-\rv{Y}$, and recall $\gamma:=\alpha-\tfrac{\delta}{A}$.
We obtain:
\begin{align*}
    \prob{\rv{Y}\leq\expect{\rv{X}}-\delta}
    &=
    \prob{A-\rv{Y}\geq A-\expect{\rv{X}}+\delta}
    =
    \prob{A-\rv{Y}\geq\expect{A-\rv{X}}+\delta}\\
    &\leq
    \exp\left(
        -(t+1)\kldiv{1-\alpha+\tfrac{\delta}{A}}{1-\alpha}
    \right)\\
    &=
    \exp\left(
        -(t+1)\kldiv{1-\gamma}{1-\alpha}
    \right)\,,
\end{align*}
finishing the proof.\qed

\section{Appendix: Debugging Sniffy on the Circle.}\label{app:sniffy on the circle}
The purpose of this appendix is to explain in some detail the reasons for Sniffy's behavior on the circle in its qualitative BUA incarnation, as discussed in \Cref{sim:2:sniffy:circle}.
We proceed in a manner similar to the discussion of an agent on the interval from \Cref{poc:ex:moving bead on an interval}.

\subsection{Sensors and Relations.}\label{app:circle_debug:sensors} Let $a_k$ denote the sensor centered at $k\in\env=\{0,\ldots,19\}$, reporting $1$ at time $t$ if and only if $\dist{k}{\pos(t)}\leq 4$, where $\pos(t)$ is the position occupied by Sniffy at time $t$.
It will be convenient to think of $\env$ as a copy of the additive group $ZZ_{20}$, keeping in mind its action on subsets $S\subseteq\env$ given by $k+S:=\{i-k\,|\,i\in S\}$.

It is then easy to verify that $a_k<a\com_{k+\{9,10,11\}}$ are the only relations among the $a_k$.
Consequently, the analogous relations $\delay a_k<\delay a\com_{k+\{9,10,11\}}$ must also hold for all $k$ (see \Cref{poc:ex:circle} for details).
Motion is described by the conditional relations $\rt:\delay a_k<a_{k+1}$ and $\lt:\delay a_k<a_{k-1}$, leading to the unconditional implications $\delay a_k<a_{k+10}$ (compare with the case of the interval discussed in \Cref{poc:ex:moving bead on an interval}).

Finally, observing that the entire setting is rotation-invariant, without loss of generality we may assume for the rest of this section that Sniffy's target is located at position $T=0$.
Setting $M:=\{a_0,a\com_{10}\}\cup\{a_{\pm 1},\ldots,a_{\pm 4}\}\cup\{a_{\pm 5},\ldots,a_{\pm 9}\}\com$, the {\em eventual} target sets (minsets) determined by the individual snapshots are:
\begin{equation}\label{eqn:Sniffy circle targets}
\begin{array}{rcl}
    M(\rt\com)=M(\lt\com)&:=&M\cup\delay M\,,\\
    M(\rt)&:=&M\cup\delay(1+M)\,,\\
    M(\lt)&:=&M\cup\delay(-1+M)\,,
\end{array}
\end{equation}
where, due to the hard-wired arbitration enforcing $\neg(\rt\wedge\lt)$ at all times, the derived PCR $G^{\rt}$ of the $\rt$ snapshot identifies each $a_k$ with $\delay a_{k-1}$; and, similarly, $G^{\lt}$ identifies each $a_k$ with $\delay a_{k+1}$.

Finally, note that given $\rt\com$ it is impossible to witness $a_{k+9}\wedge\sharp a_k$ or $a_{k+10}\wedge\sharp a_k$ (while $a_{k+11}=a_{k-9}$ is still possible in conjunction with $\sharp a_k$, if $\lt$ is active).
We conclude that the relations $\sharp a_k<a\com_{k+9},a\com_{k+10}$ hold in $G^{\rt\com}$ for every $k$, as do their analogous counterparts in $G^{\lt\com}$.

\subsection{The ``dull peak'' value signal.}\label{app:circle_debug:dull}
The weights recorded on any snapshot in this case are $\{0,1,\infty\}$, implying that (1) any implications appearing in Sniffy's four snapshots at any time are a subset of the implications listed in the preceding section; and (2) any raw observation generated by Sniffy is coherent for any of its snapshots.

Let $k\in\env$.
Each snapshot forms its prediction for the next state by propagating $\sharp(-k+M)$, and giving rise to:
\begin{equation}\label{eqn:sniffy circle dull predictions}
\begin{array}{lcl}
    \prediction{\lt}&=&
    \sharp(-k+M)\cup(1-k+M)\\
    \prediction{\lt\com}&=&
    \sharp(-k+M)\cup\coh{\lt\com}{(-k+M)\cup(1-k+M)}\\
    &=&
    \sharp(-k+M)\cup\left((-k+M)\minus\{a\com_{k-5},a_{k+4}\}\right)
\end{array}
\end{equation}
In order to compute the divergences from the targets, we first note that, for $k,\ell\in\env$:
\begin{equation}\label{eqn:divergences on the circle}
    \card{(k+M)\minus(\ell+M)}=
    \card{M\minus(\ell-k+M)}=
    2\min\left\{
        9,
        \dist{k}{\ell}
    \right\}\,.
\end{equation}
For the $\lt$ snapshot this results in:
\begin{eqnarray}
    \nonumber\divergence{\prediction{\lt};M(\lt)}
    &=&
    M(\lt)\minus\prediction{\lt}\\
    \nonumber&=&
    \card{\sharp(-1+M)\minus\sharp(-k+M)}+\card{M\minus(1-k+M)}\\
    \nonumber&=&
    \card{(-1+M)\minus(-k+M)}+\card{M\minus(1-k+M)}\\
    &=&
    4\min\left\{9,\dist{1}{k}\right\}\,.
\end{eqnarray}
By symmetry, we conclude that the $\rt$ snapshot has:
\begin{equation}
    \divergence{\prediction{\rt};M(\rt)}=4\min\left\{9,\dist{-1}{k}\right\}\,.
\end{equation}
Now, for the $\lt\com$ snapshot we have:
\begin{eqnarray*}
    \divergence{\prediction{\lt\com};M(\lt\com)}&=&
    \card{M\minus(-k+M)}+\card{M\minus((-k+M)\minus\{a\com_{k-5},a_{k+4}\})}\\
    &=&
    4\min\left\{9,\dist{0}{k}\right\}+\delta(k)\,,
\end{eqnarray*}
where $\delta(k)=2$ for $-9\leq k=0$, $\delta(k)=0$ for $0<k\leq 9$, $\delta(k)=1$ for $k=10$.
By symmetry:
\begin{equation}
    \divergence{\prediction{\rt};M(\rt)}=4\min\left\{9,\dist{0}{k}\right\}+\delta(-k)\,.
\end{equation}
Armed with these formulae, we go over the possibilities and conclude:
\begin{itemize}
    \item The BUA $\lt$ is active if and only if $k\in T+\{1,\ldots,11\} (\mathrm{mod } 20)$;
    \item The BUA $\rt$ is active if and only if $k\in T-\{1,\ldots,11\} (\mathrm{mod } 20)$.
\end{itemize}
In particular, $k\in T\pm\{0,\ldots,8\}$ is a basin of attraction for the target position, $T$, while $k\in T+\{9,10,11\}$ is a region where both $\lt$ and $\rt$ seek to be active, triggering the hard-wired arbitration mechanism.

\subsection{The ``sharp peak'' value signal.}\label{app:circle_debug:sharp}
Despite the value signal being more informative than that of its $\{0,1\}$-valued ``dull-peak'' counterpart, a ``sharp peak'' qualitative agent's performance on the target-finding task is clearly worse (\Cref{fig:split the runs}).

In a nutshell, the reason for this deficiency is that the limiting PCR satisfies the additional relations $a_{\pm 9}<a_{\pm 8}<\ldots<a_{\pm 1}$, which we verified by hand.
As a result, properties (1) and (2) stated in \Cref{app:circle_debug:dull} for the ``dull peak'' setting will not hold in this one.
In fact, just these extra relations (there may be others) suffice for the current state representation of any point other than $k=0,10$ in any of Sniffy's snapshots to degenerate (through coherent projection) into a less and less complete $\ast$-selection as Sniffy's physical distance from the target (in the environment) increases, while the quality of prediction deteriorating accordingly.

Specifically, any $k\in B:=10\pm\{0,1,2,3\}$ yields a raw observation containing $a_{10},a\com_0$ and both of $a_9$ and $a_{11}$.
On one hand, $a_9,\ldots,a_1$ and $a_{11},\ldots,a_{19}$ are directed paths in each of Sniffy's PCRs.
On the other hand, though, so are $a_k,a\com_{k+10}$.
Thus, (1) the coherent projection of the raw observation generated by $k$ is merely $\{a_{10},a\com_0\}$, showing that Sniffy is unable to distinguish among the points of $B$; and (2) if $k$ is moved closer to the target, fewer conflicts of the above form will affect coherent projection.

In total, the observations above suffice for explaining the most visible differences (see \Cref{fig:split the runs}) between the behaviors of the two variants of Sniffy in the qualitative learning regime.

\end{document}